\documentclass{article}

     \usepackage[final,nonatbib]{neurips_2020}
\usepackage[table,xcdraw]{xcolor}
\usepackage{cite}
\usepackage[utf8]{inputenc} % allow utf-8 input
\usepackage[T1]{fontenc}    % use 8-bit T1 fonts
\usepackage{hyperref}       % hyperlinks
\usepackage{url}            % simple URL typesetting
\usepackage{booktabs}       % professional-quality tables
\usepackage{amsfonts}       % blackboard math symbols
\usepackage{nicefrac}       % compact symbols for 1/2, etc.
\usepackage{microtype}      % microtypography

\usepackage{amsthm, amssymb, amsmath}
\usepackage{verbatim, float}
\usepackage{mathrsfs}
\usepackage{graphics}
\usepackage{subfig}
\usepackage[usenames,dvipsnames]{pstricks}
\usepackage{multirow}
\usepackage{url}
\usepackage{array}
\usepackage{tabu}
\usepackage{epsfig}
\usepackage{parallel}
\usepackage{graphicx}
\usepackage{diagbox}
\usepackage{mathtools}
\usepackage[absolute,overlay]{textpos}
\usepackage{xspace}

\newcommand{\hide}[1]{}

\newcommand{\proj}{DE-GNN\xspace}
\newcommand{\projA}{DEA-GNN\xspace}
\theoremstyle{definition}
\newtheorem{theorem}{Theorem}[section]

\newtheorem{definition}[theorem]{Definition}
\newtheorem{lemma}[theorem]{Lemma}
\newtheorem{corollary}[theorem]{Corollary}

\theoremstyle{remark}

\newtheorem{remark}{Remark}[section]

\title{Distance Encoding: Design Provably More Powerful Neural Networks for Graph Representation Learning}

\author{%
  Pan Li \\
  Department of Computer Science\\
  Purdue University \\
  panli@purdue.edu \\
  \And
  Yanbang Wang \\
  Department of Computer Science\\
  Stanford University \\
  ywangdr@cs.stanford.edu \\
  \AND
  Hongwei Wang \\
  Department of Computer Science\\
  Stanford University \\
  hongweiw@cs.stanford.edu \\
  \And
  Jure Leskovec \\
  Department of Computer Science\\
  Stanford University \\
  jure@cs.stanford.edu \\
}

\begin{document}

\maketitle
\vspace{-0.4cm}
\begin{abstract}
\vspace{-0.1cm}
Learning representations of sets of nodes in a graph is crucial for applications ranging from node-role discovery to link prediction and molecule classification. Graph Neural Networks (GNNs) have achieved great success in graph representation learning. However, expressive power of GNNs is limited by the 1-Weisfeiler-Lehman (WL) test and thus GNNs generate identical representations for graph substructures that may in fact be very different. More powerful GNNs, proposed recently by mimicking higher-order-WL tests, only focus on representing entire graphs and they are computationally inefficient as they cannot utilize sparsity of the underlying graph. 
Here we propose and mathematically analyze a general class of structure-related features, termed Distance Encoding (DE). DE assists GNNs in representing any set of nodes, while providing strictly more expressive power than the 1-WL test. DE captures the distance between the node set whose representation is to be learned and each node in the graph. To capture the distance DE can apply various graph-distance measures such as shortest path distance or generalized PageRank scores. We propose two ways for GNNs to use DEs (1) as extra node features, and (2) as controllers of message aggregation in GNNs. Both approaches can utilize the sparse structure of the underlying graph, which leads to computational efficiency and scalability. We also prove that DE can distinguish node sets embedded in almost all regular graphs where traditional GNNs always fail. We evaluate DE on three tasks over six real networks: structural role prediction, link prediction, and triangle prediction. Results show that our models outperform GNNs without DE by up-to 15\% in accuracy and AUROC. Furthermore, our models also significantly outperform other state-of-the-art methods especially designed for the above tasks. %\pan{As we have not introduced the name \proj and \projA, so I use "our models" instead.}

\hide{
Learning structural representations of node sets from graph-structured data is crucial for applications ranging from node-role discovery to link prediction and molecule classification. Graph Neural Networks (GNNs) have achieved great success in structural representation learning. However, most GNNs are limited by the 1-Weisfeiler-Lehman (WL) test and thus possible to generate identical representation for structures and graphs that are actually different. More powerful GNNs, proposed recently by mimicking higher-order-WL tests, only focus on entire-graph representations and cannot utilize sparsity of the graph structure to be computationally efficient. Here we propose a general class of structure-related features, termed Distance Encoding (DE), to assist GNNs in representing node sets with arbitrary sizes with strictly more expressive power than the 1-WL test. DE essentially captures the distance between the node set whose representation is to be learnt and each node in the graph, which includes important graph-related measures such as shortest-path-distance and generalized PageRank scores. We propose two general frameworks for GNNs to use DEs (1) as extra node attributes and (2) further as controllers of message aggregation in GNNs. Both frameworks may still utilize the sparse structure to keep scalability to process large graphs. In theory, we prove that these two frameworks can distinguish node sets embedded in almost all regular graphs where traditional GNNs always fail. We also rigorously analyze their limitations. Empirically, we evaluate these two frameworks on node structural roles prediction, link prediction and triangle prediction over six real networks. The results show that DE-assisted GNNs outperform GNNs without DEs by up-to 15\% improvement in average accuracy and AUC. DE-assisted GNNs also significantly outperform other state-of-the-art baselines particularly designed for those tasks. %to improve GNN design. 
}
\end{abstract}
%in overcoming the limit of 1-WL test to
\vspace{-0.2cm}
\section{Introduction} \label{sec:intro}
\vspace{-0.15cm}
%Structural representation learning, referring to the machine learning tasks that extract (vector) representations of graph-structured data, has ignited great passion of researchers in the past decades~\cite{hamilton2017representation}. Representations of different numbers of nodes over a graph have been leveraged in a wide range of applications, such as discovery of functions/roles of nodes based on node representations~\cite{borgatti1992notions,henderson2012rolx,rossi2014role,ribeiro2017struc2vec,donnat2018learning}, link or link type prediction based on node-pair representations~\cite{liben2007link,zhang2017weisfeiler,zhang2018link,you2019position}, graph comparison or molecule classification based on entire-graph representations~\cite{prvzulj2007biological,zager2008graph,shervashidze2011weisfeiler,gilmer2017neural,ying2018hierarchical,xu2018powerful,maziarka2020molecule}, and so on and so forth.

%\jure{Remove Structural everywhere}
Graph representation learning aims to learn representation vectors of graph-structured data~\cite{hamilton2017representation}. Representations of node sets in a graph can be leveraged for a wide range of applications, such as discovery of functions/roles of nodes based on individual node representations~\cite{borgatti1992notions,henderson2012rolx,rossi2014role,ribeiro2017struc2vec,donnat2018learning}, link or link type prediction based on node-pair representations~\cite{liben2007link,zhang2017weisfeiler,zhang2018link,you2019position} and graph comparison or molecule classification based on entire-graph representations~\cite{prvzulj2007biological,zager2008graph,shervashidze2011weisfeiler,gilmer2017neural,ying2018hierarchical,xu2018powerful,maziarka2020molecule}.

Graph neural networks (GNNs), inheriting the power of neural networks~\cite{hornik1989multilayer}, have become the {\em de facto} standard for representation learning in graphs~\cite{scarselli2008graph}. Generaly, GNNs use message passing procedure over the input graph, which can be summarized in three steps: (1) Initialize node representations with their initial attributes (if given) or structural features such as node degrees; (2) Iteratively update the representation of each node by aggregating over the representations of its neighboring nodes; (3) Readout the final representation of a single node, a set of nodes, or the entire node set as required by the task. Under the above framework, researchers have proposed many  GNN architectures~\cite{kipf2016semi,hamilton2017inductive,velivckovic2017graph,gilmer2017neural,xu2018powerful,ying2018hierarchical,zhang2018end}. Interested readers may refer to tutorials on GNNs for further details~\cite{battaglia2018relational,hamilton2017representation}. 

Despite the success of GNNs, their representation power in representation learning is limited~\cite{xu2018powerful}. Recent works proved that the representation power of GNNs that follow the above framework is bounded by the 1-WL test~\cite{weisfeiler1968reduction,xu2018powerful,morris2019weisfeiler} (We shall refer to these GNNs as WLGNNs). Concretely, WLGNNs yield identical vector representations for any subgraph structure that the 1-WL test cannot distinguish. Consider an extreme case: If node attributes are all nodes are
the same, then for any node in a $r$-regular graph GNN will output identical representation. %though they may not be structurally equivalent (formally defined in Section~\ref{sec:prem}). 
Such an issue becomes even worse when WLGNNs are used to extract representations of node sets, \textit{e.g.}, node-pairs for link prediction (Fig.~\ref{fig:WLGNN} (a)). A few works have been recently proposed to improve the power of WLGNNs~\cite{maron2019provably}. However, they either focus on building theory only for entire-graph representations~\cite{morris2019weisfeiler,murphy2019relational,maron2018invariant,chen2019equivalence,maron2019provably}, or show empirical success using heuristic methods without strong theoretical characterization~\cite{zhang2018link,you2019position,chen2019path,maziarka2020molecule,klicpera2019diffusion,chien2020joint}. We review these methods in detail in Section~\ref{sec:related-works}. 

Here we address the limitations of WLGNNs and propose and mathematically analyze a new class of node features, termed {\em Distance Encoding (DE)}. DE comes with both theoretical guarantees and empirical efficiency. Given a node set $S$ %\jure{whatever letter you use to denote the set}
whose structural representation is to be learnt, for every node $u$ in the graph %\jure{letter to denote a node}
DE is defined as a mapping of a set of landing probabilities of random walks from each node of the set $S$ to node $u$. DE may use measures such as shortest path distance (SPD) and generalized PageRank scores~\cite{li2019optimizing}. DE can be combined with any GNN architecture in simple but effective ways: First, we propose \proj %DE-GNN \jure{Change DEGNN to DE-GNN. Define a macro for it} 
that utilizes DE as an extra node feature. We further enhance \proj by allowing DE to control the message aggregation procedure of WLGNNs, which yields another model \projA. %DEA-GNN \jure{same here: DEAGNN --> DEA-GNN}. 
Since DE purely depends on the graph structure and is independent of node identifiers, DE also provides inductive and generalization capability. 

\begin{figure}[t]
\begin{minipage}{0.36\textwidth}
\flushright
{\includegraphics[trim={1.5cm 7.5cm 14cm 6cm},clip,width=0.9\textwidth]{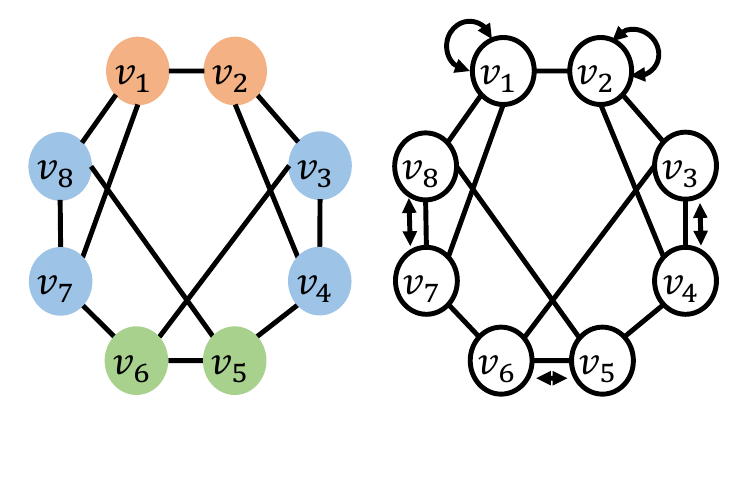}}
\end{minipage}
\begin{minipage}{0.63\textwidth}
\hfill\vline\hfill
\centering
\begin{tabular}{l}
\hline 
\textbf{WLGNN-$p$ to represent $\mathcal{T} = (S, \mathbf{A})$, $|S|=p$} \\
\hline
Initialize: For all $v\in V$, $h_v^{(0)} = \mathbf{A}_{vv}$  \\
For layers $l = 0, 1, ..., L-1$ and all $v\in V$, do: \\
%\quad $m_{vu}^{(l)} = f_1(h_v^{(l)},h_u^{(l)}, \mathbf{A}_{vu})$, $\forall u\in \mathcal{N}_v$   \\
\quad $h_v^{(l+1)} = f_1(h_v^{(l)}, \text{AGG}(\{f_2(h_u^{(l)}, \mathbf{A}_{vu})\}_{u\in \mathcal{N}_v}))$ \\
Output $\Gamma(\mathcal{T}) = \text{AGG}(\{h_v^{(L)}\}_{v\in S})$ \\
\hline 
\end{tabular}
\end{minipage}
\vspace{-0.3cm}
\caption{\small{\textbf{(a)} 3-regular graph with 8 nodes. Briefly assume that all node attributes are the same and the nodes can only be distinguished based on their network structure. Then for all nodes, WLGNN will produce the same representation and thus fail to distinguish them. However, nodes with different colors should have different representations, as they are not structurally equivalent (or ``isomorphic'' as defined in Section~\ref{sec:prem}). Furthermore, WLGNN cannot distinguish all the node-pairs (\textit{e.g.}, $\{v_1,v_2\}$ vs $\{v_4,v_7\}$). However, if we use shortest-path-distances (SPDs) between nodes as features we can distinguish blue nodes from green and red nodes because there is another node with SPD$=3$ to a blue node of interest (\textit{e.g.}, SPD between $v_3$ and $v_8$), while all SPDs between other nodes to red/green nodes are less than 3. %To distinguish red and green nodes, one needs in-depth analysis of the computation graphs (trees) of GNNs (Figure~\ref{fig:subtrees} in Appendix~\ref{sec:apd-power}). 
Note that the structural equivalence between any two nodes of the same color can be obtained from the reflexivity of the graph while the equivalence between two vertically-aligned blue nodes can be further obtained from the node permutation shown in the right. \textbf{(b)} WLGNN algorithm to represent a node set $S$ of size $p$ --- $f_i(\cdot)$'s are arbitrary neural networks; AGG$(\cdot)$'s are set-pooling operators; $L$ is the number of layers. }}
\label{fig:WLGNN}
\vspace{-0.3cm}
\end{figure}

\begin{textblock*}{1cm}(3.9cm,3.55cm) % {block width} (coords) 
   \textbf{(a)}
\end{textblock*}
\begin{textblock*}{1cm}(9.8cm,3.55cm) % {block width} (coords) 
   \textbf{(b)}
\end{textblock*}

We mathematically analyze the expressive power of \proj and \projA for structural representation learning. We prove that the two models are able to distinguish two non-isomorphic equally-sized node sets (including nodes, node-pairs, $\dots$, entire-graphs) that are embedded in almost all sparse regular graphs, where WLGNN always fails to distinguish them unless discriminatory node/edge attributes are available. %Note that this statement covers learning structural representations of nodes, node-pairs and graphs. 
We also prove that the two models are not more powerful than WLGNN when applied to distance regular graphs~\cite{brouwer2012distance}, which implies the limitation of DEs. However, we show that DE has an extra power to learn the structural representations of node-pairs over distance regular graphs~\cite{brouwer2012distance}. 

We experimentally evaluate \proj and \projA on three levels of tasks including node-structural-role classification (node-level), link prediction (node-pair-level), triangle prediction (node-triad-level). Our methods outperform WLGNN on all three tasks by up-to 15\% improvement in average accuracy. Our methods also outperform other baselines specifically designed for these tasks.

%The rest of the paper is organized as follows: Section~\ref{sec:prem} introduces the definition of structural representation learning and reviews WLGNNs and WL tests. Section~\ref{sec:PE} defines Distance Encoding and introduces how DE is used to improve the design of GNNs. Section~\ref{sec:related-works} reviews related works. Section~\ref{sec:exp} compares DE-assisted GNNs with other models. We show the proof of all the theoretical results to the appendix.

\vspace{-0.05cm}
\section{Preliminaries} \label{sec:prem}
\vspace{-0.05cm}
In this section we formally define the notion of structural representation and review how WLGNN learns structural representation and its relation to the 1-WL test. %Some notations are shared with~\cite{srinivasan2019equivalence}.
%In this section, we will formally define structural representation, review how WLGNN learns structural representation and its relation to 1-WL test. Some notations are shared with~\cite{srinivasan2019equivalence}.
\vspace{-0.1cm}
\subsection{Graph Representation Learning}
%\vspace{-0.15cm}
%\begin{definition}
%We consider a graph, directed or undirected, which can be represented as $G=(V, E, \mathbf{E})$, where $V$ is the node set with size $|V|=n$, $E\subseteq V\times V$ is the edge set, $\mathbf{X}$ is a mapping from $V$ to some node feature space $\mathcal{X}\subset \mathbb{R}^{k}$ and $\mathbf{X}_i$ is the feature of node $i$, $\mathbf{E}$ is a mapping from $E$ to some node feature space $\mathcal{E}\subset \mathbb{R}^{k'}$ and $\mathbf{E}_e$ is the feature of edge $e$. Equivalently, $\mathbf{X}\in \mathcal{X}^n$ and $\mathbf{E}\in \mathcal{E}^{n\times n}$.  
%\end{definition}

%\begin{definition}
%We consider a graph, directed or undirected, which can be represented as $G=(V, E, \mathbf{A}, \mathbf{X})$, where $V$ is the node set with size $|V|=n$, $E\subseteq V\times V$ is the edge set, $\mathbf{A}$ is in the edge feature space $\mathcal{A}\subset \mathbb{R}^{n\times n\times k}$, and $\mathbf{X}$ is in the node feature space $\mathcal{X}\subset \mathbb{R}^{n\times k'}$. In practice, graphs are generally sparse, \textit{i.e.}, $E \subset V\times V$, for which we introduce $A\in \{0,1\}^{n\times n}$ to denote the adjacency matrix of $G$ such that $A_{uv} = 1$ if and only if $(u,v)\in E$. Note that $A$ can be viewed as one slice of the edge feature tensor $\mathbf{A}$. If there are no edge features or node features, we let $\mathbf{A} = A$ with $k=1$ or $\mathbf{X}$ be all-one matrix correspondingly. 
%\end{definition}
\begin{definition}
We consider an undirected graph which can be represented as $G=(V, E, \mathbf{A})$, where $V=[n]$ is the node set, $E\subseteq V\times V$ is the edge set, and $\mathbf{A}$ contains all features in the space $\mathcal{A}\subset \mathbb{R}^{n\times n\times k}$. %The diagonal component $\mathbf{A}_{vv\cdot}$ correspond to the features of node $v$, while the off-diagonal component, $\mathbf{A}_{vu\cdot}$, corresponds to the feature of edge $(v,u)\in E$.
Its diagonal component, $\mathbf{A}_{vv\cdot}$, denotes the node attributes of node $v(\in V)$, while its off-diagonal component in $\mathbf{A}_{vu\cdot}$ denotes the node-pair attributes of $(v,u)$.  We set $\mathbf{A}_{vu\cdot}$ as all zeros if $(v,u)\not\in E$. In practice, graphs are usually sparse, \textit{i.e.}, $|E|\ll n^2$. We introduce $A\in \{0,1\}^{n\times n}$ to denote the adjacency matrix of $G$ such that $A_{uv} = 1$ iff $(u,v)\in E$. Note that $A$ can be also viewed as one slice of the feature tensor $\mathbf{A}$. If no node/edge attributes are available, we let $\mathbf{A} = A$. %with $k=1$. %or $\mathbf{X}$ be all-one matrix correspondingly. 
\end{definition}

\begin{definition} \label{def:nodeset}
The node permutation denoted by $\pi$ is a bijective mapping from $V$ to $V$. All possible $\pi$'s are collected in the permutation group $\Pi_n$. We denote $\pi$ acting on a subset $S$ of $V$ as $\pi(S) = \{\pi(i)| i\in S\}$. We further define $\pi(\mathbf{A})_{uv\cdot} = \mathbf{A}_{\pi^{-1}(u)\pi^{-1}(v)\cdot}$ for any $(u,v)\in V\times V$.
%$\mathbf{X}$ and $\mathbf{E}$ as $(\pi(\mathbf{X}))_v = \mathbf{X}_{\pi(v)}$ and $(\pi(\mathbf{E}))_e = \mathbf{E}_{\pi(e)}$ respectively.
%Given a $p+1$-mode tensor $\mathcal{M}\in \mathbb{R}^{n\times \cdots \times n\times k}$ where first $p$ modes are $n$-dimensional,  we define $\pi$ acting on $\mathcal{M}$ as $\pi(\mathcal{M}) = (\mathcal{M}_{\pi(i_1)...\pi(i_p)\cdot})_{(i_1, i_2,...,i_p)\in [n]\times \cdots \times[n]}$. We may also define $\pi$ acting on a subset of $V$, say $S$, as $\pi(S) = \{\pi(i)| i\in S\}$. 
%and $\mathbf{X}\in \mathbb{R}^{n\times k'}$ as  $\pi(\mathbf{A}) = (\mathbf{A}_{\pi(u)\pi(v)\cdot})_{(u,v)\in V\times V}$ and $\pi(\mathbf{X}) = (\mathbf{X}_{\pi(u)\cdot})_{u\in V}$ 
\end{definition}
\begin{definition}\label{def:tuple} \vspace{-0.05cm}
Denote all $p$-sized subsets $S$ of $V$ as $S \in \mathcal{P}_p(V)$ and define the space $\Omega_p = \mathcal{P}_p(V)\times \mathcal{A}$. For two tuples $\mathcal{T}_1 = (S^{(1)}, \mathbf{A}^{(1)})$ and $\mathcal{T}_2 = (S^{(2)}, \mathbf{A}^{(2)})$ in $\Omega_p$, we call that that they are \emph{isomorphic} (otherwise \emph{non-isomorphic}), if $\exists \pi \in \Pi_n$ such that $S^{(1)} = \pi(S^{(2)})$ and $\mathbf{A}^{(1)}=\pi(\mathbf{A}^{(2)})$. %For brevity, we use $\mathcal{T}_1 \sim \mathcal{T}_2$ to denote isomorphic relation. 
\end{definition}
\begin{definition}
A function $f$ defined on $\Omega_p$ is \emph{invariant} if $\forall \pi\in \Pi_n$, $f(S, \mathbf{A})  = f(\pi(S), \pi(\mathbf{A}))$. 
\end{definition}
\begin{definition}
The \emph{structural representation} of \emph{a tuple} $(S, \mathbf{A})$ is an invariant function $\Gamma(\cdot): \Omega_p \rightarrow \mathbb{R}^d$ where $d$ is the dimension of representation. % such that $\Gamma$ satisfies permutation invariance, \textit{i.e.}, $\Gamma(S, \mathbf{A})  = \Gamma(\pi(S), \pi(\mathbf{A}))$ for any $\pi\in \Pi_n$. 
Therefore, if two tuples are isomorphic, they should have the same structural representation. 
\end{definition}
\vspace{-0.1cm}
The invariant property is critical for the inductive and generalization capability as it frees structural representations from node identifiers and effectively reduces the problem dimension by incorporating the symmetry of the parameter space~\cite{maron2018invariant} (\textit{e.g.}, the convolutional layers in GCN~\cite{kipf2016semi}). The invariant property also implies that structural representations do not allow encoding the absolute positions of $S$ in the graph. 

The definition of structural representation is very general. Suppose we set two node sets $S^{(1)}$, $S^{(2)}$ as two single nodes and set two graph structures $\mathbf{A}^{(1)}$ and $\mathbf{A}^{(2)}$ as the ego-networks around these two nodes. Then, the definition of structural representation provides a mathematical characterization the concept ``structural roles'' of nodes~\cite{henderson2012rolx,ribeiro2017struc2vec,donnat2018learning}, where two far-away nodes could have the same structural roles (representations) as long as their ego-networks have the same structure. 

%The invariant property of structural representations allows two far-away nodes having the same structural representations as long as their ego-networks are the 

%Suppose two nodes set  as $\mathbb{A}$ can be viewed the contextual structures, while the later one mainly referred to node representations. 

%the invariant property similar philosophy of 

%there could be two far-away node sets $S_1$, $S_2$ having the same structural representations.

Note that depending on the application one can vary/select the size $p$ of the node set $S$. For example, when $p=1$ then we are in the regime of node classification, $p=2$ is link prediction, and when $S=V$, structural representations reduce to entire graph representations. However, in this work we will primarily focus on the case that the node set $S$ has a fixed and small size $p$, where $p$ does not depend on the graph size $n$. Although Corollary~\ref{col:graphpower} later shows the potential of our techniques on learning the entire graph representations, this is not the main focus of our work here. We expect the techniques proposed here can be further used for entire-graph representations while we leave the detailed investigation for future work. 

%Structural representation is also very general  to include entire graph representation, as graph representation specifies the set $S$ as $V$. However, this generalized concept is important as it includes structural representations of node sets with arbitrary sizes, which are used in many applications. \textbf{For the theory established in this work, we focus on the case that the node set $S$ has a fixed and small size $p$, where $p$ does not depend on $n$. In the experiments, we evaluate the cases when $p\in\{1,2,3\}$.} However, these structural representations can be further used for graph representation as we will discuss later.  %However, structural representations of subsets with size $p$ can be naturally leveraged to obtain graph representation by mapping the set $\{\Gamma(S, \mathbf{A}): |S|=p\}$  .}
%which are useful in many applications of GNNs such as 

Although structural representation defines a more general concept, it shares some properties with traditional entire-graph representation. For example, the universal approximation theorem regarding entire-graph representation~\cite{chen2019equivalence} can be directly generalized to the case of structural representations: %We summarize it as follows. %Note that we follow the statement of \cite{chen2019equivalence} by assuming the representation dimension $d=1$. The $d>1$ case is interesting but out of the scope of this work.  %This result refines former universal approximation results in~\cite{maron2018invariant,maron2019universality} in a more practical way by avoiding polynomial-order tensor operation. 

\begin{theorem}\label{thm:univapp}
If structural representations $\Gamma$ are different over any two non-isomorphic tuples $\mathcal{T}_1$ and $\mathcal{T}_2$ in $\Omega_p$, then for any invariant function $f: \Omega_p \rightarrow \mathbb{R}$, $f$ can be universally approximated by feeding $\Gamma$ into a 3-layer feed-forward neural network with ReLu as the activation function, as long as (1) the feature space $\mathcal{A}$ is compact and (2) $f(S, \cdot)$ is continuous over $\mathcal{A}$ for any $S\in \mathcal{P}_p(V)$. \vspace{-0.1cm}

%\begin{itemize}
%\item The feature space $\mathcal{A}$ is compact.
%\item $f(S, \cdot)$ is continuous over $\mathcal{A}$ for any $S\in \mathcal{P}_p(V)$. 
%\end{itemize}
\end{theorem}

Theorem~\ref{thm:univapp} formally establishes the relation between learning structural representations and distinguishing non-isomorphic structures, \textit{i.e.}, $\Gamma(\mathcal{T}_1) \neq \Gamma(\mathcal{T}_2)$ iff $\mathcal{T}_1$ and $\mathcal{T}_2$ are non-isomorphic. However, no polynomial algorithm has been found to distinguish even just non-isomorphic entire graphs ($S=V$) without node/edge attributes ($\mathbf{A}=A$), which is known as the graph isomorphism problem~\cite{babai2016graph}. In this work, we will use the range of non-isomorphic structures that GNNs can distinguish to characterize their expressive power for graph representation learning. %Srinivasan \& Ribeiro~\cite{srinivasan2019equivalence} term the structural representation $\Gamma$ that satisfies the above condition as a \emph{most-expressive} structural representation though they do not prove the universal approximation results. %Later, we keep using this terminology. 

\subsection{Weisfeiler-Lehman Tests and WLGNN for Structural Representation Learning} \label{sec:WLGNN}
%\vspace{-0.05cm}
Weisfeiler-Lehman test (WL-test) is a family of very successful algorithmic heuristics used in graph isomorphism problems~\cite{weisfeiler1968reduction}. 1-WL test, the simplest one among this family, starts with coloring nodes with their degrees, then it  iteratively aggregates the colors of nodes and their neighborhoods, and hashes the aggregated colors into unique new colors. The coloring procedure finally converges to some static node-color configuration. Here a node-color configuration is a multiset that records the types of colors and their numbers. Different node-color configurations indicate two graphs are non-isomorphic while the reverse statement is not always true. %1-WL test may leverage the sparsity of graphs and is able to test most graphs in practice. %Hence, 1-WL test inspires many machine learning approaches to extract graph representations, including kernel-based approaches~\cite{shervashidze2011weisfeiler} and GNN-based approaches~\cite{xu2018powerful,morris2019weisfeiler}.

More than the graph isomorphism problem, the node colors obtained by the $1$-WL test naturally provide a test of structural isomorphism. Consider two tuples $\mathcal{T}_1 = (S^{(1)}, \mathbf{A}^{(1)})$ and $\mathcal{T}_2 = (S^{(2)}, \mathbf{A}^{(2)})$ according to Definition~\ref{def:tuple}. We temporarily ignore node/edge attributes for simplicity, so $\mathbf{A}^{(1)}, \mathbf{A}^{(2)}$ reduce to adjacent matrices. It is easy to show that different node-color configurations of nodes in $S^{(1)}$ and in $S^{(2)}$ obtained by the 1-WL test also indicate that $\mathcal{T}_1$ and $\mathcal{T}_2$ are not isomorphic. %Due to this observation, mimicking 1-WL test can be used to learn structural representations~\cite{hamilton2017inductive}. %This linyields the GraphSAGE architecture~\cite{hamilton2017inductive}. %We summarize WLGNN procedure in the left of Figure\ref{fig:WLGNN }, which follows MPNN in~\cite{gilmer2017neural}. 

WLGNNs refer to those GNNs that mimic the 1-WL test to learn structural representation, which is summarized in Fig.~\ref{fig:WLGNN} (b). It covers many well-known GNNs of which difference may appear in the implementation of neural networks $f_i$ and set-poolings AGG$(\cdot)$ (Fig.~\ref{fig:WLGNN} (b)), including GCN~\cite{kipf2016semi}, GraphSAGE~\cite{hamilton2017inductive}, GAT~\cite{velivckovic2017graph}, MPNN~\cite{gilmer2017neural}, GIN~\cite{xu2018powerful} and many others~\cite{maron2018invariant}. Note that we use WLGNN-$p$ to denote the WLGNN that is to learn structural representations of node sets $S$ with size $|S|=p$. One may directly choose $S=V$ to obtain the entire-graph representation. Theoretically, the structural representation power of WLGNN-$p$ is provably bounded by the 1-WL test~\cite{xu2018powerful}. The result can be also generalized to the case of structural representations as follows. 

%The GNN model, GraphSAGE~\cite{hamilton2017inductive}, has been proposed to mimic 1-WL test to learn structural representations~\cite{hamilton2017inductive}. Many other well-know GNN models, including GCN~\cite{kipf2016semi}, GAT~\cite{velivckovic2017graph}, MPNN~\cite{gilmer2017neural}, can be viewed as different implementations of hashing functions imposed on the labels aggregated from neighbors. We uniformly call them WLGNN and summarize it in the left of Fig.~\ref{fig:WLGNN }. Note that the output of WLGNN is only the structural representation of one set $S$. This structural representation can be further used to represent the whole graph by using an additional aggregation as shown in Fig.~\ref{fig:WLGNN }. GNN models for graph representation learning essentially follow the procedure by first learning structural representation of each node and then aggregating all obtained node representations~\cite{battaglia2018relational}. Theoretically, the graph representation power of 1-WL GNNs is provably bounded by 1-WL test~\cite{xu2018powerful}. %Roughly speaking, by specifying the aggregation operations (AGG, Readout) as injective mappings, both frameworks could be as powerful as 1 and 2-dim WL tests to distinguish isomorphic graphs~\cite{cai1992optimal}. 

\begin{theorem}\label{thm:1-WL-limit}
Consider two tuples $\mathcal{T}_1 = (S^{(1)}, \mathbf{A}^{(1)})$ and $\mathcal{T}_2 = (S^{(2)}, \mathbf{A}^{(2)})$ in $\Omega_p$. If $\mathcal{T}_1, \,\mathcal{T}_2$ cannot be distinguished by the 1-WL test, then the corresponding outputs of WLGNN-$p$  satisfy $\Gamma(\mathcal{T}_1) = \Gamma(\mathcal{T}_2)$. On the other side, if they can be distinguished by the 1-WL test and we suppose aggregation operations (AGG) and neural networks $f_1,\,f_2$ are all injective mappings, then with a large enough number of layers $L$, the outputs of WLGNN-$p$ also satisfy $\Gamma(\mathcal{T}_1)\neq \Gamma(\mathcal{T}_2)$. 
\end{theorem}
\vspace{-0.1cm}
Because of Theorem~\ref{thm:1-WL-limit}, WLGNN inherits the limitation of the 1-WL test. For example, WLGNN cannot distinguish two equal-sized node sets  in all $r$-regular graphs (unless node/edge features are discriminatory). Here, a $r$-regular graph means that all its nodes have degree $r$. Therefore, researchers have recently focused on designing GNNs with expressive power greater than the 1-WL test. Here we will improve the power of GNNs by developing a general class of structural features. %DE can be properly leveraged to improve the expressive power of WLGNNs.  %We illustrate this point in Figure.~\ref{fig:WLGNN} LEFT. Actually, it cannot distinguish any two nodes of $r$-regular graphs and therefore cannot distinguish any two equal-sized subsets of nodes of $r$-regular graphs.  

\vspace{-0.2cm}
\section{Distance Encoding and Its Power}\label{sec:PE}
\vspace{-0.15cm}
\subsection{Distance Encoding}
%\textbf{Distance Encoding (DE).} 

Suppose we aim to \emph{learn the structural representation of the target node set $S$.} Intuitively, our proposed DE will then encode the distance from $S$ to any other node $u$. We define DE as  follows: %Recall that $\mathcal{P}_p(V)$ is the $p$-sized power set of $V$.
\begin{definition} \label{def:de}
Given a target set of nodes $S\in 2^V\backslash\emptyset$ of $G$ with the adjacency matrix $A$, we denote \emph{distance encoding} as a function $\zeta(\cdot|S, A): V \rightarrow \mathbb{R}^{k}$. $\zeta$ should also be permutation invariant, \textit{i.e.}, $\zeta(u|S, A)=\zeta(\pi(u)|\pi(S), \pi(A))$ for all $u\in V$ and $\pi\in \Pi_n$. Then we denote DEs \textit{w.r.t.} the size of $S$ and call them as \emph{DE-$p$} if $|S|=p$. 
\end{definition}
Later we use $\zeta(u|S)$ for brevity where $A$ could be inferred from the context. %Note that general representation of $\zeta(u|S)$ goes back to characterize all permutation invariant functional over $A$, which could be as complex as $\Gamma(S, A)$. However, the motivation of PE is to increase expressive power of WLGNNs in a scalable way. Therefore, 
For simplicity, we choose DE as a set aggregation (\textit{e.g.}, the sum-pooling) of DEs between nodes $u,\,v$ where $v\in S$: \vspace{-0.1cm}
\begin{align} \label{eq:node2setPE}
\zeta(u|S) = \text{AGG}(\{\zeta(u|v) | v\in S\})
\end{align}
 More complicated DE may be used while this simple design can be efficiently implemented and achieves good empirical performance. Then, the problem reduces to choosing a proper $\zeta(u|v)$. Again for simplicity, we consider the following class of functions that is based on the mapping of a list of landing probabilities of random walks from $v$ to $u$ over the graph, \textit{i.e.}, 
\begin{align} \label{eq:node2nodePE}
\zeta(u|v) = f_3(\ell_{uv}), \; \ell_{uv} = ((W)_{uv}, (W^2)_{uv}, ..., (W^{k})_{uv},...)
\end{align}
where  $W=AD^{-1}$ is the random walk matrix, $f_3$ may be simply designed by some heuristics or be parameterized and learnt as a feed-forward neural network. In practice, a finite length of $\ell_{vu}$, say 3,4, is enough. 
Note that Eq. \eqref{eq:node2nodePE} covers many important distance measures. First, setting $f_3(\ell_{uv})$ as the first non-zero position in $\ell_{uv}$ gives the \emph{shortest-path-distance (SPD)} from $v$ to $u$. We denote this specific choice as $\zeta_{spd}(u|v)$. Second, one may also use \emph{generalized PageRank scores}~\cite{li2019optimizing}:
\begin{align}\label{eq:gpr}
\zeta_{gpr}(u|v) = \sum_{k\geq 1} \gamma_k (W^k)_{uv} = (\sum_{k\geq 1} \gamma_k W^k)_{uv}, \quad \gamma_k\in \mathbb{R},\,\text{for all $k\in \mathbb{N}$ }.
\end{align}
%which corresponds to \emph{generalized PageRank scores}~\cite{li2019optimizing}. %$\zeta_{spd}(u|v)$ and $\zeta_{gpr}(u|v)$ are used in many GNN models. We will discuss them in Section~\ref{sec:related}.
\begin{table}[h]
\vspace{-1.2cm}
\end{table}

Note that the permutation invariant property of DE is beneficial for inductive learning, which fundamentally differs from positional node embeddings such as node2vec~\cite{grover2016node2vec} or one-hot node identifiers. In the rest of this work, we will show that DE improves the expressive power of GNNs in both theory and practice. In Section~\ref{sec:PEGNN}, we use DE as extra node features. We term this model as \proj, and theoretically demonstrate its expressive power. In the next subsection, we further use DE-1 to control the aggregation procedure of WLGNN. We term this model as \projA and extend our theory there. 
%which unifies many recently proposed GNN models~\cite{chen2019path,maziarka2020molecule,klicpera2019diffusion,abu2019mixhop,zhang2018link}. Our theory also can be extended to those cases. 

\vspace{-0.1cm}
\subsection{\proj --- Distance Encodings as Node Features}\label{sec:PEGNN}
\vspace{-0.1cm}
%We have introduced PE. The next question is how to utilize PEs to improve the expressive power of GNNs for structural representation? In this subsection, we introduce  a simple way, , to incorporate WLGNN with PEand theoretically demonstrate its expressive power. 
%
%first consider a simple way to incorporate WLGNN with PE and provably demonstrate its expressive power for structural representation. In the next subsection, we will discuss other ways to use PEs, which unifies many recently proposed GNN models~\cite{chen2019path,maziarka2020molecule,klicpera2019diffusion,abu2019mixhop,zhang2018link}. Our theory in this subsection can be directly extended to those cases. 

DE can be used as extra node features. Specifically, we improve WLGNNs by setting $h_v^{(0)} = \mathbf{A}_{vv} \oplus \zeta(v|S) $ where $\oplus$ is the concatenation. We call the obtained model \proj. We similarly use \textbf{\proj-$p$} to specify the case when $|S|=p$. For simplicity, we give the following definition.

\begin{definition}
\proj is called \emph{proper} if $f_1, f_2$, $\text{AGG}$s in the WLGNN (Fig.~\ref{fig:WLGNN} (b)), and $\text{AGG}$ in Eq.~\eqref{eq:node2setPE}, $f_3$ in Eq.~\eqref{eq:node2nodePE} are injective mappings as long as the input features are all countable.
\end{definition}
We know that a proper \proj exists because of the universal approximation theorem of feed-forward networks (to construct $f_i$, $i\in\{1,2,3\}$) and Deep Sets~\cite{zaheer2017deep} (to construct AGGs). 
 
 \vspace{-0.1cm}
\subsubsection{The Expressive Power of \proj}
 \vspace{-0.1cm}

Next, we demonstrate the power of \proj to distinguish structural representations. Recall that the fundamental limit of WLGNN is the 1-WL test for structural representation (Theorem~\ref{thm:1-WL-limit}). One important class of graphs that cannot be distinguished by the 1-WL test are regular graphs (although, in practice, node/edge attributes may help diminish such difficulty by breaking the symmetry). In theory, we may consider the most difficult case by assuming that no node/edge attributes are available. In the following, our main theorem shows that even in the most difficult case, \proj is able to distinguish two equal-sized node sets that are embedded in almost all $r$-regular graphs. One example where \proj using $\zeta_{spd}(\cdot)$ (SPD) is shown in Fig.~\ref{fig:WLGNN} (a): The blue nodes can be easily distinguished from the green or red nodes as SPD$=3$ may appear between two nodes when a blue node is the node set of interest, while all SPDs from other nodes to red and green nodes are less than 3. Actually, \proj-1 with SPD may also distinguish the red or green nodes by investigating its procedure in details (Fig.~\ref{fig:subtrees} in Appendix). 

\begin{theorem}\label{thm:power}
Given two equal-sized sets $S^{(1)}, S^{(2)}\subset V$, $|S^{(1)}| = |S^{(2)}|=p$. Consider two tuples $\mathcal{T}^{(1)}=(S^{(1)}, \mathbf{A}^{(1)})$ and $\mathcal{T}^{(2)}=(S^{(2)}, \mathbf{A}^{(2)})$ in the most difficult setting where features $\mathbf{A}^{(1)}$ and $\mathbf{A}^{(2)}$ are only different in graph structures specified by $A^{(1)}$ and $A^{(2)}$ respectively. Suppose $A^{(1)}$ and $A^{(2)}$ are uniformly independently sampled from all r-regular graphs over $V$ where $3\leq r< (2\log n)^{1/2}$. Then, for any small constant $\epsilon >0$, within $L \leq \lceil (\frac{1}{2} + \epsilon)\frac{\log n}{\log(r-1)}\rceil $ layers, there exists a proper \proj-$p$ using DEs $\zeta(u|S^{(1)}),\,\zeta(u|S^{(2)})$ for all $u\in V$, such that with probability $1-o(n^{-1})$, the outputs $\Gamma(\mathcal{T}^{(1)})\neq \Gamma(\mathcal{T}^{(2)})$. Specifically, $f_3$ can be simply chosen as SPD, \textit{i.e.}, $\zeta(u|v) = \zeta_{spd}(u|v)$. The big-O notations here and later are w.r.t. $n$.
\end{theorem}
\begin{remark}\label{rmk:same-graph}
In some cases, we are to learn representations of structures that lie in a single large graph, \textit{i.e.}, $A^{(1)}=A^{(2)}$. Actually, there is no additional difficulty to extend the proof of Theorem~\ref{thm:power} to this setting as long as $A^{(1)}(=A^{(2)})$ is uniformly sampled from all r-regular graphs and $S^{(1)}\cap S^{(2)}=\emptyset$. The underlying intuition is that for large $n$, the local subgraphs (within $L$-hop neighbors) around two non-overlapping fixed sets $S^{(1)}, S^{(2)}$ are almost independent. Simulation results to validate the single node case ($p=1$) of Theorem~\ref{thm:power} and Remark~\ref{rmk:same-graph} are shown in Fig.~\ref{fig:2PE} (a). Discussion on the node sets over irregular graphs that WLGNN cannot distinguish is in Appendix.~\ref{sec:irregular}.
\end{remark}

Actually, the power of structural representations of small node sets can be used to further characterize the power of entire graph representations. Consider that we directly aggregate all the representations of nodes of a graph output by \proj-1 via set-pooling as the graph representation, which is a common strategy adopted to learn graph representation via WLGNN-1~\cite{gilmer2017neural,xu2018powerful,zhang2018end}. So how about the power of \proj-1 to represent graphs? To answer this question, suppose two $n$-sized $r$-regular graphs $\mathbf{A}^{(1)}$ and $\mathbf{A}^{(2)}$ satisfy the condition in Theorem~\ref{thm:power}. Then, by using a union bound, Theorem~\ref{thm:power} indicates that for a node $v\in V$, its representation $\Gamma((v, \mathbf{A}^{(1)})) \not\in \{\Gamma((u, \mathbf{A}^{(2)}))| u\in V\}$ with probability $1 - no(n^{-1}) = 1 - o(1)$. Therefore, these two graphs $\mathbf{A}^{(1)}$ and $\mathbf{A}^{(2)}$ can be distinguished via \proj-1 with high probability. We formally state this result in the following corollary.

\begin{corollary}\label{col:graphpower}
Suppose two graphs are uniformly independently sampled from all $n$-sized $r$-regular graphs over $V$ where $3\leq r< (2\log n)^{1/2}$. Then, within $L \leq \lceil (\frac{1}{2} + \epsilon)\frac{\log n}{\log(r-1)} \rceil$ layers, \proj-1 can distinguish these two graphs with probability $1 - o(1)$ by being concatenated with injective set-pooling over all the representations of nodes.
\end{corollary}

\begin{figure}[t]
\begin{minipage}{0.335\textwidth}
\flushright
\includegraphics[trim={0.5cm 0cm 1.8cm 1.1cm},clip,width=0.94\textwidth]{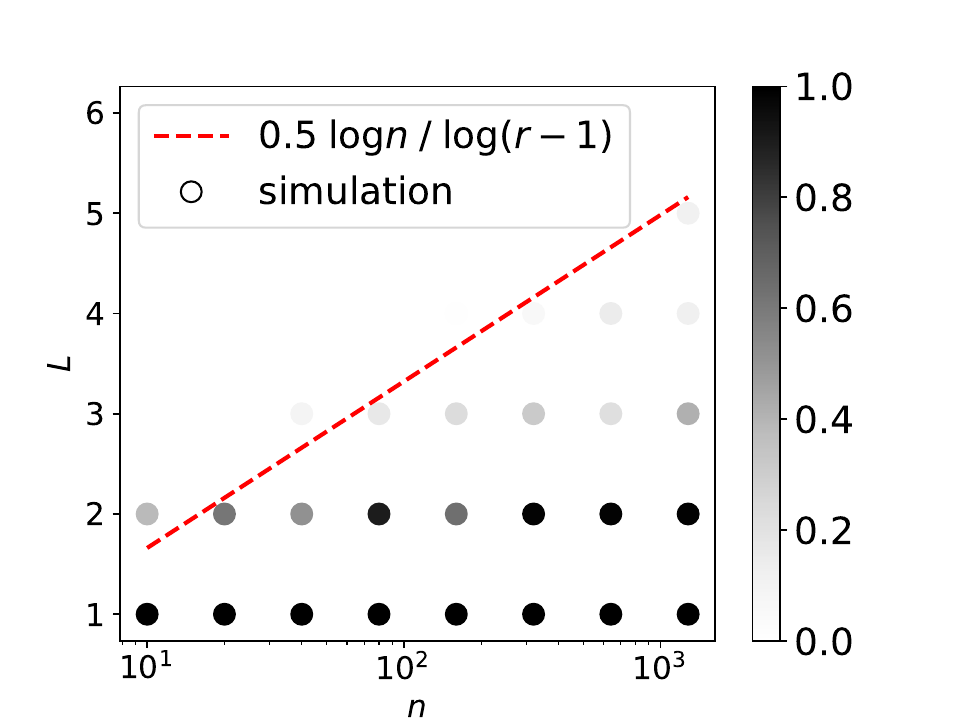}
\end{minipage}
\begin{minipage}{0.664\textwidth}
\hfill\vline\hfill\hfill
\includegraphics[trim={2.5cm 12cm 10cm 4cm},clip,width=0.935\textwidth]{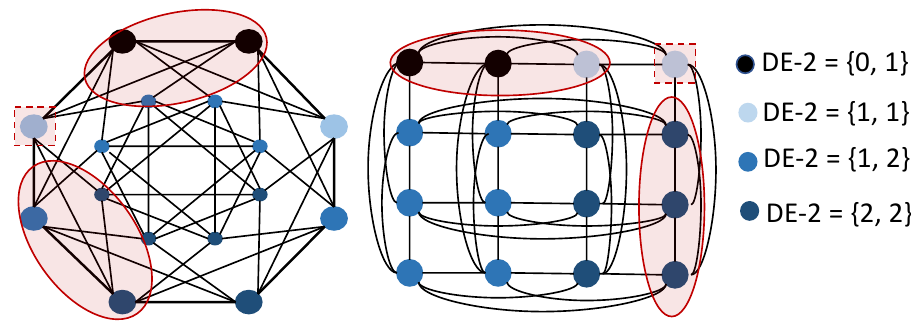}
\end{minipage}
\vspace{-0.25cm}
\caption{\small{\textbf{(a)} Simulation to validate Theorem~\ref{thm:power}. We uniformly at random generate $10^4/n$ 3-regular graphs and compare node representations output by a randomly initialized but untrained \proj-1 with $L$ layers, $L\leq 6$. All the nodes in these graphs are considered and thus for each $n$, there are $10^4$ nodes from the same or different graphs. For any two nodes $u, v$, if $\|h_u^{(L)}-h_v^{(L)}\|_2$ is greater than machine accuracy, they are regarded to be distinguishable. The colors of the scatter plot indicate the portion of two nodes that are not distinguishable by \proj-1. The red line is boundary predicted by our theory, which well matches the simulation.  \textbf{(b)} The power of DE-2. The left is the Shrikhande graph while the right is the $4\times4$ Rook's graph. \proj-1 assigns all nodes with the same representation. \proj-2 may distinguish the structures by learning representations of node-pairs (edges)---the node-pairs colored by black. Each node is colored with its DE-2 that is a set of SPDs to either node in the target edges (Eq.~\eqref{eq:node2setPE}). Note the neighbors of nodes with DE-2$=\{1,1\}$ (covered by dashed boxes) that are highlighted by red ellipses. As these neighbors have different DE-2's, after one layer of \proj-2, the intermediate representations of nodes with DE-2$=\{1,1\}$ are different between these two graphs. Using another layer, \proj-2 can distinguish the representations of two target edges.}}
\vspace{-0.1cm}
\label{fig:2PE}
\end{figure}

\begin{textblock*}{1cm}(3.7cm,4.0cm) % {block width} (coords) 
   \textbf{(a)}
\end{textblock*}
\begin{textblock*}{1cm}(8.9cm,4.0cm) % {block width} (coords) 
   \textbf{(b)}
\end{textblock*} 

One insightful observation is that the structural representation with a small set $S$ may become easier to be learnt than the entire graph representation as the knowledge of the node set $S$ can be viewed as a piece of side information. Models can be built to leverage such information, while when to compute the entire graph representation ($S=V$), the model loses such side information. This could be a reason why the successful probability to learn structural representation of a small node set (Theorem~\ref{thm:power}) is higher than to that of an entire graph in (Corollary~\ref{col:graphpower}), though our derivation of the probabilistic bounds is not tight. Note that DE provides a convenient way to effectively capture such side information. One naive way to leverage the information of $S$ is to simply annotate the nodes in $S$ with binary encoding 1 and those out of $S$ with 0. This is obviously a special case of DE but it is not as powerful as the general DE (even compared to the special case SPD). Think about setting the target node set in Eq.\eqref{eq:node2setPE} as the entire graph ($S=V$), where annotating the nodes in/out of $S$ does not improve the representation power of WLGNN because all nodes are annotated as $1$. However, DE still holds extra representation power: For example, we want to distinguish a graph with two disconnected 3-circles and a graph with a 6-circle. These two graphs generate different SPDs between nodes. 

%More than just capture such side information, DE also provides a sufficient way to leverage it. To illustrate this point, consider One naive way to use the set $S$: We simply annotate the nodes in $S$ with binary encoding 1 and those out of $S$ with 0. This is obviously a special case of DE but it is not as powerful as DE (even compared to the special case SPD): Think about learning the entire graph representation ($S=V$), where annotating the nodes in/out of $S$ does not improve the representation power of WLGNN, while Corollary~\ref{col:graphpower} implies that DE (or the special case SPD) can provide extra representation power. 

\subsubsection{The Limitation of \proj}

Next, we show the limitation of \proj. We prove that over a subclass of regular graphs, distance regular graphs (DRG), DE-1 is useless for structural representation. We provide the definition of DRG as follows while we refer interested readers to check more properties of DRGs in~\cite{brouwer2012distance}.

\begin{definition}\label{def:DRG}
A \emph{distance regular graph} is a regular graph such that for any two nodes $v, u\in V$, the number of vertices $w$ s.t. SPD$(w, v)=i$ and SPD$(w, u)=j$, only depends on $i$, $j$ and SPD$(v, u)$. 
\end{definition}
The Shrikhande graph and the $4\times4$ Rook's graph are two non-isomorphic DRGs shown in Fig.~\ref{fig:2PE} (b) (We temporarily ignore the nodes colors which will be discussed later). For simplicity, we only consider connected DRGs that can be characterized by arrays of integers termed intersection arrays. 
\begin{definition}
The \emph{intersection array} of a connected DRG with diameter $\triangle$ is an array of integers $\{b_0, b_1, ..., b_{\triangle-1};c_1, c_2, ..., c_{\triangle}\}$ such that for any node pair $(u, v)\in V\times V$ that satisfies SPD$(v, u)=j$, $b_j$ is the number of nodes $w$ that are neighbors of $v$ and satisfy SPD$(w, u)=j+1$, and $c_j$ is the number of nodes $w$ that are neighbors of $v$ and satisfy SPD$(w, u)=j-1$.
\end{definition}

It is not hard to show that the two DRGs in Fig.~\ref{fig:2PE} (b) share the same intersecion array $\{6, 3;1, 2\}$. The following theorem shows that over distance regular graphs, \proj-1 requires discriminatory node/edge attributes to distinguish structures, which indicates the limitation of DE-1.

\begin{theorem}\label{thm:limit}
 Given any two nodes $v, u\in V$, consider two tuples $\mathcal{T}_1=(v, \mathbf{A}^{(1)})$ and $\mathcal{T}_2=(u, \mathbf{A}^{(2)})$ with graph structures $A^{(1)}$ and $A^{(2)}$ that correspond to two connected DRGs with a same intersection array. Then, \proj-1 must use discriminatory node/edge attributes to distinguish $\mathcal{T}_1$ and $\mathcal{T}_2$. \vspace{-0.1cm} %then there is also an WLGNN-1 without using PE that can distinguish $\mathcal{T}_1$ and $\mathcal{T}_2$.
\end{theorem}
Note Theorem~\ref{thm:limit} only works for node representations using DE-1. Therefore, \proj-1 may not associate distinguishable node representations in the two DRGs in Fig.~\ref{fig:2PE} (b). 

However, if we are to learn higher-order structural representations ($|S|\geq 2$) with DE-$p$ ($p\geq2$), \proj-p may have even stronger representation power. We illustrate this point by considering structural representations of two node-pairs that form edges of the two DRGs respectively. Consider two node-pairs that correspond to two edges of these two graphs in Fig.~\ref{fig:2PE} (b) respectively. Then, there exists a proper \proj-2 via using SPD as DE-2, associating these two node-pairs with different representations. Moreover, by simply aggregating the obtained representations of all node-pairs into graph representations via a set-pooling, we may also distinguish these two graphs. %However, the 2-WL test will fail in this case. 
Note that distinguishing the node-pairs of the two DRGs is really hard, because even the 2-WL test \footnote{We follow the terminology 2-WL test in \cite{cai1992optimal}, which refines representations of node-pairs iteratively and is proved to be more powerful than 1-WL test. This 2-WL test is termed 2-WL' test in~\cite{grohe2017descriptive} or 2-FWL test in~\cite{maron2019provably}. A brief introduction of higher-order WL tests can be found in Appendix~\ref{sec:WL-supp}.} will fail to distinguish any edges in the DRGs with a same intersection array and diameters exactly equal to 2~\footnote{Actually, the 2-WL test may not distinguish edges in a special class of DRGs, termed strongly regular graphs (SRG)~\cite{cai1992optimal}. A connected SRG is a DRG with diameter constrained as 2~\cite{brouwer2012strongly}.}. This means that the recently proposed more powerful GNNs, such as RingGNN~\cite{chen2019equivalence} and PPGN~\cite{maron2019provably}, will also fail in this case. However, it is possible to use \proj-2 to distinguish those two DRGs.

%which is formally summarized in the following example.  
%These two graphs belong to a special class of DRGs, termed strongly regular graphs (SRG) that are DRGs with diameters constrained as 2~\cite{brouwer2012strongly}. SRGs

%\begin{example}[the extra power of DE-2 (Fig.~\ref{fig:2PE} (b))]\label{exp:morepower}
%Consider any two node-pairs $S^{(1)} = \{v_1, u_1\}, S^{(2)} = \{v_2, u_2\}$. Consider two tuples $\mathcal{T}_1=(S^{(1)}, \mathbf{A}^{(1)})$ and $\mathcal{T}_2=(S^{(2)}, \mathbf{A}^{(2)})$ with graphs corresponding to Shrikhande graph and Rook's graph with $4\times4$ nodes respectively. Moreover, $S^{(1)}$ and $S^{(2)}$ correspond to node-pairs that are connected with edges of these two graphs respectively. 
%Consider two node-pairs that correspond to two edges in Shrikhande graph and $4\times4$ Rook's graph respectively. Then, there exists a proper \proj-2, whose outputs associate these two edges with different representations. Moreover, by simply aggregating the obtained representations of all node-pairs into graph representations via a set-pooling, we may also distinguish these two graphs. However, the 2-WL test will fail in this case.  \vspace{-0.05cm}
%\end{example} }
It is interesting to generalize Theorem~\ref{thm:power} to DRGs to demonstrate the power of \proj-$p$ ($p\geq2$). However, missing analytic-friendly random models for DRGs makes such generalization challenging.

%\vspace{-0.02cm}
\subsection{\projA --- Distance Encoding-1's as Controllers of the Message Aggregation} \label{sec:related}
%\vspace{-0.02cm}
\proj only uses DEs as initial node features. In this subsection, we further consider leveraging DE-1 between any two nodes to control the aggregation procedure of \proj. Specifically, we propose DE-Aggregation-GNN (\projA) to do the following change 
\begin{align}\label{eq:agg-pe}
\text{AGG}(\{f_2(h_u^{(l)}, \mathbf{A}_{vu})\}_{u\in \mathcal{N}_v})   \rightarrow \text{AGG}(\{(f_2(h_u^{(l)}, \mathbf{A}_{vu}), \zeta(u|v))\}_{u \in V})
\end{align}
Note that the representation power of \projA is at least no worse than \proj because the later one is specialized by aggregating the nodes with $\zeta_{spd}(u|v)=1$, so Theorem~\ref{thm:power}, Corollary~\ref{col:graphpower} are still true. Interestingly, its power is also limited by Theorem~\ref{thm:limit}. We conclude as the follows.

\begin{corollary}\label{col:PEAGNN}
Theorem~\ref{thm:power}, Corollary~\ref{col:graphpower} and Theorem~\ref{thm:limit} are still true for \projA. \vspace{-0.15cm}
\end{corollary}
The general form Eq.~\eqref{eq:agg-pe} that aggregates all nodes in each iteration holds more theoretical significance than practical usage due to scalability concern. In practice, the aggregation procedure of \projA may be trimmed by balancing the tradeoff between complexity and performance. For example, we may choose $\zeta(u|v) = \zeta_{spd}(u|v)$, and only aggregate the nodes $u$ such that $\zeta_{spd}(u|v)\leq K$, \textit{i.e.}, $K$-hop neighbors. Multi-hop aggregation allows avoiding the training issues of deep architecture, \textit{e.g.}, gradient degeneration. Particularly, we may prove that $K$-hop aggregation decreases the number of layers $L$ requested to $\lceil(\frac{1}{2} + \epsilon)\frac{\log n}{K\log(r-1)} \rceil$ in Theorem~\ref{thm:power} and Corollary~\ref{col:graphpower} with proof in Appendix~\ref{apd:PEAGNN}. We may also choose $\zeta(u|v) = \zeta_{gpr}(u|v)$ with non-negative $\gamma_k$ in Eq.~\eqref{eq:gpr} and aggregate the nodes whose $\zeta(u|v)$ are top-$K$ ranked among all $u\in V$. This manner is able to control fix-sized aggregation sets. As \projA does not show provably better representation power than \proj, all the above approaches share the same theoretical power and limitations. However, in practice their specific performance may vary across datasets and applications. 

%Although \projA does not show provably better representation power than \proj, \projA may outperform \proj in practice because directly aggregating multi-hop neighbors via \eqref{eq:agg-pe} helps avoid the training issues caused by deep architecture such as gradient degeneration. In the mean time, the cost for PEAGNN is to work on a denser graph. Particularly, if we decide to aggregate all nodes within $K$-hop neighbors, \textit{i.e.} using $\zeta_{spd}(u|v)\leq K$, the number of layers $L$ requested may decrease to $(\frac{1}{2} + \epsilon)\frac{\log n}{K\log(r-1)}$ for Theorem~\ref{thm:power} and Corollary~\ref{col:graphpower}. %More discussion on PEAGNN can be found in Appendix~\ref{apd:PEAGNN}.
\section{Related Work} \label{sec:related-works}
\vspace{-0.1cm}
Recently, extensive effort has been taken to improve the structural representation power of WLGNN. From the theoretical perspective, most previous works only considered representations of entire graphs~\cite{murphy2019relational,maron2018invariant,chen2019equivalence,maron2019provably,morris2019weisfeiler,kondor2018covariant,kondor2018generalization} while Srinivasan \& Ribeiro initialized the investigation of structural representations of node sets~\cite{srinivasan2019equivalence} from the view of joint probabilistic distributions. Some works view GNNs as general approximators of invariant functions but the proposed models hold more theoretical implication than practical usage because of their dependence on polynomial$(n)$-order tensors~\cite{maron2018invariant,maron2019universality,keriven2019universal}. Ring-GNN~\cite{chen2019equivalence} (or equivalently PPGN~\cite{maron2019provably}), a relatively scalable model among them, was based on 3-order tensors and was proposed to achieve the expressive power of the 2-WL test (a brief introduction in Appendix~\ref{sec:WL-supp}). However, Ring-GNN (PPGN) was proposed for entire-graph representations and cannot leverage the sparsity of the graph structure to be scalable enough to process large graphs~\cite{chen2019equivalence,maron2019provably}. \proj still benefits from such sparsity and are also used to represent node sets of arbitrary sizes. Moreover, our models theoretically behave orthogonal to Ring-GNN, as DE-2 can distinguish some non-isomorphic node-pairs that Ring-GNN fails to distinguish because the power of Ring-GNN is limited by the 2-WL test (Fig.~\ref{fig:2PE} (b)). %More discussion on higher-order WL tests and DEs are postponed to Appendix~\ref{sec:higher-order-WL}.

Some works with empirical success inspire the proposal of DE, though we are the first one to derive theoretical characterization and leverage our theory to better those models as a return. SEAL~\cite{zhang2018link} predicts links by reading out the representations of ego-networks of node-pairs. Although SEAL leverages a specific DE-2, the representations of those ego-networks are extracted via complex SortPooling~\cite{zhang2018end}. However, we argue against such complex operations as DE-2 yields all the benefit of representations of node-pairs, as demonstrated by our experiments. PGNN~\cite{you2019position} uses SPDs between each node and some anchor nodes to encode distance between nodes. As those encodings are not permutation invariant, PGNN holds worse inductive/generalization capability than our models. 

Quite a few works targeted at revising neighbor aggregation procedure of WLGNN and thus are related to the \projA. However, none of them demystified their connection to DE or provided theoretical characterization. MixHop~\cite{abu2019mixhop}, PAGTN~\cite{chen2019path}, MAT~\cite{maziarka2020molecule} essentially used $\zeta_{spd}(u|v)$ to change the way of aggregation (Eq.~\eqref{eq:agg-pe}) while GDC~\cite{klicpera2019diffusion} and PowerGNN~\cite{chen2019supervised} used variants of $\zeta_{gpr}(u|v)$. MixHop, GDC and PowerGNN are evaluated for node classification while PAGTN and MAT are evaluated for graph classification. GDC claims that the aggregation based on $\zeta_{gpr}(u|v)$ does not help link prediction. However, we are able to show its empirical success for link prediction, as the key point missed by GDC is using DEs as extra node attributes (Appendix~\ref{sec:degnn_supple}). Note that as the above models are covered by DEA-GNN-1, their representation powers are all bounded by Theorem~\ref{thm:limit} according to Corollary~\ref{col:PEAGNN}.

\vspace{-0.1cm}
\section{Experiments}\label{sec:exp}
\vspace{-0.1cm}
Extensive experiments\footnote{The code to evaluate our model can be downloaded from https://github.com/snap-stanford/distance-encoding.} are conducted to evaluate our \proj and \projA over three levels of tasks involving target node sets with sizes 1, 2 and 3 respectively: roles of nodes classification (Task 1), link prediction (Task 2), and triangle prediction (Task 3). Triangle prediction is to predict for any given subset of 3 nodes, $\{u, v, w\}$, whether links $uv$, $uw$, and $vw$ all exist. This task belongs to the more general class of higher-order network motif prediction tasks \cite{benson2018simplicial,nassar2019pairwise} and has recently attracted much significance to~\cite{benson2016higher,li2017inhomogeneous,li2019motif,yin2017local,tsourakakis2017scalable}. We briefly introduce the experimental settings and save the details of the datasets and the model parameters to Appendix~\ref{sec:exp_supple}.

\textbf{Dataset \& Training.} We use the following six real graphs for the three tasks introduced above: Brazil-Airports (Task 1), Europe-Airports (1), USA-Airports (1), NS (2 \& 3), PB (2), C.ele (2 \& 3). For Task 1, the goal is to predict the passenger flow level of a given airport based solely on the flight traffic network. These airports datasets are chosen because the labels indicate the structural roles of nodes (4 levels in total from hubs to switches) rather than community identifiers of nodes as traditionally used ~\cite{kipf2016semi,hamilton2017inductive,sen2008collective}. For Tasks 2 \& 3, the datasets were used by the strongest baseline ~\cite{zhang2018link}, which consist of existing links/triangles plus the the same number of randomly sampled negative instances from those graphs. The positive test links/triangles are removed from the graphs during the training phase. For all tasks, we use 80\%, 10\%, 10\% dataset splitting for training, validation, testing respectively. All the models are trained until loss converges and the testing performance of the best model on validation set is reported. We also report the experiments without validation sets that follow the original settings of the baselines ~\cite{ribeiro2017struc2vec,zhang2018link} in Appendix~\ref{sec:exp-res-sup}.   %\footnote{We have confirmed with the original authors that two baselines ~\cite{ribeiro2017struc2vec,zhang2018link} do not have validation set. Simply reporting the best test result without any validation is widely known as a "biased setting" that raises overfitting concerns. However, to ensure fair comparison with their reported results we additionally evaluate all our proposed models under this "biased" setting, whose result is reported in Appendix.}. 
% We follow the setting in \cite{ribeiro2017struc2vec} and use $20\%$ nodes for testing and the rest for training.

% We following the setting in~\cite{zhang2018link} and use $10\%$ of existing links / triangles plus the same number of randomly sampled negative instances for testing, and the rest for training. Notice that for these two tasks, all the positive test links / triangles are removed from the graphs during the training phase. %\textcolor{red}{Still need to mention why not using standard dataset like pubmed and cora? Move the last paragraph in "Broader Impact" to here?}

\textbf{Baselines.} We choose six baselines. GCN~\cite{kipf2016semi}, GraphSAGE(SAGE)~\cite{hamilton2017inductive}, GIN~\cite{xu2018powerful} are representative methods of WLGNN. These models use node degrees as initial features when attributes are not available to keep inductive ability. Struc2vec\cite{ribeiro2017struc2vec} is a kernel-based method, particularly designed for structural representations of single nodes. PGNN\cite{you2019position} and SEAL\cite{zhang2018link} are also GNN-based methods: PGNN learns node positional embeddings and is not inductive for node classification; SEAL is particularly designed for link prediction by using entire-graph representations of ego-networks of node-pairs. SEAL outperforms other link prediction approaches such as VGAE\cite{kipf2016variational}. The node initial features for these two models are set as the inductive setting suggested in their papers. We tune parameters of all baselines (Appendix~\ref{sec:hyper-supple}) and list their optimal performance here. 

\textbf{Instantiation of \proj and \projA.} We choose GCN as the basic WLGNN and implement three variants of \proj over it. Note that GIN could be a more  powerful basis while we tend to keep our models simple. The first two variants of Eq. \eqref{eq:node2nodePE} give us \proj-SPD and \proj-LP. The former uses SPD-based one-hot vectors $\zeta_{\text{sdp}}$ as extra nodes attributes, and the latter uses the sequence of landing probabilities Eq. \eqref{eq:node2nodePE}. Next, we consider an instantiation of Eq. \eqref{eq:agg-pe}, \projA-SPD that uses SPDs, $\zeta(u|v) = \zeta_{sdp}(u|v) \leq K$, to control the aggregation, which enables $K$-hop aggregation ($K=2,3$ and the better performance will be used). \projA-SPD uses SPD-based one-hot vectors as extra nodes attributes. Appendix~\ref{sec:degnn_supple} provides thorough discussion on implementation of the three variants and another implementation that uses Personalized PageRank scores to control the aggregation. Experiments are repeated 20 times using different seeds and we report the average. 

% \textbf{Model training.} Task 1's train-val-test split is 80-10-10. For Task We use 80\%, 10\%, and 10\% of the existing links plus the the same number of randomly sampled negative instances for training, validation and testing, respectively. We trained the model on training set until loss converges, and then pick the best model on validation set for testing. Notice that a critical difference with ~\cite{zhang2018link} is that they do not have validation set. Instead, they report the best performance on test set while training. 

%PGNN learns node position information while propagating attributes; struc2vec handcrafts local structural features for nodes; SEAL does link prediction based on extracted ego-network locally enclosing the target link. We refer the readers to the original papers for more details.

%All experiments are repeated 10 times using different seeds and the average accuracy is reported. For more implementation details please refer to the appendix []. The code is also available at https://github.com/Abel0828/pe-wl2.

\begin{table}[t]
\centering
\resizebox{\textwidth}{!}{%
\begin{tabular}{l|lll|lll|ll}
\hline
 &
  \multicolumn{3}{c|}{\cellcolor[HTML]{F5F9FF}Nodes (Task 1): Average Accuracy} &
  \multicolumn{3}{c|}{\cellcolor[HTML]{FFEFEE}Node-pairs (Task 2): AUC} &
  \multicolumn{2}{c}{\cellcolor[HTML]{EAFDEA}Node-triads (Task 3): AUC} \\
\multirow{-2}{*}{\backslashbox{Method}{Data}} &
  \multicolumn{1}{c}{Bra.-Airports} &
  \multicolumn{1}{c}{Eur.-Airports} &
  \multicolumn{1}{c|}{USA-Airports} &
  \multicolumn{1}{c}{C.elegans} &
  \multicolumn{1}{c}{NS} &
  \multicolumn{1}{c|}{PB} &
  \multicolumn{1}{c}{C.elegans} &
  \multicolumn{1}{c}{NS} \\ \hline
GCN~\cite{kipf2016semi} &
  64.55$\pm$4.18 &
  54.83$\pm$2.69 &
  56.58$\pm$1.11 &
  74.03$\pm$0.99 &
  74.21$\pm$1.72 &
  89.78$\pm$0.99 &
  80.94$\pm$0.51 &
  81.72$\pm$1.50 \\
SAGE~\cite{hamilton2017inductive} &
  70.65$\pm$5.33 &
  56.29$\pm$3.21 &
  50.85$\pm$2.83 &
  73.91$\pm$0.32 &
  79.96$\pm$1.40 &
  90.23$\pm$0.74 &
  84.72$\pm$0.40 &
  84.06$\pm$1.14 \\
GIN~\cite{xu2018powerful} &
  {71.89$\pm$3.60}$^\dagger$ &
  {57.05$\pm$4.08} &
  58.87$\pm$2.12&
  75.58$\pm$0.59&
  87.75$\pm$0.56&
  91.11$\pm$0.52&
  {86.42$\pm$1.12}$^\dagger$ &
  {94.59$\pm$0.66}$^\dagger$ \\
Struc2vec~\cite{ribeiro2017struc2vec} &
  70.88$\pm$4.26 &
  57.94$\pm$4.01$^\dagger$ &
  {61.92$\pm$2.61}$^\dagger$ &
  72.11$\pm$0.31 &
  82.76$\pm$0.59 &
  90.47$\pm$0.60 &
  77.72$\pm$0.58 &
  81.93$\pm$0.61 \\
PGNN~\cite{you2019position} &
  \multicolumn{1}{c}{N/A} &
  \multicolumn{1}{c}{N/A} &
  \multicolumn{1}{c|}{N/A} &
  78.20$\pm$0.33 &
  94.88$\pm$0.77 &
  89.72$\pm$0.32 &
  86.36$\pm$0.74 &
  79.36$\pm$1.49 \\
SEAL~\cite{zhang2018link} &
  \multicolumn{1}{c}{N/A} &
  \multicolumn{1}{c}{N/A} &
  \multicolumn{1}{c|}{N/A} &
  {88.26$\pm$0.56}$^\dagger$ &
  {98.55$\pm$0.32}$^\dagger$ &
  {94.18$\pm$0.57}$^\dagger$ &
  \multicolumn{1}{c}{N/A} &
  \multicolumn{1}{c}{N/A} \\ \hline
\proj-SPD &
  \textbf{73.28$\pm$2.47} &
  56.98$\pm$2.79 &
  \textbf{63.10$\pm$0.68}$^*$ &
  \textbf{89.37$\pm$0.17}$^*$ &
  \textbf{99.09$\pm$0.79} &
  \textbf{94.95$\pm$0.37}$^*$ &
  \textbf{92.17$\pm$0.72}$^*$ &
  \textbf{99.65$\pm$0.40}$^*$ \\
\proj-LP &
  \textbf{75.10$\pm$3.80}$^*$ &
  58.41$\pm$3.20$^*$ &
  \textbf{64.16$\pm$1.70}$^*$ &
  86.27$\pm$0.33 &
  98.01$\pm$0.55 &
  91.45$\pm$0.41 &
  86.24$\pm$0.18 &
  \textbf{99.31$\pm$0.12}$^*$ \\
\projA-SPD &
  \textbf{75.37$\pm$3.25}$^*$ &
  57.99$\pm$2.39$^*$ &
  \textbf{63.28$\pm$1.59} &
  \textbf{90.05$\pm$0.26}$^*$ &
  \textbf{99.43$\pm$0.63}$^*$ &
  \textbf{94.49$\pm$0.24}$^*$ &
  \textbf{93.35$\pm$0.65}$^*$ &
  \textbf{99.84$\pm$0.14}$^*$ \\ \hline
\end{tabular}%
}
\caption{\footnotesize{Performance in Average Accuracy and Area Under the ROC Curve (AUC) (mean in percentage $\pm$ 95\% confidence level). $\dagger$ highlights the best baselines. $^*$, \textbf{bold font}, \textbf{bold font}$^*$ respectively highlights the case where our models' performance exceeds the best baseline on average, by 70\% confidence, by 95\% confidence.}  }
\label{tab:performance}
\vspace{-0.6cm}
\end{table}

\textbf{Results} are shown in Table~\ref{tab:performance}. Regarding the node-level task, GIN outperforms other WLGNNs, which matches the theory in~\cite{xu2018powerful,maron2018invariant,morris2019weisfeiler}. Struc2vec is also powerful though it is kernel-based. \proj's significantly outperform the baselines (except Eur.-Airport) which imply the power of DE-1's. Among them, landing probabilities (LP) work slightly better than SPDs as DE-1's. 

Regarding node-pairs-level tasks, SEAL is the strongest baseline, as it is particularly designed for link prediction by using a special DE-2 plus a graph-level readout~\cite{zhang2018end}. However, our \proj-SPD  performs even significantly better than SEAL: The numbers are close, but the difference is still significant; The decreases of error rates are always greater than 10\% and achieve almost 30\% over NS. This indicates that DE-2 is the key signal that makes SEAL work while the complex graph-level readout adopted by SEAL is not necessary. Moreover, our set-pooling form of DE-2 (Eq.~\eqref{eq:node2setPE}) decreases the dimension of DE-2 adopted in SEAL, which also betters the generalization of our models (See the detailed discussion in Appendix~\ref{sec:degnn_supple}). Moreover, for link prediction, SPD seems to be much better to be chosen as DE-2 than LP.

%We claim that such benefit essentially comes from the usage of DE-2 as node attributes since

%Regarding node-pairs-level tasks, SEAL is the strongest baseline, as it is particularly designed for link prediction by using a special DE-2 plus a graph-level readout of the ego-networks around target node-pairs~\cite{zhang2018end}. We claim that such benefit essentially comes from the usage of DE-2 as node attributes since our \proj-SPD achieves comparable performance. \proj-SPD performs even better than SEAL due to its better generalizing ability by removing the complex graph-level readout. Moreover, the set-pooling form of DE-2 (Eq.~\eqref{eq:node2setPE}) decreases the dimension of DE-2 adopted in SEAL, which also betters the generalization of our models (See detailed discussion in Appendix~\ref{sec:degnn_supple}). 

Regarding node-triads-level tasks, no baselines were particularly designed for this setting. We have not expected that GIN outperforms PGNN as PGNN captures node positional information that seems useful to predict triangles. We guess that the distortion of absolute positional embeddings learnt by PGNN may be the reason that limits its ability to distinguish structures with nodes in close positions: For example, the three nodes in a path of length two are close in position and the three nodes in a triangle are also close in position.  However, this is not a problem for DE-3. We also conjecture that the gain based on DEs grows w.r.t. their orders (\textit{i.e.}, $|S|$ in Eq.~\eqref{eq:node2setPE}). Again, for triangle prediction, SPD seems to be much better to be chosen as DE-3 than LP.

Note that \projA-SPD further improves \proj-SPD (by almost 1\% across most of the tasks). This demonstrates the power of multi-hop aggregation (Eq.~\eqref{eq:agg-pe}). However, note that \projA-SPD needs to aggregate multi-hop neighbors simultaneously and thus pays an additional cost of scalability. 

\section{Conclusion and Discussion}
This work proposes a novel angle to systematically improve the structural representation power of GNNs. We break from the convention that previous works characterize and further improve the power of GNNs by intimating different-order WL tests~\cite{chen2019equivalence,maron2019provably,morris2019weisfeiler,chen2020can}. As far as we know, we are the first one to provide non-asymptotic analysis of the expressive power of the proposed GNN models. Therefore, the proof techniques of Theorems~\ref{thm:power},\ref{thm:limit} may be expected to inspire new theoretical studies of GNNs and further better the practical usage of GNNs. Moreover, our models have good scalability by avoiding using the framework of WL tests, as higher-order WL tests are not able to leverage the sparsity within graph structures. To be evaluated over extremely large graphs~\cite{hu2020open}, our models can be simply trimmed and work on the ego-networks sampled with a limited size around the target node sets, just as the strategy adopted by GraphSAGE~\cite{hamilton2017inductive} and GraphSAINT~\cite{zeng2019graphsaint}.

Distance encoding unifies the techniques of many GNN models~\cite{zhang2018link,abu2019mixhop,chen2019path,maziarka2020molecule,klicpera2019diffusion} and provides a extremely general framework with clear theoretical characterization. In this paper, we only evaluate four specific instantiations over three levels of tasks. However, there are some other interesting instantiations and applications. For example, we expect a better usage of PageRank scores as edge attributes (Eq.~\eqref{eq:agg-pe}). Currently, our instantiation \projA-PR simply uses those scores as weights in a weighted sum to aggregate node representations. We also have not considered any attention-based mechanism over DEs in aggregation while it seems to be practically useful~\cite{velivckovic2017graph,maziarka2020molecule}. Researchers may try these directions in a more principled manner based on this work. Our approaches may also help other tasks based on structural representation learning, such as graph-level classification/regression~\cite{gilmer2017neural,zhang2018end,xu2018powerful,maron2019provably,morris2019weisfeiler} and subgraph counting~\cite{chen2020can}, which we leave for future study. %which correspond to many applications including drug discovery and structured data analysis.
%which correspond to \textbf{many applications with wide societal impact including drug discovery and structured data analysis.} %As this paper has already covered broad content, the evaluation over these new applications are left out and may inspire researchers and practitioners for further investigation. 

There are also two important implications coming from the observations of this work. First, Theorem~\ref{thm:limit} and Corollary~\ref{col:PEAGNN} show the limitation of DE-1 over distance regular graphs, including the cases when DE-1's are used as node attributes or controllers of message aggregation. As distance regular graphs with the same intersection array have the important co-spectral property~\cite{brouwer2012distance}, we guess that DE-1 is a bridge  to connect GNN frameworks to spectral approaches, two fundamental approaches in graph-structured data processing. This point sheds some light on the question left in \cite{chen2019equivalence} while more rigorous characterization is still needed. Second, as observed in the experiments, higher-order DE's induce larger gains as opposed to WLGNN, while Theorem~\ref{thm:power} is not able to characterize this observation as the probability $1-o(\frac{1}{n})$ does not depend on the size $p$. We are sure that the probabilistic quantization in Theorem~\ref{thm:power} is not tight, so it is interesting to see how such probability depends on $p$ by deriving tighter bounds.

\section*{Acknowledgement}
The authors would like to thank Weihua Hu for raising the preliminary concept of distance encoding that initializes the investigation. The authors also would like to thank Jiaxuan You and Rex Ying for their insightful comments during the discussion. The authors also would thank Baharan Mirzasoleiman and Tailin Wu for their suggestions on the paper writing. The authors also would like to thank the NeurIPS reviewers for their insightful observations and actionable suggestions to improve the manuscript. This research has been supported in part by Purdue CS start-up, NSF CINES, NSF HDR, NSF Expeditions, NSF RAPID, DARPA MCS, DARPA ASED, ARO MURI, Stanford Data Science Initiative, Wu Tsai Neurosciences Institute, Chan Zuckerberg Biohub, Amazon, Boeing, Chase, Docomo, Hitachi, Huawei, JD.com, NVIDIA, Dell. J. L. is a Chan Zuckerberg Biohub investigator.

\bibliographystyle{IEEETran}
\bibliography{example_paper}

% Generated by IEEEtran.bst, version: 1.14 (2015/08/26)
\begin{thebibliography}{10}
\providecommand{\url}[1]{#1}
\csname url@samestyle\endcsname
\providecommand{\newblock}{\relax}
\providecommand{\bibinfo}[2]{#2}
\providecommand{\BIBentrySTDinterwordspacing}{\spaceskip=0pt\relax}
\providecommand{\BIBentryALTinterwordstretchfactor}{4}
\providecommand{\BIBentryALTinterwordspacing}{\spaceskip=\fontdimen2\font plus
\BIBentryALTinterwordstretchfactor\fontdimen3\font minus
  \fontdimen4\font\relax}
\providecommand{\BIBforeignlanguage}[2]{{%
\expandafter\ifx\csname l@#1\endcsname\relax
\typeout{** WARNING: IEEEtran.bst: No hyphenation pattern has been}%
\typeout{** loaded for the language `#1'. Using the pattern for}%
\typeout{** the default language instead.}%
\else
\language=\csname l@#1\endcsname
\fi
#2}}
\providecommand{\BIBdecl}{\relax}
\BIBdecl

\bibitem{hamilton2017representation}
W.~L. Hamilton, R.~Ying, and J.~Leskovec, ``Representation learning on graphs:
  Methods and applications,'' \emph{IEEE Data Engineering Bulletin}, vol.~40,
  no.~3, pp. 52--74, 2017.

\bibitem{borgatti1992notions}
S.~P. Borgatti and M.~G. Everett, ``Notions of position in social network
  analysis,'' \emph{Sociological Methodology}, pp. 1--35, 1992.

\bibitem{henderson2012rolx}
K.~Henderson, B.~Gallagher, T.~Eliassi-Rad, H.~Tong, S.~Basu, L.~Akoglu,
  D.~Koutra, C.~Faloutsos, and L.~Li, ``Rolx: structural role extraction \&
  mining in large graphs,'' in \emph{the ACM SIGKDD international conference on
  Knowledge discovery and data mining}, 2012, pp. 1231--1239.

\bibitem{rossi2014role}
R.~A. Rossi and N.~K. Ahmed, ``Role discovery in networks,'' \emph{IEEE
  Transactions on Knowledge and Data Engineering}, vol.~27, no.~4, pp.
  1112--1131, 2014.

\bibitem{ribeiro2017struc2vec}
L.~F. Ribeiro, P.~H. Saverese, and D.~R. Figueiredo, ``struc2vec: Learning node
  representations from structural identity,'' in \emph{the ACM SIGKDD
  International Conference on Knowledge Discovery and Data Mining}, 2017, pp.
  385--394.

\bibitem{donnat2018learning}
C.~Donnat, M.~Zitnik, D.~Hallac, and J.~Leskovec, ``Learning structural node
  embeddings via diffusion wavelets,'' in \emph{the ACM SIGKDD International
  Conference on Knowledge Discovery \& Data Mining}, 2018, pp. 1320--1329.

\bibitem{liben2007link}
D.~Liben-Nowell and J.~Kleinberg, ``The link-prediction problem for social
  networks,'' \emph{Journal of the American society for information science and
  technology}, vol.~58, no.~7, pp. 1019--1031, 2007.

\bibitem{zhang2017weisfeiler}
M.~Zhang and Y.~Chen, ``Weisfeiler-lehman neural machine for link prediction,''
  in \emph{the ACM SIGKDD International Conference on Knowledge Discovery and
  Data Mining}, 2017, pp. 575--583.

\bibitem{zhang2018link}
------, ``Link prediction based on graph neural networks,'' in \emph{Advances
  in Neural Information Processing Systems}, 2018, pp. 5165--5175.

\bibitem{you2019position}
J.~You, R.~Ying, and J.~Leskovec, ``Position-aware graph neural networks,'' in
  \emph{International Conference on Machine Learning}, 2019, pp. 7134--7143.

\bibitem{prvzulj2007biological}
N.~Pr{\v{z}}ulj, ``Biological network comparison using graphlet degree
  distribution,'' \emph{Bioinformatics}, vol.~23, no.~2, pp. 177--183, 2007.

\bibitem{zager2008graph}
L.~A. Zager and G.~C. Verghese, ``Graph similarity scoring and matching,''
  \emph{Applied mathematics letters}, vol.~21, no.~1, pp. 86--94, 2008.

\bibitem{shervashidze2011weisfeiler}
N.~Shervashidze, P.~Schweitzer, E.~J.~v. Leeuwen, K.~Mehlhorn, and K.~M.
  Borgwardt, ``Weisfeiler-lehman graph kernels,'' \emph{Journal of Machine
  Learning Research}, vol.~12, no. Sep, pp. 2539--2561, 2011.

\bibitem{gilmer2017neural}
J.~Gilmer, S.~S. Schoenholz, P.~F. Riley, O.~Vinyals, and G.~E. Dahl, ``Neural
  message passing for quantum chemistry,'' in \emph{International Conference on
  Machine Learning}, 2017, pp. 1263--1272.

\bibitem{ying2018hierarchical}
Z.~Ying, J.~You, C.~Morris, X.~Ren, W.~Hamilton, and J.~Leskovec,
  ``Hierarchical graph representation learning with differentiable pooling,''
  in \emph{Advances in Neural Information Processing Systems}, 2018, pp.
  4800--4810.

\bibitem{xu2018powerful}
K.~Xu, W.~Hu, J.~Leskovec, and S.~Jegelka, ``How powerful are graph neural
  networks?'' in \emph{International Conference on Learning Representations},
  2019.

\bibitem{maziarka2020molecule}
{\L}.~Maziarka, T.~Danel, S.~Mucha, K.~Rataj, J.~Tabor, and
  S.~Jastrz{\k{e}}bski, ``Molecule attention transformer,'' \emph{arXiv
  preprint arXiv:2002.08264}, 2020.

\bibitem{hornik1989multilayer}
K.~Hornik, M.~Stinchcombe, H.~White \emph{et~al.}, ``Multilayer feedforward
  networks are universal approximators.'' \emph{Neural Networks}, vol.~2,
  no.~5, pp. 359--366, 1989.

\bibitem{scarselli2008graph}
F.~Scarselli, M.~Gori, A.~C. Tsoi, M.~Hagenbuchner, and G.~Monfardini, ``The
  graph neural network model,'' \emph{IEEE Transactions on Neural Networks},
  vol.~20, no.~1, pp. 61--80, 2008.

\bibitem{kipf2016semi}
T.~N. Kipf and M.~Welling, ``Semi-supervised classification with graph
  convolutional networks,'' in \emph{International Conference on Learning
  Representations}, 2017.

\bibitem{hamilton2017inductive}
W.~Hamilton, Z.~Ying, and J.~Leskovec, ``Inductive representation learning on
  large graphs,'' in \emph{Advances in Neural Information Processing Systems},
  2017, pp. 1024--1034.

\bibitem{velivckovic2017graph}
P.~Veli{\v{c}}kovi{\'c}, G.~Cucurull, A.~Casanova, A.~Romero, P.~Lio, and
  Y.~Bengio, ``Graph attention networks,'' in \emph{International Conference on
  Learning Representations}, 2018.

\bibitem{zhang2018end}
M.~Zhang, Z.~Cui, M.~Neumann, and Y.~Chen, ``An end-to-end deep learning
  architecture for graph classification,'' in \emph{the AAAI Conference on
  Artificial Intelligence}, 2018, pp. 4438--4445.

\bibitem{battaglia2018relational}
P.~W. Battaglia, J.~B. Hamrick, V.~Bapst, A.~Sanchez-Gonzalez, V.~Zambaldi,
  M.~Malinowski, A.~Tacchetti, D.~Raposo, A.~Santoro, R.~Faulkner
  \emph{et~al.}, ``Relational inductive biases, deep learning, and graph
  networks,'' \emph{arXiv preprint arXiv:1806.01261}, 2018.

\bibitem{morris2019weisfeiler}
C.~Morris, M.~Ritzert, M.~Fey, W.~L. Hamilton, J.~E. Lenssen, G.~Rattan, and
  M.~Grohe, ``Weisfeiler and leman go neural: Higher-order graph neural
  networks,'' in \emph{the AAAI Conference on Artificial Intelligence},
  vol.~33, 2019, pp. 4602--4609.

\bibitem{weisfeiler1968reduction}
B.~Weisfeiler and A.~Leman, ``A reduction of a graph to a canonical form and an
  algebra arising during this reduction,'' \emph{Nauchno-Technicheskaya
  Informatsia}, 1968.

\bibitem{murphy2019relational}
R.~Murphy, B.~Srinivasan, V.~Rao, and B.~Riberio, ``Relational pooling for
  graph representations,'' in \emph{International Conference on Machine
  Learning}, 2019.

\bibitem{maron2018invariant}
H.~Maron, H.~Ben-Hamu, N.~Shamir, and Y.~Lipman, ``Invariant and equivariant
  graph networks,'' in \emph{International Conference on Learning
  Representations}, 2019.

\bibitem{chen2019equivalence}
Z.~Chen, S.~Villar, L.~Chen, and J.~Bruna, ``On the equivalence between graph
  isomorphism testing and function approximation with gnns,'' in \emph{Advances
  in Neural Information Processing Systems}, 2019, pp. 15\,868--15\,876.

\bibitem{maron2019provably}
H.~Maron, H.~Ben-Hamu, H.~Serviansky, and Y.~Lipman, ``Provably powerful graph
  networks,'' in \emph{Advances in Neural Information Processing Systems},
  2019, pp. 2153--2164.

\bibitem{chen2019path}
B.~Chen, R.~Barzilay, and T.~Jaakkola, ``Path-augmented graph transformer
  network,'' \emph{ICML 2019 Workshop on Learning and Reasoning with
  Graph-Structured Data}, 2019.

\bibitem{klicpera2019diffusion}
J.~Klicpera, S.~Wei{\ss}enberger, and S.~G{\"u}nnemann, ``Diffusion improves
  graph learning,'' in \emph{Advances in Neural Information Processing
  Systems}, 2019, pp. 13\,333--13\,345.

\bibitem{chien2020joint}
E.~Chien, J.~Peng, P.~Li, and O.~Milenkovic, ``Joint adaptive feature smoothing
  and topology extraction via generalized pagerank gnns,'' \emph{arXiv preprint
  arXiv:2006.07988}, 2020.

\bibitem{li2019optimizing}
P.~Li, I.~Chien, and O.~Milenkovic, ``Optimizing generalized pagerank methods
  for seed-expansion community detection,'' in \emph{Advances in Neural
  Information Processing Systems}, 2019, pp. 11\,705--11\,716.

\bibitem{brouwer2012distance}
A.~E. Brouwer and W.~H. Haemers, ``Distance-regular graphs,'' in \emph{Spectra
  of Graphs}.\hskip 1em plus 0.5em minus 0.4em\relax Springer, 2012, pp.
  177--185.

\bibitem{babai2016graph}
L.~Babai, ``Graph isomorphism in quasipolynomial time,'' in \emph{Proceedings
  of the Forty-Eighth Annual ACM Symposium on Theory of Computing}, 2016, pp.
  684--697.

\bibitem{grover2016node2vec}
A.~Grover and J.~Leskovec, ``node2vec: Scalable feature learning for
  networks,'' in \emph{the ACM SIGKDD International Conference on Knowledge
  Discovery and Data Mining}, 2016, pp. 855--864.

\bibitem{zaheer2017deep}
M.~Zaheer, S.~Kottur, S.~Ravanbakhsh, B.~Poczos, R.~R. Salakhutdinov, and A.~J.
  Smola, ``Deep sets,'' in \emph{Advances in Neural Information Processing
  Systems}, 2017, pp. 3391--3401.

\bibitem{arvind2019weisfeiler}
V.~Arvind, F.~Fuhlbr{\"u}ck, J.~K{\"o}bler, and O.~Verbitsky, ``On
  weisfeiler-leman invariance: subgraph counts and related graph properties,''
  in \emph{International Symposium on Fundamentals of Computation
  Theory}.\hskip 1em plus 0.5em minus 0.4em\relax Springer, 2019, pp. 111--125.

\bibitem{cai1992optimal}
J.-Y. Cai, M.~F{\"u}rer, and N.~Immerman, ``An optimal lower bound on the
  number of variables for graph identification,'' \emph{Combinatorica},
  vol.~12, no.~4, pp. 389--410, 1992.

\bibitem{grohe2017descriptive}
M.~Grohe, \emph{Descriptive complexity, canonisation, and definable graph
  structure theory}.\hskip 1em plus 0.5em minus 0.4em\relax Cambridge
  University Press, 2017, vol.~47.

\bibitem{brouwer2012strongly}
A.~E. Brouwer and W.~H. Haemers, ``Strongly regular graphs,'' in \emph{Spectra
  of Graphs}.\hskip 1em plus 0.5em minus 0.4em\relax Springer, 2012, pp.
  115--149.

\bibitem{kondor2018covariant}
R.~Kondor, H.~T. Son, H.~Pan, B.~Anderson, and S.~Trivedi, ``Covariant
  compositional networks for learning graphs,'' \emph{arXiv preprint
  arXiv:1801.02144}, 2018.

\bibitem{kondor2018generalization}
R.~Kondor and S.~Trivedi, ``On the generalization of equivariance and
  convolution in neural networks to the action of compact groups,'' in
  \emph{International Conference on Machine Learning}, 2018, pp. 2747--2755.

\bibitem{srinivasan2019equivalence}
B.~Srinivasan and B.~Ribeiro, ``On the equivalence between node embeddings and
  structural graph representations,'' in \emph{International Conference on
  Learning Representations}, 2020.

\bibitem{maron2019universality}
H.~Maron, E.~Fetaya, N.~Segol, and Y.~Lipman, ``On the universality of
  invariant networks,'' in \emph{International Conference on Machine Learning},
  2019, pp. 4363--4371.

\bibitem{keriven2019universal}
N.~Keriven and G.~Peyr{\'e}, ``Universal invariant and equivariant graph neural
  networks,'' in \emph{Advances in Neural Information Processing Systems},
  2019, pp. 7090--7099.

\bibitem{abu2019mixhop}
S.~Abu-El-Haija, B.~Perozzi, A.~Kapoor, N.~Alipourfard, K.~Lerman,
  H.~Harutyunyan, G.~Ver~Steeg, and A.~Galstyan, ``Mixhop: Higher-order graph
  convolutional architectures via sparsified neighborhood mixing,'' in
  \emph{International Conference on Machine Learning}, 2019, pp. 21--29.

\bibitem{chen2019supervised}
Z.~Chen, L.~Li, and J.~Bruna, ``Supervised community detection with line graph
  neural networks,'' in \emph{International Conference on Learning
  Representations}, 2019.

\bibitem{benson2018simplicial}
A.~R. Benson, R.~Abebe, M.~T. Schaub, A.~Jadbabaie, and J.~Kleinberg,
  ``Simplicial closure and higher-order link prediction,'' \emph{the
  Proceedings of the National Academy of Sciences}, vol. 115, no.~48, pp.
  11\,221--11\,230, 2018.

\bibitem{nassar2019pairwise}
H.~Nassar, A.~R. Benson, and D.~F. Gleich, ``Pairwise link prediction,'' in
  \emph{the IEEE/ACM International Conference on Advances in Social Networks
  Analysis and Mining}, 2019, pp. 386--393.

\bibitem{benson2016higher}
A.~R. Benson, D.~F. Gleich, and J.~Leskovec, ``Higher-order organization of
  complex networks,'' \emph{Science}, vol. 353, no. 6295, pp. 163--166, 2016.

\bibitem{li2017inhomogeneous}
P.~Li and O.~Milenkovic, ``Inhomogeneous hypergraph clustering with
  applications,'' in \emph{Advances in Neural Information Processing Systems},
  2017, pp. 2308--2318.

\bibitem{li2019motif}
P.~Li, G.~J. Puleo, and O.~Milenkovic, ``Motif and hypergraph correlation
  clustering,'' \emph{IEEE Transactions on Information Theory}, 2019.

\bibitem{yin2017local}
H.~Yin, A.~R. Benson, J.~Leskovec, and D.~F. Gleich, ``Local higher-order graph
  clustering,'' in \emph{Proceedings of the 23rd ACM SIGKDD International
  Conference on Knowledge Discovery and Data Mining}.\hskip 1em plus 0.5em
  minus 0.4em\relax ACM, 2017, pp. 555--564.

\bibitem{tsourakakis2017scalable}
C.~E. Tsourakakis, J.~Pachocki, and M.~Mitzenmacher, ``Scalable motif-aware
  graph clustering,'' in \emph{Proceedings of the 26th International Conference
  on World Wide Web}, 2017, pp. 1451--1460.

\bibitem{sen2008collective}
P.~Sen, G.~Namata, M.~Bilgic, L.~Getoor, B.~Galligher, and T.~Eliassi-Rad,
  ``Collective classification in network data,'' \emph{AI magazine}, vol.~29,
  no.~3, pp. 93--93, 2008.

\bibitem{kipf2016variational}
T.~N. Kipf and M.~Welling, ``Variational graph auto-encoders,'' \emph{NeurIPS
  Bayesian Deep Learning Workshop}, 2016.

\bibitem{chen2020can}
Z.~Chen, L.~Chen, S.~Villar, and J.~Bruna, ``Can graph neural networks count
  substructures?'' \emph{arXiv preprint arXiv:2002.04025}, 2020.

\bibitem{hu2020open}
W.~Hu, M.~Fey, M.~Zitnik, Y.~Dong, H.~Ren, B.~Liu, M.~Catasta, and J.~Leskovec,
  ``Open graph benchmark: Datasets for machine learning on graphs,''
  \emph{arXiv preprint arXiv:2005.00687}, 2020.

\bibitem{zeng2019graphsaint}
H.~Zeng, H.~Zhou, A.~Srivastava, R.~Kannan, and V.~Prasanna, ``Graphsaint:
  Graph sampling based inductive learning method,'' in \emph{International
  Conference on Learning Representations}, 2020.

\bibitem{bollobas1980probabilistic}
B.~Bollob{\'a}s, ``A probabilistic proof of an asymptotic formula for the
  number of labelled regular graphs,'' \emph{European Journal of
  Combinatorics}, vol.~1, no.~4, pp. 311--316, 1980.

\bibitem{RePEc:oxp:obooks:9780198514978}
L.~W. Beineke and R.~J. Wilson, Eds., \emph{Graph Connections: Relationships
  between Graph Theory and Other Areas of Mathematics}.\hskip 1em plus 0.5em
  minus 0.4em\relax Oxford University Press, 1997.

\bibitem{ackland2005mapping}
R.~Ackland \emph{et~al.}, ``Mapping the us political blogosphere: Are
  conservative bloggers more prominent?'' in \emph{BlogTalk Downunder 2005
  Conference, Sydney}.\hskip 1em plus 0.5em minus 0.4em\relax BlogTalk
  Downunder 2005 Conference, Sydney, 2005.

\bibitem{4726313}
C.~C. Kaiser, Marcus;~Hilgetag, ``Nonoptimal component placement, but short
  processing paths, due to long-distance projections in neural systems,''
  \emph{PLoS Computational Biology}, vol.~2, no.~7, p.~95, 2006.

\bibitem{newman2006finding}
M.~E. Newman, ``Finding community structure in networks using the eigenvectors
  of matrices,'' \emph{Physical review E}, vol.~74, no.~3, p. 036104, 2006.

\bibitem{jeh2003scaling}
G.~Jeh and J.~Widom, ``Scaling personalized web search,'' in \emph{the
  International Conference on World Wide Web}, 2003, pp. 271--279.

\bibitem{chung2007heat}
F.~Chung, ``The heat kernel as the pagerank of a graph,'' \emph{Proceedings of
  the National Academy of Sciences}, vol. 104, no.~50, pp. 19\,735--19\,740,
  2007.

\bibitem{gleich2014dynamical}
D.~F. Gleich and R.~A. Rossi, ``A dynamical system for pagerank with
  time-dependent teleportation,'' \emph{Internet Mathematics}, vol.~10, no.
  1-2, pp. 188--217, 2014.

\end{thebibliography}

\newpage
 
\appendix

\begin{appendix}

\begin{center}
{\Large \textbf{Appendix}}
\end{center}
\end{appendix}

\section{Proof of Universal Approximate Theorem for Structural Representation}
We restate Theorem~\ref{thm:univapp}:  If the structural representation $\Gamma$ can distinguish any two non-isomorphic tuples $\mathcal{T}^{(1)}$ and $\mathcal{T}^{(2)}$ in $\Omega_p$, then for any invariant function $f: \Omega_p \rightarrow \mathbb{R}$, $f$ can be universally approximated by $\Gamma$ via a 3-layer feed-forward neural network with ReLU as rectifiers, as long as
\begin{itemize}
\item The feature space $\mathcal{A}$ is compact.
\item $f(S, \cdot)$ is continuous over $\mathcal{A}$ for any $S\in \mathcal{P}_p(V)$. 
\end{itemize}

\begin{proof}
This result is a direct generalization of Theorem 4~\cite{chen2019equivalence}. Specifically, we extend the statement of representing graphs featured by $\mathbf{A}$ to that of representing structures featured by $(S, \mathbf{A})$.%; 2) We extend the single dimension case $d=1$ to multiple dimension case. We first prove the first aspect by assuming that the image of $\Gamma$ and $f$ are in $\mathbb{R}$, \textit{i.e.}, $d=1$. Then, we generalize the result to the case $d > 1$. 

Recall the original space $\mathcal{A} \subset \mathbb{R}^{n\times n\times k}$. We define a space $\mathcal{A}' \subset \mathbb{R}^{n\times n\times (k+1)}$: For any $\mathbf{A}' \in \mathcal{A}'$, its slice in the first $k$ dimensions of 3-rd mode, \textit{i.e.}, $\mathbf{A}'_{\cdot,\cdot, 1:k}$, is in $\mathcal{A}$ and the slice corresponds to the last dimension of 3-rd mode, \textit{i.e.}, $\mathbf{A}'_{\cdot,\cdot, k+1}$ is a diagonal matrix where the diagonal components could be only 0 or 1. Then, we may build a bijective mapping between $\mathbf{A}'\in\mathcal{A}'$ and $(S, \mathbf{A})\in\Omega_p$ by 
\begin{align*}
\mathbf{A}'_{\cdot,\cdot, 1:k} = \mathbf{A}, \quad \mathbf{A}'_{u, u, k+1} = \;\text{1 if $u\in S$ or 0 if $u\not\in S$}
\end{align*} 
As $\mathcal{A}$ is compact in $\mathbb{R}^{n\times n\times k}$ and we have only finite possible choices of $\mathbf{A}'_{\cdot,\cdot, k+1}$, actually ${n\choose |S|}$, the space $\mathcal{A}'$ is compact in $\mathbb{R}^{n\times n\times (k+1)}$.

Then, we may transfer all definitions from $\Omega_p$ to $\mathcal{A}'$. Specifically, the structural representation $\Gamma$ that distinguishes any two non-isomorphic tuples $\mathcal{T}^{(1)}$ and $\mathcal{T}^{(2)}$ in $\Omega_p$ defines $\Gamma': \mathcal{A}'  \rightarrow \mathbb{R}^d$ that distinguishes any two non-isomorphic tensors $\mathbf{A}'^{(1)}$ and $\mathbf{A}'^{(2)}$ in $\mathcal{A}'$, as $\mathcal{T}^{(i)}$ and $\mathbf{A}'^{(i)}$ form a bijective mapping for $i=1,2$. Moreover, one invariant function $f: \Omega_p \rightarrow \mathbb{R}$ also defines another invariant function $f': \mathcal{A}'  \rightarrow \mathbb{R}$, as $\mathcal{T}^{(i)}$ and $\mathbf{A}'^{(i)}$ form a bijective mapping for $i=1,2$.

Suppose the original metric over $\mathcal{A}$ is denoted by $\mathcal{M}: \mathcal{A} \times \mathcal{A} \rightarrow \mathbb{R}_{\geq 0}$. Define a metric  over $\mathcal{A}'$ as
\begin{align*}
\mathcal{M}'(\mathbf{A}'^{(1)}, \mathbf{A}'^{(2)}) = \mathcal{M}(\mathbf{A}'^{(1)}_{\cdot,\cdot, 1:k}, \mathbf{A}'^{(2)}_{\cdot,\cdot, 1:k}) + \sum_{u\in V} 1_{\mathbf{A}'^{(1)}_{u, u, k+1} \neq \mathbf{A}'^{(2)}_{u, u, k+1}}.
\end{align*}
Then, it is easy to show that we have the following lemma based on the definition of continuity. 
\begin{lemma}
If $f(S, \cdot)$ is continuous over $\mathcal{A}$ for any $S\in \mathcal{P}_p(V)$ with respect to $\mathcal{M}$, then $f'$ is continuous over $\mathcal{A}'$ with respect to $\mathcal{M}'$.
\end{lemma}

Now, we only need to use Theorem 4~\cite{chen2019equivalence} to prove the statement. Actually the dimensions of $\Gamma'$ overall forms a collection of one-dimensional functions $\Xi=(\Gamma'[i])_{i\in [d]}$, where $\Gamma'[i]$ is the $i$th component of $\Gamma'\in \mathbb{R}^d$. According to the definition of $\Gamma'$, we know $\Xi$ distinguishes all the non-isomorphic  $\mathbf{A}'^{(1)}$ and $\mathbf{A}'^{(2)}$ in $\mathcal{A}'$. Moreover, because $\mathcal{A}'$ is compact and $f'$ is continuous over $\mathcal{A}'$,  Theorem 4~\cite{chen2019equivalence} shows that the arbitrary invariant function $f'$ defined on $\mathcal{A}'$ can be universally approximated by $\Xi$ via 3-layer feed-forward neural networks with ReLu as rectifiers. Recall $f$ and $f'$ are bijective, and $\Gamma,\,\Gamma'$, and $\Xi$ are mutually bijective. Therefore, we claim that $f$ can be universally approximated $\Gamma$  via 3-layer feed-forward neural networks with ReLu as rectifiers.
\end{proof}

\section{Proof for The Power of WLGNN for Structural Representation}
We restate Theorem~\ref{thm:1-WL-limit}: Consider two tuples $\mathcal{T}_1 = (S^{(1)}, \mathbf{A}^{(1)})$ and $\mathcal{T}_2 = (S^{(2)}, \mathbf{A}^{(2)})$ in $\Omega_p$. If $\mathcal{T}_1, \,\mathcal{T}_2$ cannot be distinguished by the 1-WL test, then the corresponding outputs of WLGNN  satisfy $\Gamma(\mathcal{T}_1) = \Gamma(\mathcal{T}_2)$. On the other side, if they can be distinguished by the 1-WL test and suppose aggregation operations (AGG) and feed-forward neural networks $f_1,\, f_2$ are all injective mappings, then a large enough number of layers $L$, the outputs of WLGNN  satisfy $\Gamma(\mathcal{T}_1)\neq \Gamma(\mathcal{T}_2)$.

\begin{proof}
There is no fundamental difficulty to generalize the results from the case of graph representation to that of structural representation, because the only difference according to WLGNN is the final readout step (AGG($\cdot$)), which works on a subset of nodes instead of the entire node set. Therefore, the same logic of proofs of Lemma 2 and Theorem 3 in~\cite{xu2018powerful} can be directly applied for structural representation learning with little revision.
\end{proof}

\section{Proof for The Power of DE ---  Theorem~\ref{thm:power}}\label{sec:apd-power}
We restate Theorem~\ref{thm:power}: Given two fixed-sized sets $S^{(1)}, S^{(2)}\subset V$, $|S^{(1)}| = |S^{(2)}|=p$. Consider two tuples $\mathcal{T}^{(1)}=(S^{(1)}, \mathbf{A}^{(1)})$ and $\mathcal{T}^{(2)}=(S^{(2)}, \mathbf{A}^{(2)})$ in the most difficult setting when features $\mathbf{A}^{(1)}$ and $\mathbf{A}^{(2)}$ are only different in graph structures specified by $A^{(1)}$ and $A^{(2)}$ respectively. Suppose $A^{(1)}$ and $A^{(2)}$ are uniformly independently sampled from all r-regular graphs over $V$ where $3\leq r< (2\log n)^{1/2}$. Then, for any constant $\epsilon > 0$, there exist a proper \proj-$p$ with layers $L < (\frac{1}{2} + \epsilon)\frac{\log n}{\log(r-1)}$, using DE-$p$ $\zeta(u|S^{(1)}),\,\zeta(u|S^{(2)})$ for all $u\in V$ such that with probability $1-o(n^{-1})$, its outputs $\Gamma(\mathcal{T}^{(1)})\neq \Gamma(\mathcal{T}^{(2)})$. Specifically, $f_3$ can be simply chosen as SPD, \textit{i.e.}, $\zeta(u|v) = \zeta_{spd}(u|v)$. The big-O notation is with respect to $n$.
\begin{proof}
To prove the statement, we only need to prove the case that $|S^{(1)}| = |S^{(2)}| = 1$, $\zeta(u|v) = \zeta_{spd}(u|v)$ because of the following lemma.
\begin{lemma}\label{lemma:power-step0}
Suppose the statement is true when $|S^{(1)}| = |S^{(2)}| = 1$, $\zeta(u|v) = \zeta_{spd}(u|v)$. Then, the statement is also true for the case when $|S^{(1)}| = |S^{(2)}| = p > 1$ for some fixed $p$, and $\zeta(u|v)$ is a neural network fed with the list of landing probabilities.
\end{lemma}
\begin{proof}
We first focus on the case when DE is chosen as SPD, \textit{i.e.}, $\zeta(u|v) = \zeta_{spd}(u|v)$. We want to use the results of the case $|S^{(1)}| = |S^{(2)}|=1$ to prove that of the case $|S^{(1)}| = |S^{(2)}|>1$. Suppose $|S^{(1)}| = |S^{(2)}|>1$. We choose an arbitrary node from $S^{(1)}$, say $w_1$. As we assume the statement is true for the single node case, for any node in $S^{(2)}$, say $w_2$, \proj with DEs $\zeta_{spd}(u|w_1),\,\zeta_{spd}(u|w_2)$ is able to distinguish two tuples $(w_1, \mathbf{A}^{(1)})$ and $(w_2, \mathbf{A}^{(2)})$, with probability at least $1-o(n^{-1})$. Given that the space of SPD is countable and AGG in Eq.~\eqref{eq:node2setPE} is injective, $\zeta_{spd}(u|S^{(1)})$ and $\zeta_{spd}(u|S^{(2)})$ are different if $\zeta_{spd}(u|w_1)$ is different from any $\zeta_{spd}(u|w_2)$ ($w_2\in S^{(2)}$), which happens with probability at least $1- |S^{(2)}|o(n^{-1}) = 1-  o(n^{-1})$. Therefore, \proj with DEs $\zeta_{spd}(u|S^{(1)}),\,\zeta_{spd}(u|S^{(2)})$ is also able to distinguish two tuples $(w_1, \mathbf{A}^{(1)})$ and $(w_2, \mathbf{A}^{(2)})$, with probability at least $1-o(n^{-1})$. Based on the union bound, we know that \proj with DEs $\zeta_{spd}(u|S^{(1)}),\,\zeta_{spd}(u|S^{(2)})$ is able to distinguish two tuples $(w_1, \mathbf{A}^{(1)})$ and $(v, \mathbf{A}^{(2)})$ for any $v\in S^{(2)}$, with probability at least $1- |S^{(2)}|o(n^{-1}) = 1-  o(n^{-1})$. Therefore, we prove the capability to generalize the result from the single node case to the multiple node cases.

Now, let us generalize the result from $\zeta_{spd}(u|v)$ to arbitrary $\zeta(u|v)$ represented by neural networks fed with the list of landing probabilities. As $\zeta_{spd}(u|v)$ is indeed a function of the list of landing probabilities (Eq.~\eqref{eq:node2nodePE}), the general $\zeta(u|v)$ should have stronger discriminatory power unless neural networks cannot provide a good mapping from the list of landing probabilities to SPD. However, we do not have to worry this because the list of landing probabilities fortunately lies a countable space for unweighted graphs: 1) The dimension of this list is countable (finite in practice); 2) Each component of this list is always a rational number if the graph is unweighted. According to our assumption, $f_3$ is allowed to have an injective mapping over the list of landing probabilities. Therefore, neural networks on the list of landing probabilities will not decrease the representation power that is just based on $\zeta_{spd}(u|v)$.
\end{proof}

\begin{figure}[t]
\centering
\includegraphics[trim={1cm 9cm 0.5cm 3.3cm},clip,width=0.9\textwidth]{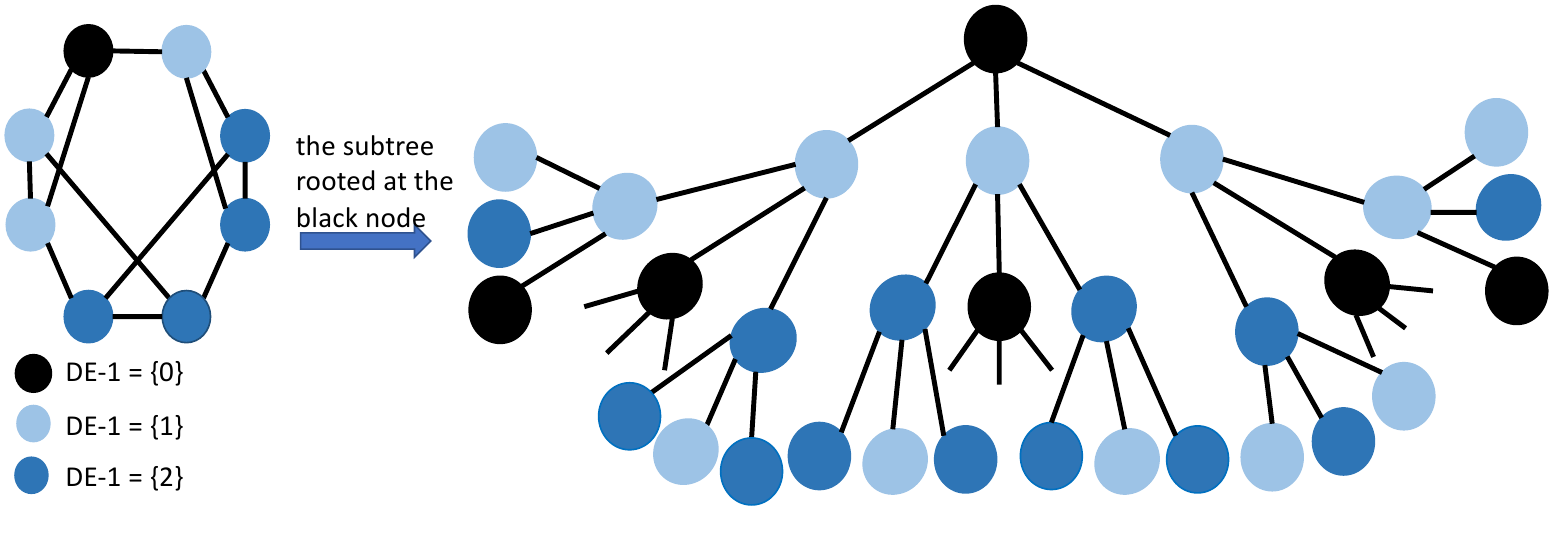}
\includegraphics[trim={1cm 9cm 0.5cm 3.5cm},clip,width=0.9\textwidth]{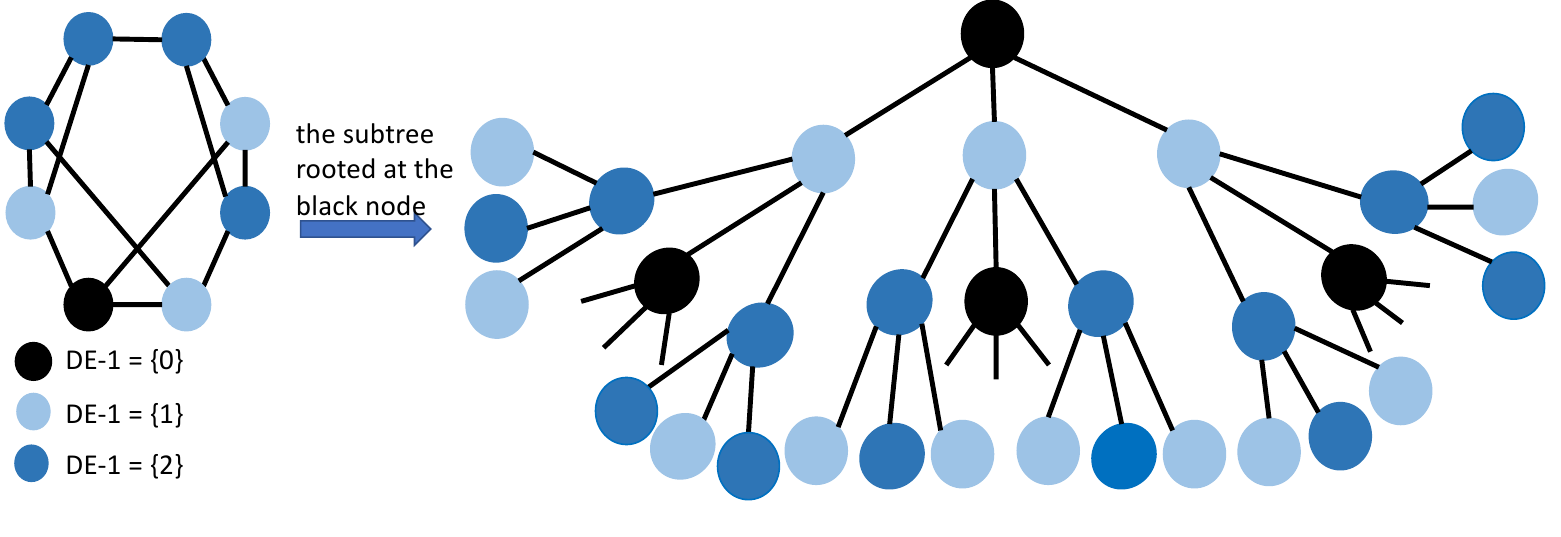}
\caption{The subtree rooted at a node: In the left two graphs, the black nodes are the target nodes who structural representations are to be learnt. Different colors of the nodes correspond to different SPDs with respect to the target nodes. The right two trees correspond to the subtrees rooted at these two target nodes respectively. \proj essentially works from bottom to top along these subtrees to obtain the representations of the target nodes. Different types of subtrees yield different representations based on a proper \proj. In this example, these two target nodes are all from 3-regular graphs so they cannot be distinguished via WLGNN without DE-1's or the informative node/edge attributes. However, with DE-1's, we can see the corresponding subtrees of these two nodes are different and the difference appears in the second layer.}
\label{fig:subtrees}
\end{figure}

Keep in mind that Lemma~\ref{lemma:power-step0} could be loose. We leave the tight probability bound for the $p>2$ case in the future. 

\textbf{Outline:} From now on, we focus on the single node case, \textit{i.e.} $|S^{(1)}| = |S^{(2)}| = 1$, with SPD as the DE-1 ($f_3$), \textit{i.e.}, $\zeta(u|v) = \zeta_{spd}(u|v)$. Without loss of generality, we suppose $S^{(1)} = S^{(2)} = \{u\}$. As SPD is countable, there exists a proper \proj that guarantees that all the operations are injective, which follows the basic condition used in~\cite{xu2018powerful}. Because of the iterative procedure of \proj and all mappings are injection, the label of node $u$ only depends on the subtree with depth $L$ rooted at $u$ (See the illustration of the subtree rooted at a given node in Fig.~\ref{fig:subtrees}). Recall $L$ is the number of layers in DE-GNN. Therefore, we only need to show that $A^{(1)}$ and $A^{(2)}$, which are uniformly sampled $r$-regular graphs, with probability at most $o(n^{-1})$, have the same subtrees rooted at $u$ given SPDs as initial node labels. To show this, our proof contains four steps. Note that all through the following proof, we assume that $n$ is very large and $\epsilon$ is any small positive constant that is independent from $n$.  
\begin{itemize}
\item We first explain that we are able to work on the configuration model of $r$-regular graphs proposed in~\cite{bollobas1980probabilistic} which associates uniform measure over all $r$-regular graphs. Given the condition $r< (2\log n)^{1/2}$, there are a large portion ($\Omega(n^{-1/2})$) of all the graphs generated by this model are simple (without self-loops and multi-edges) $r$-regular graphs~\cite{bollobas1980probabilistic}. Since the configuration model alleviates the difficulty to analyze the dependence between edges in $r$-regular graphs, we consider the graphs generated by the configuration model for the next two steps.%In the next two steps, we are to prove that there are at most $o(n^{-\frac{3}{2}})$ graphs generated by the configuration model that shares the same subtrees rooted at $u$ given SPD as initial node labels.
\item Suppose the set of nodes that are associated with SPD$=k$ from $u$ is denoted by $Q_k$, and the number of edges that connect the nodes in $Q_k$ and those in $Q_{k+1}$ is denoted by $p_k$. We prove that with probability $1-o(n^{-\frac{3}{2}})$, for all $k\in (\frac{\epsilon}{5}\frac{\log n}{\log(r-1)} + 1, (\frac{2}{3}-\epsilon)\frac{\log n}{\log(r-1)})$, $|Q_{k}| \geq (r-1-\epsilon)^{k-1} $ and $p_k \geq (r-1-\epsilon) |Q_{k}|$ based on the configuration model. 
\item Next, we define the edge configuration between $Q_k$ and $Q_{k+1}$ as a list $C_k = (a_{1,k}, a_{2,k},...)$ where $a_{i,k}$ denotes the number of nodes in $Q_{k+1}$ of which each has exactly $i$ edges from $Q_k$. We prove that for each $k \in (\frac{1}{2}\frac{\log n}{\log(r-1-\epsilon)}, \frac{4}{7}\frac{\log n}{\log(r-1-\epsilon)})$, as the edges between $Q_k$ and $Q_{k+1}$ are so many, there are too much randomness that makes each type of edge configuration $C_k$ appear with only limited probability $\mathbb{P}(C_k) = O(\frac{n^{1/2}}{p_k})$. Recall that $p_k$ is defined in step 2. Then, given any $\epsilon \frac{\log n}{\log(r-1-\epsilon)}$ many $k$'s, the probability that $A^{(1)}$ and $A^{(2)}$ have all the same edge configurations for these $k$'s is bounded by $\Pi_{k} \mathbb{P}(C_k)  \sim o(n^{-\frac{3}{2}})$. Therefore, we only need to consider edge configurations for $k\in  (\frac{1}{2}\frac{\log n}{\log(r-1-\epsilon)}, (\frac{1}{2} + \epsilon)\frac{\log n}{\log(r-1-\epsilon)})$ to distinguish $A^{(1)}$ and $A^{(2)}$, as this will give us $1-o(n^{-3}{2})$ probability to succeed. 

\item Since there are at least $\Omega(n^{-1/2})$ of all the graphs generated by the configuration model that are simple $r$-regular graphs, and there are at most $o(n^{-\frac{3}{2}}) $ probability that $A^{(1)}$ and $A^{(2)}$ share the same subtrees rooted at $u$, there are at most $o(n^{-\frac{3}{2}}/n^{-1/2})=o(n^{-1})$ probability that $A^{(1)}$ and $A^{(2)}$ are simple $r$-regular graphs and share the same subtrees rooted at $u$, which concludes the proof.
\end{itemize}

\textbf{Step 1:} We first introduce the configuration model proposed in \cite{bollobas1980probabilistic} for $r$-regular graphs of $n$ nodes. Suppose we have $n$ sets of items, $W_u$, $u\in[n]$, where each set corresponds to one node in $[n]$. Each set $W_u$ has $r$ items. Now, we randomly partition all these $nr$ items into $\frac{nr}{2}$ pairs. Then, each partitioning result corresponds to a r-regular graph: if a pair contains items from $W_u$ and $W_v$, then there is an edge between nodes $u$ and $v$ in the graph. Note that such partitioning results may render self-loops and multi-edges. Of course, we would like to consider only simple graphs which do not have self-loops and multi-edges. For this, the theory in~\cite{bollobas1980probabilistic} shows that for all these $r$-regular graphs, if $r<(2\log n)^{1/2}$, there are about $\exp(-\frac{r^2-1}{4})$ portion among them, \textit{i.e.}, $\Omega(n^{-1/2})$, which are simple graphs. 

\textbf{Step 2:} Now, we consider a graph that is uniformly sampled from the configuration model. Recall that the set of nodes that are associated with SPD$=k$ from $u$ is denoted by $Q_k$, and the number of edges that connect the nodes in $Q_k$ and those in $Q_{k+1}$ is denoted by $p_k$. Now, we prove that there exists a small constant $\epsilon > 0$, such that with probability $1-o(n^{-\frac{3}{2}})$, for all $k\in (\frac{\epsilon}{5}\frac{\log n}{\log(r-1)} + 1, (\frac{2}{3}-\epsilon)\frac{\log n}{\log(r-1)})$, $|Q_{k}| \geq (r-1-\epsilon)^{k-1} $ and $p_k \geq (r-1-\epsilon) |Q_{k}|$. We prove an even stronger lemma that gives the previous argument via a union bound and doing induction over all $k\in  (\frac{\epsilon}{5}\frac{\log n}{\log(r-1)} + 1, (\frac{2}{3}-\epsilon)\frac{\log n}{\log(r-1)})$. 
\begin{lemma} There exists a small constant $\epsilon>0$, with probability $1-O(n^{-{2+\epsilon}})$, such that: 1) For any $k<(\frac{2}{3}-\epsilon)\frac{\log n}{\log(r-1)}$, if $|Q_k| \geq n^{\epsilon/5}$,  $|Q_{k+1}| \geq p_k - |Q_{k}|^{1/2} $ and $p_k \geq (r-1) |Q_{k}| - |Q_{k}|^{1/2} $; 2) When $k=\lceil\frac{\epsilon}{5}\frac{\log n}{\log(r-1)}\rceil+1$, $|Q_k|\geq (r-1)^{k-1} = n^{\epsilon/5} $. 
\end{lemma}
\begin{proof}
We consider the following procedure to generate the graph based on the configuration model. Recall the node $u$ is the target node, or the node with $SPD=0$ to the target node. We start from generating the edges attached to this node. We start from the set $W_u$ and generate the $r$ pairs with at least one item in $W_u$. Then, we have all the nodes in $Q_1$. Based on the set $\cup_{v\in Q_1}W_v$, we generate all the $(r-1)|Q_1|$ pairs with at least one item in $\cup_{v\in Q_1}W_v$, and we have all the nodes in $Q_2$. The procedure goes on so on and so forth, from $Q_k$ to $Q_{k+1}$. 
%First, we prepare some inequalities. We have $\sum_{j=0}^k|Q_k| \leq (r-1)^{k+1} < n^{2/3-\epsilon}$ and for $i \geq \lceil |Q_{k}|^{1/2}\rceil$, $\frac{|Q_{k}|\cdot (r-1)|Q_{k}|}{(n - \sum_{j=0}^k|Q_k|)\cdot i} \leq n^{- \epsilon}$. Therefore 
%\begin{align}\label{eq:step2-1}
%\left[\frac{|Q_{k}|}{n - \sum_{j=0}^k|Q_k|} \right]^i {(r-1)|Q_{k}| \choose i} < \left[\frac{|Q_{k}|}{n - \sum_{j=0}^k|Q_k|} \right]^{\lceil |Q_{k}|^{1/2}\rceil} {(r-1)|Q_{k}| \choose \lceil |Q_{k}|^{1/2}\rceil} n^{-\epsilon(i - \lceil |Q_{k}|^{1/2}\rceil)},
%\end{align}
%and  for some constant $c_2$,
%\begin{align}\label{eq:step2-2}
%(n - \sum_{j=0}^k|Q_k|)^{\lceil |Q_{k}|^{1/2}\rceil} > c_2 n^{\lceil |Q_{k}|^{1/2}\rceil}.
%\end{align}

Now, we prove 1). First, we prepare some inequalities. We have $|Q_k| \leq r(r-1)^{k-1} < n^{2/3-\epsilon}$ by the assumption on $k$. For $i \geq \lceil |Q_{k}|^{1/2}\rceil$, we have
\begin{align}\label{eq:step2-1}
\frac{|Q_{k}|\cdot (r-1)|Q_{k}|}{n\cdot i} \leq n^{- \epsilon}
\end{align}
Moreover, Recall $|Q_{k}| < n^{2/3 -\epsilon}$. As $|Q_{k}| \geq n^{\epsilon/5}$, then 
\begin{align}\label{eq:step2-3}
\left(\frac{e(r-1)|Q_{k}|^{3/2}}{n}\right)^{\lceil |Q_{k}|^{1/2}\rceil} = O(n^{-{2+\epsilon}}).
\end{align}
This inequality is very crude. 
%Because of the induction $|Q_{k}| \geq (r-1)|Q_{k-1}| - 2|Q_{k-1}|^{1/2}$ and $|Q_{1}| \geq r \geq 3$, we obtain $|Q_{k}|\geq 3$. If $|Q_{k}| \leq \frac{4}{\epsilon^2} + 1$ and thus $r$ is constant,
%\begin{align}\label{eq:step2-4}
%\left(\frac{e(r-1)|Q_{k}|^{3/2}}{n}\right)^{\lceil |Q_{k}|^{1/2}\rceil} = O(n^{-2})
%\end{align}

Now, we go back to prove the bound for $p_k$. Recall the definition of $p_k$ that is the number of edges between $Q_k$ and $Q_{k+1}$. Then, the number of edges that are generated with both end-nodes are in $Q_k$ is $(r-1)|Q_k| - p_k$. As we suppose the edges are generated sequentially, the probability to generate an edge whose two end-nodes are in $Q_k$ is upper bounded by $\frac{(r-1)|Q_{k}|}{r(n - \sum_{j=0}^k|Q_k|)}<  \frac{|Q_{k}|}{n}$ where we use $\sum_{j=0}^k|Q_j| \leq r(r-1)^k= O(n^{2/3})$. Then, the probability that $(r-1)|Q_k| - p_k > |Q_{k}|^{1/2}$ is upper bounded by (just summing over all possible $(r-1)|Q_k| - p_k=i > |Q_{k}|^{1/2}$)
\begin{align*}
\sum_{i = \lceil |Q_{k}|^{1/2}\rceil}^{(r-1)|Q_{k}|} \left[\frac{|Q_{k}|}{n} \right]^i {(r-1)|Q_{k}| \choose i} 
\stackrel{\eqref{eq:step2-1}}{<}&\left[\frac{|Q_{k}|}{n} \right]^{\lceil |Q_{k}|^{1/2}\rceil} {(r-1)|Q_{k}| \choose \lceil |Q_{k}|^{1/2}\rceil}\sum_{i\geq 0} n^{-i\epsilon}\\
<&c_1\left[\frac{|Q_{k}|}{n} \right]^{\lceil |Q_{k}|^{1/2}\rceil} {(r-1)|Q_{k}| \choose \lceil |Q_{k}|^{1/2}\rceil} \\
<&c_2\left[\frac{|Q_{k}|}{n} \right]^{\lceil |Q_{k}|^{1/2}\rceil} \left[\frac{e(r-1)|Q_{k}|}{\lceil |Q_{k}|^{1/2}\rceil}\right]^{\lceil |Q_{k}|^{1/2}\rceil}\stackrel{\eqref{eq:step2-3}}{=} O(n^{-{2+\epsilon}})
\end{align*}
where $c_1, c_2$ are constants, the numbers above the equality/inequality signs refer to which equations are used. 

Next, we prove the bound for $|Q_{k+1}| \geq p_k - |Q_{k}|^{1/2}$. Again, if the edges are generated sequentially, $p_k - Q_{k+1}$ indicates the number of edges whose end-nodes in $Q_{k+1}$ also belong to other edges that has been generated between $Q_k$ and $Q_{k+1}$. The probability of this edge is upper bounded by $\frac{|Q_{k+1}|}{r(n - \sum_{j=0}^{k+1}|Q_{j}|)}<  \frac{(r-1)|Q_{k}|}{r(n - \sum_{j=0}^{k+1}|Q_{j}|)} \leq \frac{|Q_{k}|}{n}$.  Then, the probability that $p_k - Q_{k+1} > |Q_{k}|^{1/2}$ is again upper bounded by (just summing over all possible $p_k - Q_{k+1}=i > |Q_{k}|^{1/2}$)
\begin{align*}
\sum_{i = \lceil |Q_{k}|^{1/2}\rceil}^{(r-1)|Q_{k}|} \left[\frac{|Q_{k}|}{n} \right]^i {(r-1)|Q_{k}| \choose i}  =  O(n^{-{2+\epsilon}})
\end{align*}
Till now, we have proved the statement 1).

Now, we prove the statement 2). Actually, at the time when $Q_k$ is generated, the number of edges having been generated is at most $r(r-1)^k - 2$. These edges cover at most $r(r-1)^k - 1$ nodes. When $k=\lceil\frac{\epsilon}{5}\frac{\log n}{\log(r-1)}\rceil + 1$, we claim that with probability $1-O(n^{-2+\epsilon})$, at most 1 edge among these edges when generated is not connected to a new node. This is because if there are more than 1 such edges, the probability is at most (by summing over $i$ such edges)
\begin{align*}
\sum_{i = 2}^{r(r-1)^k } \left[\frac{r(r-1)^k}{n - r(r-1)^k} \right]^i {r(r-1)^k  \choose i} \leq c_3  \left(\frac{r^2n^{\epsilon/5}}{n} \right)^2 {r^2n^{\epsilon/5} \choose 2} = O(n^{-2+\epsilon}).
\end{align*}
We use this result to give a lower bound of $|Q_k|$ for $k=\lceil\frac{\epsilon}{5}\frac{\log n}{\log(r-1)}\rceil + 1$. Because there is at most 1 edge when generated do not connect to a new node. The worst case appears when two items in $W_u$ are mutually connected, which leads to $|Q_1| \geq r - 2 \geq 1$. All edges after $Q_1$ is generated are connected to new nodes and furthermore $|Q_k| \geq (r-1)^{k-1} \geq n^{\epsilon/5}$.
\end{proof}

\textbf{Step 3:} We start to consider the edge configuration between $Q_k$ and $Q_{k+1}$ for $k\in (\frac{1}{2}\frac{\log n}{\log(r-1-\epsilon)}, \frac{4}{7}\frac{\log n}{\log(r-1-\epsilon)})$. We focus our attention on the graphs that satisfy the properties developed in Step 2, which, as demonstrated in Step 2 , are with high probability $1- o(n^{-3/2})$. For those graphs, we know that for $k\in (\frac{1}{2}\frac{\log n}{\log(r-1-\epsilon)}, \frac{4}{7}\frac{\log n}{\log(r-1-\epsilon)})$, $p_k \geq (r-1-\epsilon)|Q_k| \geq (r-1-\epsilon)^k\geq n^{1/2}$ and $p_k\leq r(r-1)^k < n^{2/3 - \epsilon}$.  Moreover, $\sum_{j=1}^k |Q_k| \leq (r-1)|Q_k| = o(n)$ and therefore at the time when $Q_k$ is generated, there are still $q_k = n- o(n) = \Theta(n)$ nodes that have not been connected.

Recall that we define the edge configuration between $Q_k$ and $Q_{k+1}$ is a list $C_k = (a_{1,k}, a_{2,k},...)$ where $a_{i,k}$ means the number of nodes in $Q_{k+1}$ of which each has exactly $i$ edges from $Q_k$. According to the definition of $C_k$, it satisfies 
\begin{align}
\sum_{i=1}^r i\times a_{i, k} = p_k 
\end{align}

Note that if \proj cannot distinguish $(u, A^{(1)})$ and $(u, A^{(2)})$, then $A^{(1)}$ and $A^{(2)}$ must share the same edge configuration between $Q_k$ and $Q_{k+1}$. Otherwise, after one iteration, the intermediate representation of nodes in $Q_{k+1}$ are different between $A^{(1)}$ and $A^{(2)}$. Such difference will be propagated to $u$ later. %Suppose the set of all possible edge configurations given fixed $p_k$ and $q_k$ as $\mathbb{C}_k$. 
To bound the probability that $A^{(1)}$ and $A^{(2)}$ must share the same edge configuration between $Q_k$ and $Q_{k+1}$, for simplicity, we consider the probability of $C_k$ given the number of edges between $Q_k$ and $Q_{k+1}$, \textit{i.e.}, $p_k$ and the number remaining nodes, \textit{i.e.}, $q_k = [n]/\cup_{i=1}^k Q_i = \Theta(n)$. We are to derive a upper bound of $\mathbb{P}(C_k)$ based on the configuration model in the following lemma.
\begin{lemma}\label{lemma:step3}
Suppose $ p_k\in [n^{1/2}, n^{2/3-\epsilon}]$ and $q_k = \Theta(n)$. Consider the configuration model to generate edges: there are $p_k$ edges that correspond to two items  that are one in $\cup_{v\in Q_k}W_v$ and one among the rest $q_k r$ items. Then, for any possible edge configuration $C_k$ obtained based on this generating procedure, $\mathbb{P}(C_k) \leq c_5\frac{q_k^{1/2}}{p_k}$ for some constant $c_5$.
\end{lemma}
\begin{proof}
First, for the configuration $C_k = (a_{1,k}, a_{2,k},...)$, we claim that the most probable $C_k$ is achieved when $a_{i,k} = 0$ for $i\geq 3$. We prove this statement via the adjustment method: We fix the value of $a_{1,k} + i a_{i,k}$ and all the other $a_{i',k}$'s. We compare the probability of $(a_{1,k} = x, a_{i,k} = y)$ and that of $(a_{1,k} = x + i, a_{i,k} = y-1)$. Because $x, y\leq p_k = O(n^{2/3-\epsilon})$, for some constant $c_6$, we have
\begin{align*}
\frac{\mathbb{P}(a_{1,k} = x, a_{i,k} = y)}{\mathbb{P}(a_{1,k} = x+i, a_{i,k} = y-1)} &= \frac{{q_k \choose x}r^x{q_k - x \choose y} {r \choose i}^y}{{q_k \choose x+i}r^{x+i}{q_k - x - i \choose y-1} {r \choose i}^{y-1}} \leq c_6 \frac{x^i}{q_k^{i-1}} \leq c_6 \frac{x^3}{q_k^2} < 1.
%\frac{(q_k - x  - y + 1)}{y} \cdot \frac{(x+1)(x+2)\cdots (x+i)}{(q_k - x)(q_k - x - 1)\cdots (q_k - x + 1 - i)}
\end{align*}
Therefore, we only need to consider the case when $a_{1,k}, a_{2,k} > 0$ so $a_{1,k} + 2a_{2,k} = p_k$. Define a function $g(x)$ to denote the probability of the edge configuration $(a_{1,k} =  p_k - 2y, a_{2,k} = y)$. We compare $g(y)$ and $g(y+1)$
\begin{align}\label{eq:step3-3}
\frac{g(y)}{g(y+1)} = \frac{{q_k \choose p_k - 2y}r^{p_k - 2y}{q_k - p_k + 2y \choose y} {r \choose 2}^y}{{q_k \choose p_k - 2y - 2}r^{p_k - 2y - 2}{q_k - p_k + 2y + 2 \choose y+1} {r \choose 2}^{y +1}} =\frac{2r}{(r-1)} \frac{(y+1)(q_k - p_k + y + 1)}{(p_k - 2y)(p_k - 2y - 1)}.
\end{align}
Consider the choice $y=y^*$ to make $g(y^*)/g(y^*+1)\geq 1$ while  $g(y^*-1)/g(y^*)\leq 1$ that corresponds to $g(y^*) = \max_{y} g(y)$. Then, we must have $y^* = o(p_k)$ and otherwise $g(y^*-1)/g(y^*) > 1$ because $q_k=\Theta(n)$. As  $y^* = o(p_k)$ and $p_k = o(q_k)$, $y^*$ according to Eq.~\ref{eq:step3-3} is about $\frac{(r-1)p_k^2}{2rq_k}$  by setting Eq.~\ref{eq:step3-3}$=1$. We define $y_0 = \frac{(r-1)p_k^2}{2rq_k}$. Consider $y_0 + \delta$ where $\delta = o(y_0)$. Then, using $p_k = O(n^{2/3-\epsilon})$ and hence $\frac{y_0p_k}{q_k} = o(1) = o(\delta)$, we have
\begin{align*}
\frac{g(y_0 + \delta)}{g(y_0 + \delta +1)} = 1 + \frac{\delta}{y_0} + o(\frac{\delta}{y_0}).
\end{align*}   
Moreover, for $\delta > 0$
\begin{align*}
\frac{g(y_0)}{g(y_0 + \delta)} = \prod_{j=0}^{\delta-1} (1 + \frac{j}{y_0} + o(\frac{\delta}{y_0})) \leq 1 + \frac{\delta(\delta-1)}{2y_0}  + o(\frac{\delta(\delta-1)}{2y_0}), \\
\frac{g(y_0)}{g(y_0 - \delta)} = \prod_{j=1}^{\delta-1} (1 - \frac{j}{y_0} + o(\frac{\delta}{y_0}))^{-1} \leq 1 + \frac{\delta(\delta-1)}{2y_0}  + o(\frac{\delta(\delta-1)}{2y_0}). 
\end{align*}
Choose $\delta = y_0^{1/2}$, then $$g(y_0 +  y_0^{1/2}),  g(y_0 -  y_0^{1/2})\geq (\frac{2}{3} + o(1))g(y_0).$$ As $\sum_{j = - y_0^{1/2}}^{y_0^{1/2}} g(j) \leq 1$ and $y^*$ should be in $[y_0 -  y_0^{1/2}, y_0 +  y_0^{1/2}]$, we obtain $$g(y^*) \leq \frac{3}{4 y_0^{1/2}} = c_5 \frac{q_k^{1/2}}{p_k},$$ which concludes the proof. 
\end{proof}

Now, we go back to consider any $\epsilon \frac{\log n}{\log (r-1-\epsilon)}$ many $k$'s in $(\frac{1}{2}\frac{\log n}{\log(r-1-\epsilon)}, \frac{4}{7}\frac{\log n}{\log(r-1-\epsilon)})$. Based on Lemma~\ref{lemma:step3}, the probability that $A^{(1)}$ and $A^{(2)}$ share the same edge configurations between $Q_k$ and $Q_{k+1}$ for all these $k$'s is bounded by 
\begin{align*}
\prod_{k = c_7\frac{\log n}{\log(r-1-\epsilon)}}^{(c_7+\epsilon)\frac{\log n}{\log(r-1-\epsilon)}}  \mathbb{P}(C_k) \leq &\prod_{k = c_7\frac{\log n}{\log(r-1-\epsilon)}}^{(c_7+\epsilon)\frac{\log n}{\log(r-1-\epsilon)}} c_5\frac{q_k^{1/2}}{p_k} < n^{\epsilon \log c_5} \frac{n^{\frac{\epsilon}{2}\frac{\log n}{\log(r-1-\epsilon)}}} {n^{(c_7 + \frac{\epsilon}{2})\epsilon\frac{\log n}{\log(r-1-\epsilon)}}} \\
< &n^{- \frac{\epsilon^2}{3} \frac{\log n}{\log\log n}} = o(n^{-3/2}).
 \end{align*}

\textbf{Step 4:} From step 2, we know that the graphs that do not satisfy $|Q_{k}| \geq (r-1-\epsilon)^{k-1} $ and $p_k \geq (r-1-\epsilon) |Q_{k}|$ for $k\in  (\frac{\epsilon}{5}\frac{\log n}{\log(r-1)} + 1, (\frac{2}{3}-\epsilon)\frac{\log n}{\log(r-1)})$ are only $o(n^{-3/2})$ portion of all graphs generated from  the configuration model. From step 3, we know that $A^{(1)}$ and $A^{(2)}$ that satisfy the properties of step 2, with probability at most $o(n^{-3/2})$, their subtrees rooted at node $u$ are the same even with DE-1's. Even if all these graphs belong to simple regular graphs, as step 1 tells the portion of simple regular graphs among the graphs generated from the configuration model is $\Omega(n^{-1/2})$, we arrive at the final conclusion that if $A^{(1)}$ and $A^{(2)}$ are sampled from all simple $r$-regular graphs, with probability at least $1 - o(n^{-3/2})/\Omega(n^{-1/2}) = 1 - o(n^{-1})$, a proper \proj can distinguish $(u, A^{(1)})$ and $(u, A^{(2)})$.
\end{proof}

\section{Discussion on node sets over irregular graphs}\label{sec:irregular}
Theorem~\ref{thm:power} focuses on node sets embedded in regular graphs. A natural question therefore arises: how about the power of DEs to distinguish non-isomorphic node sets embedded in irregular graphs that the 1-WL test (also WLGNN) may not distinguish. %A full investigation of this question is out of the scope of this work. 
To answer this question, note that there is some important connection between irregular graphs and regular graphs under the umbrella of the 1-WL test. Actually, the partition of nodes over irregular graphs according to their representations (colors) stably associated by the 1-WL test has \emph{equitable} property~\cite{arvind2019weisfeiler}. Basically, suppose that the whole node set $V$ can be partitioned into several parts based on the representations (colors) of nodes, $V=\cup_{i=1}^c V_i$, where for an arbitrary $i$, nodes in $V_i$ share the same representations based on the 1-WL test. Then, the induced subgraph of the nodes in $V_i$ is a regular graph for all $i$'s. What's more, for any $i,j$, the number of nodes in $V_j$ that are neighbors of a certain node in $V_i$ is shared by all the nodes in $V_i$, which again shows certain regularity. Theoretically, if we focus on the regular subgraph defined over each $V_i$, we may further leverage DE defined over this subgraph to further distinguish the nodes in $V_i$. In practice, we do not need to work on each subgraph individually. Mostly, the capability of DE indicated by Theorem~\ref{thm:power} may still help with breaking such regularity.

%Theorem~\ref{thm:power} focuses on node sets embedded in regular graphs. A natural question therefore arises: how about the power of DEs to distinguish non-isomorphic node sets embedded in irregular graphs that the 1-WL test (also WLGNN) may not distinguish. A full investigation of this question is out of the scope of this work. However, there is some important connection between irregular graphs and regular graphs under the umbrella of the 1-WL test. Actually, the partition of nodes over irregular graphs according to their representations (colors) stably associated by the 1-WL test has \emph{equitable} property~\cite{arvind2019weisfeiler}. Basically, suppose that the whole node set $V$ can be partitioned into several parts based on the representations (colors) of nodes, $V=\cup_{i=1}^c V_i$, where for an arbitrary $i$, nodes in $V_i$ share the same representations based on the 1-WL test. Then, the induced subgraph of the nodes in $V_i$ is a regular graph for all $i$'s. What's more, for any $i,j$, the number of nodes in $V_j$ that are neighbors of a certain node in $V_i$ is shared by all the nodes in $V_i$, which again shows certain regularity. In some sense, the capability of DE indicated by Theorem~\ref{thm:power} may still help with breaking such regularity.  

%Therefore, the 
%neighbors with representation $\Gamma_1$ of an arbitrary node with representation $\Gamma_0$ only depends on $\Gamma_1, \Gamma_0$. Hence, the representations of nodes . Therefore, we may conjecture that Theorem~\ref{thm:power} also indicates an strong power of \proj for those more general cases. 

\section{Proof for The Limitation of DE-1 ---  Theorem~\ref{thm:limit}}
We restate Theorem~\ref{thm:limit}: Consider any two nodes $v, u\in V$. Consider two tuples $\mathcal{T}_1=(v, \mathbf{A}^{(1)})$ and $\mathcal{T}_2=(u, \mathbf{A}^{(2)})$ with graph topologies $A^{(1)}$ and $A^{(2)}$ that correspond to two connected DRGs with a same intersection array. Then, \proj-1 must depend on discriminatory node/edge attributes to distinguish $\mathcal{T}_1$ and $\mathcal{T}_2$.
\begin{proof}
We are to prove that if $A^{(1)}$ and $A^{(2)}$ correspond to two DRGs with a same intersection array, then the subtrees rooted at any nodes are all same even if the nodes are labeled with any DE-1's~(see illustration of a subtree rooted at a node in Fig.~\ref{fig:subtrees}). Because if the subtrees are same, the only possibility to differentiate two nodes is based on discriminatory node/edge attributes embedded in these subtrees when \proj processes them from the bottom to the top. %Therefore, there exists an WLGNN without using DEs that can differentiate these two nodes.

We recall the definition of the intersection array of a connected DRG as the following. 
\begin{definition}
The intersection array of a connected DRG with diameter $\triangle$ is an array of integers $\{b_0, b_1, ..., b_{\triangle-1};c_1, c_2, ..., c_{\triangle}\}$ such that for any node pair $(u, v)\in V\times V$ that satisfies SPD$(v, u)=j$, $b_j$ is the number of nodes $w$ that are neighbors of $v$ and satisfy SPD$(w, u)=j+1$, and $c_j$ is the number of nodes $w$ that are neighbors of $v$ and satisfy SPD$(w, u)=j-1$. %Note that $b_j$ and $c_j$ are fixed for any node pair $(u, v)\in V\times V$ that satisfies SPD$(v, u)=j$. 
\end{definition}
The definition of DRG implies the following lemma.
\begin{lemma}\label{lemma:limit-1}
Suppose each node is associated with SPD as DE-1. Consider a graph with the  intersection array $\{b_0, b_1, ..., b_{\triangle-1};c_1, c_2, ..., c_{\triangle}\}$. For an arbitrary node $u$, any node $v$ in the subtree rooted at $u$, with SPD($v,\, u) = j$, has children including $b_j$ nodes $w$ with DE-1 satisfying SPD($w,\, u)=j+1$,  $c_j$ nodes with DE-1 satisfying SPD$=j-1$, and $b_0 - b_j - c_j$ nodes $w$ with DE-1 satisfying SPD($w,\, u)=j+1$.
\end{lemma}
\begin{proof}
A DRG with the intersection array $\{b_0, b_1, ..., b_{\triangle-1};c_1, c_2, ..., c_{\triangle}\}$ is a $b_0$-regular graph and thus $v$ has $b_0$ neighbors. All the neighbors of $v$ become children of $v$ in the subtree. As SPD($v,\, u) = j$, we know that SPDs from the neighbors of $v$ to $u$ are in $\{j-1, j, j+1\}$. According to the definition of intersection array. we know the numbers of neighbors of $v$ with different SPDs, $j-1, j, j+1$, exactly are $c_j, b_0 - c_j - b_j, b_j$ respectively.
\end{proof}
We start from any node $u$ and construct the subtree rooted at $u$. By using the Lemma~\ref{lemma:limit-1}, it is obvious that subtrees for any nodes from DRGs with a same intersection array are all the same even if all the nodes use SPD as DE-1. An illustrative example of this result is shown in Fig.~\ref{fig:DRG-DE-1}.

\begin{figure}[t]
\centering
\includegraphics[trim={0.5cm 12cm 1cm 4cm},clip,width=\textwidth]{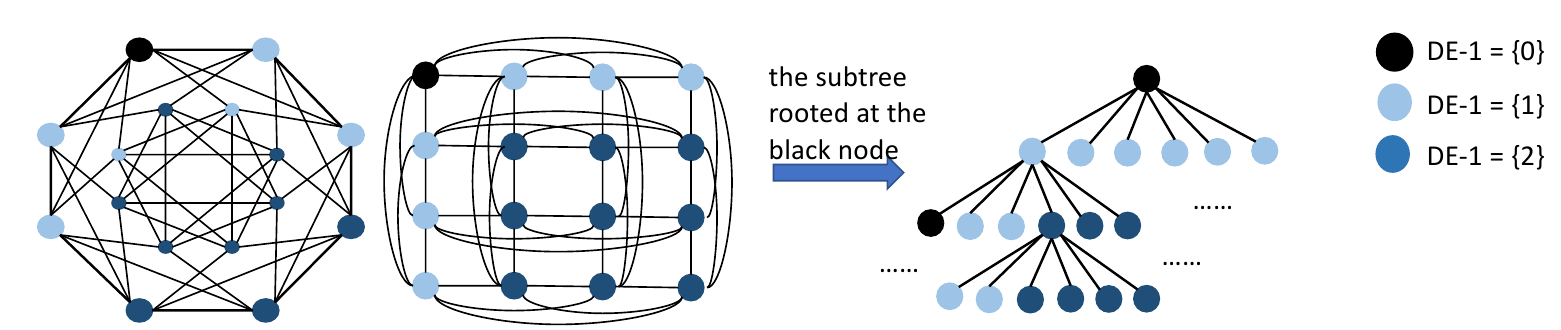}
\caption{The subtrees rooted at two nodes in the Shrikhande graph and the $4\times 4$ rook's graph respectively: In the left two graphs, the black nodes are the target nodes who structural representations are to be learnt. Different colors of the nodes correspond to different DE-1 (SPDs) with respect to the target nodes. For both the Shrikhande graph and the $4\times 4$ rook's graph, the subtrees rooted at the black nodes share the same spreading colors as shown in the right. As these two graphs are DRGs with the same intersection array, the configuration of colors (DE-1) of children only depends on the color (DE-1) of their father node.}
\label{fig:DRG-DE-1}
\end{figure}

The next step is to generalize SPD as DE-1 to the list of landing probability as DE-1. Actually, the following lemma indicates that in DRGs, SPD and the list of landing probability are bijective. 
\begin{lemma}\label{lemma:limit-2}
In any connected DRG with the same intersection array, the number of walks of given length between nodes depends only on the SPD between these vertices.
\end{lemma}
\begin{proof}
Given any node $u\in V$ over the DRG with  the intersection array $\mathcal{L} = \{b_0, b_1, ..., b_{\triangle-1};c_1, c_2, ..., c_{\triangle}\}$, denote the number of walks of length $l$ from $u$ to another node $v$ is $f(v,l)$. We need to prove that $f(v,l)$ can be written as a function $g(\text{SPD}(u,v), l, \mathcal{L})$, which is indepdent from the node identities. We prove this result via induction over the length of walks denoted by $l$. The statement is trivial for $l=1$. Suppose the statement is true for $l= l_0-1$, we consider the case $l=l_0$. 

Because of the definition of walks, there is recurrence relation 
\begin{align*}
f(u,v,l_0) &= \sum_{w\in \mathcal{N}_v} f(u,w,l_0-1) \\
&= \sum_{w:\text{SPD}(w,u) = \text{SPD}(v,u) - 1} f(u,w,l_0-1) + \sum_{w:\text{SPD}(w,u) = \text{SPD}(v,u) + 1} f(u,w,l_0-1) 
\end{align*}
Because the assumption of induction, we have 
\begin{align*}
f(u,v,l_0) &= c_{\text{SPD}(v,u)}g(\text{SPD}(v,u) - 1, l_0-1, \mathcal{L}) + b_{\text{SPD}(v,u)}g(\text{SPD}(v,u) + 1, l_0-1, \mathcal{L}) 
\end{align*}
which only depends on $\text{SPD}(v,u), l_0$ and $\mathcal{L}$ and thus can be written as $g(\text{SPD}(u,v), l_0, \mathcal{L})$. %Note that, in the subtree rooted at node $u$, a path from the root to a node of the subtree that corresponds to node $v$ in the graph essentially gives a walk from $u$ to $v$. And the length of this walk is the depth of this node in the subtree. Because of Lemma~\ref{lemma:limit-1}, all the nodes with the same SPD from $u$ always have the same number of children featured with the same SPD to $u$. So the numbers of paths in the subtree from the root to nodes with the same SPD from $u$ are always same, which concludes the statement. 
\end{proof}
Actually, Lemma~\ref{lemma:limit-2} is closely related the argument in~\cite{RePEc:oxp:obooks:9780198514978}, which claims the same result for walks over one DRG. However, Lemma~\ref{lemma:limit-2} extends the argument to any DRGs with the same intersection array. As DRGs are $b_0$-regular graphs, there is a bijective mapping between the list of landing probabilities and the list of walks of different length. Therefore, this is a bijective mapping between SPD and the list of landing probabilities, which concludes the proof. 
\end{proof}

\section{Proof for \projA --- Corollary~\ref{col:PEAGNN} and Further Discussion } \label{apd:PEAGNN}
\projA contains a general aggregation procedure assisted by DE-1 (Eq.~\eqref{eq:agg-pe}). As we discussed in Section~\ref{sec:related}, as \projA allows to use DEs as extra node features as \proj in the same time, it has at least the same representation power as \proj. Therefore, Theorem~\ref{thm:power} and Corollary~\ref{col:graphpower} are still true for \projA. So \textbf{our first question} is whether \projA shares the same limitation with \proj with DE-1 for node representation learning over DRGs. Later, we will prove our confirmation to this question.

An interesting case in practice is to set the DE-1 in Eq.~\eqref{eq:agg-pe} as SPD $\zeta(u|v) = \zeta_{spd}(u|v)$. For some $K \geq 1$, we specify the aggregation as 
\begin{align}\label{eq:agg-pe-app}
\text{AGG}(\{f_2(h_u^{(l)}, \mathbf{A}_{vu})\}_{u\in \mathcal{N}_v})   \rightarrow \text{AGG}(\{(f_2(h_u^{(l)}, \mathbf{A}_{vu}),  \zeta_{spd}(u|v))\}_{\zeta_{spd}(u|v)\leq K}),
\end{align}
which means that the aggregation happens among $K$-hop neighbors. If the model is to learn the representation of a node subset with size $p$, We term this model as \projA-$p$-$K$-hop. \textbf{The second question} is to investigate whether \projA-$p$-K-hop may decrease the number of layers of \proj required in Theorem~\ref{thm:power} and Corollary~\ref{col:graphpower} by a factor of $K$. This result is not trivial. For example, Fig.~\ref{fig:mixhop-hard} shows two trees whose root nodes are the target nodes to learn structural representation. Obviously, \proj-1 needs two layers to distinguish these two root nodes. However, \projA-1-2-hop may not decrease the number of layers by a factor of $2$ and thus still needs two layers. This is because the set aggregation in Eq.~\eqref{eq:agg-pe-app} may decrease the discriminatory power of those features interacted with graph structures. However, next, we can still prove the confirmation to this second questions.

\begin{figure}[t]
\centering
\includegraphics[trim={3.5cm 14cm 14cm 4cm},clip,width=0.7\textwidth]{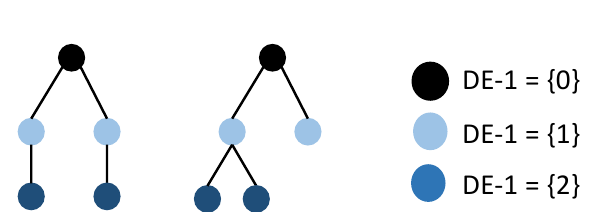}
\caption{$K$-hop aggregation does necessarily better the discriminatory power. Consider using \proj-1 and \projA-1-2-hop to learn the structural representation of the nodes colored by black. We choose SPD as DE-1. Both models require at least two layers to distinguish two black nodes. \projA-1-2-hop cannot decrease the number of layers by a factor of 2.}
\label{fig:mixhop-hard}
\end{figure}

%Another interesting case is to investigate a weaker PEAGNN with only general aggregation procedure assisted by 1-PE~\eqref{eq:agg-pe} but not using PEs concatenated to initial node features. We term this weaker PEAGNN as w-PEAGNN. Note that w-PEAGNN is not necessarily as powerful as PEGNN with 1-PE, because if we choose all 1-PEs as SPD, it can be shown that the two black nodes in Fig.~\ref{fig:subtrees} cannot be distinguished by w-PEAGNN. So \textbf{the second question} is whether w-PEAGNN still satisfies the arguments in Theorem~\ref{thm:power} (that yields Corollary~\ref{col:graphpower} as a by-product). Next, we will provide confirmation for the first questions. The second question is left for a future study. 

\textbf{Confirmation to the first question.} We formally restate the conclusion to prove: Consider any two nodes $w_1, w_2\in V$. Consider two tuples $\mathcal{T}_1=(w_1,A^{(1)})$ and $\mathcal{T}_2=(w_2,A^{(2)})$ with graph topologies $A^{(1)}$ and $A^{(2)}$ that correspond to two connected DRGs with a same intersection array. Then, \projA-1 must depend on discriminatory node/edge attributes to distinguish $\mathcal{T}_1$ and $\mathcal{T}_2$.
\begin{proof}
Because Lemma~\ref{lemma:limit-2} implies that in all DRGs with the same intersection array, there is a bijective mapping between SPD and the list of landing probabilities. Therefore, we only need to focus on the case that uses SPDs as DE-1, which both appears as extra node attributes and features in aggregation (see Eq.~\eqref{eq:agg-pe}). As the statement is about the limitation, we consider the most expressive case, \textit{i.e.}, the aggregation appearing among every pair of nodes. 

Recall the tree structure to compute \proj is termed as the subtree rooted at some node. Now, we call the tree structure to compute \projA for the structural representation of a node $w$ as the \emph{extended} subtree rooted at $w$, which has the same utility as the tree structure for \proj (Fig.~\ref{fig:subtrees}). Suppose we are to learn the structural representation of $w$. We are going to prove by induction that with arbitrary number of layers, if no discriminatory node/edge attributes are available, any node $v$ will be associated with an representation vector that only depends on SPD$(v|w)$ and the intersection array of the graph.  If this is true, then we know that any nodes in DRGs with a same intersection array will have the same representation output by \projA with an arbitrary number of layers. 

Recall that the number of layers is $L$. Obviously, when $L=0$, the node representations only depend on SPD$(v|w)$ and thus satisfy the statement. Suppose the statement is true when $L = L_0-1$, consider the case when $L=L_0$. Denote the representation of a node $v$ after $L_0-1$ layers as $h_v^{(L_0-1)}$. Then, its representation after $L_0$ layers is
\begin{align*}
    h_v^{(L_0)} = f_1( h_v^{(L_0-1)}, \text{AGG}\{(f_2(h_u^{(L_0-1)}), \zeta_{spd}(u|v))\}_{u\in V})).
\end{align*}
Note that we omit the edge attributes $\mathbf{A}_{uv}$ due to the requirement of the statement. Consider another node $v'$ who satisfies SPD$(v'|w)$ = SPD$(v|w)$. For this, we only need to prove that the following two components are the same and thus $h_v^{(L_0)} = h_{v'}^{(L_0)}$: 
\begin{align}
    &h_v^{(L_0-1)} = h_{v'}^{(L_0-1)}, \label{eq:lemma-limit-eq1}\\
    &\{(h_u^{(L_0-1)}, \zeta_{spd}(u|v))\}_{u\in V} = \{(h_u^{(L_0-1)}, \zeta_{spd}(u|v'))\}_{u\in V}. \label{eq:lemma-limit-eq2}
\end{align} 
The Eq.~\eqref{eq:lemma-limit-eq1} is directly due to the assumption of induction. To prove Eq.~\eqref{eq:lemma-limit-eq2}, we first partition all nodes in $V$ according to $\zeta_{spd}(u|v), \zeta_{spd}(u|v')$ by defining $S_{v}(a) = \{u\in V |\zeta_{spd}(u|v) = a\}$ and $S_{v'}(a)=\{u\in V |\zeta_{spd}(u|v') = a\}$. We further partition these two sets according to $\zeta_{spd}(u|w), \zeta_{spd}(u|w)$ by defining $ S_{v}(a,b) =\{u\in S_{v}(a) | \zeta_{spd}(u|w) = b\}$ and $S_{v'}(a,b) =\{u\in S_{v'}(a) | \zeta_{spd}(u|w) = b\}$. By using the definition of DRG, we know that $|S_{v}(a,b)| = |S_{v'}(a,b)|$ and such cardinality only depends on the intersection array $\mathcal{L}$. Moreover, using the assumption of induction, all nodes u in $S_{v}(a,b), S_{v'}(a,b)$ share the same representation $h_u^{(L_0-1)}$. Combining the fact that $|S_{v}(a,b)| = |S_{v'}(a,b)|$ and the fact the nodes in these two sets hold the same representation, we may claim the second Eq.~\eqref{eq:lemma-limit-eq2} is true, which concludes the proof.

%Since the new aggregation procedure is over all nodes, each node of the extended subtree will have $n$ children. Moreover, each node of the extended subtree is not only labeled with SPD to the root node as the initial node features but also with SPD to its direct father node to characterize the 1-PE in equation~\eqref{eq:agg-pe}. To prove the results, we only need to show that the extended subtrees rooted at any node of DRGs with a same intersection array are all the same..

%We are to mimic Lemma~\ref{lemma:limit-1} to prove that the labels on children of an internal node $v$ of the extended subtree only depend on the SPD between node $u$ and $v$. We prove this statement by partitioning those children according to their labels and proving the number of children in each category depends on the SPD between node $u$ and $v$. As the labels consist of two parts: SPDs to the direct father nodes, SPDs to the root node. First, all these $n$ children can be partitioned into different categories with respect to their SPDs to its father node $v$. It is easy to show that the same partitioning is shared by all nodes of DRGs with the same intersection array. Moreover, we further partition the children in each category based on their SPDs to the root node. Because of Definition~\ref{def:DRG} for DRGs, we know that the number of children in each final category only depends on the SPD between their father node $v$ and the root node $u$, which concludes the proof.
\end{proof}

%\textbf{Answer to the second question:} We formally restate the conclusion to prove: Given two fixed-sized sets $S^{(1)}, S^{(2)}\subset V$, $|S^{(1)}| = |S^{(2)}|$. Consider two tuples $\mathcal{T}^{(1)}=(S^{(1)}, \mathbf{A}^{(1)})$ and $\mathcal{T}^{(2)}=(S^{(2)}, \mathbf{A}^{(2)})$ in the most difficult setting when features $\mathbf{A}^{(1)}$ and $\mathbf{A}^{(2)}$ are only different in graph topologies specified by $A^{(1)}$ and $A^{(2)}$ respectively. Suppose $A^{(1)}$ and $A^{(2)}$ are uniformly independently sampled from all r-regular graphs over $V$ where $3\leq r< (2\log n)^{1/2}$. Then, within iterations $T < O(\log n/\log(r-1))$, there exist a proper w-PEAGNN using 1-PE $\zeta(u|v)$ in \eqref{eq:agg-pe} to control the propagation, such that with probability $1-o(n^{-1})$, the outputs $\Gamma(\mathcal{T}^{(1)})\neq \Gamma(\mathcal{T}^{(2)})$. Specifically, $\zeta(u|v)$ can be simply chosen as SPD, \textit{i.e.}, $\zeta_{spd}(u|v)$.
%\begin{proof}
%Similar to the proof for PEGNN, we only need to prove the case when $S^{(1)} = S^{(2)} =\{u\}$ and $\zeta(u|v) = \zeta_{spd}(u|v)$ (see Lemma~\ref{lemma:step0}).
%
%
%
% 
%\end{proof}

\textbf{Confirmation to the second question.} We formally restate the conclusion to prove: Given two fixed-sized sets $S^{(1)}, S^{(2)}\subset V$, $|S^{(1)}| = |S^{(2)}|=p$. Consider two tuples $\mathcal{T}^{(1)}=(S^{(1)}, \mathbf{A}^{(1)})$ and $\mathcal{T}^{(2)}=(S^{(2)}, \mathbf{A}^{(2)})$ in the most difficult setting when features $\mathbf{A}^{(1)}$ and $\mathbf{A}^{(2)}$ are only different in graph topologies specified by $A^{(1)}$ and $A^{(2)}$ respectively. Suppose $A^{(1)}$ and $A^{(2)}$ are uniformly independently sampled from all r-regular graphs over $V$ where $3\leq r< (2\log n)^{1/2}$. Then, for some constant $\epsilon > 0$ and constant positive integer $K$, there exist a proper \projA-p-K-hop with layers $L < \lceil (\frac{1}{2} + \epsilon)\frac{\log n}{K\log(r-1)}\rceil$, using DE-$p$ $\zeta(u|S^{(1)}),\,\zeta(u|S^{(2)})$ for all $u\in V$ such that with probability $1-o(n^{-1})$, its outputs $\Gamma(\mathcal{T}^{(1)})\neq \Gamma(\mathcal{T}^{(2)})$. 

\begin{proof}
Similar to the proof of Theorem~\ref{thm:power}, we use Lemma~\ref{lemma:power-step0} and focus on the case $S^{(1)} = S^{(2)}= \{w\}$ and the initial node attribute for each node $u$ are $\zeta_{spd}(u|w )$. Most of the logic of the proof is the same as that of Theorem~\ref{thm:power}. We only need to take care of the step 3 of the proof  of Theorem~\ref{thm:power}, as in \projA-1-$K$-hop, a node $v$ with SDP$(v|w)=k$ will aggregation representations of nodes $u$ even with SDP$(u|w)\in [k-K, k+K]$. Therefore, we need to redefine the edge configuration in the step 3 of the proof of Theorem~\ref{thm:power}. 

Recall $Q_{k} = \{v\in V| \text{SDP}(v|w)=k\}$. We define the edge configuration between $Q_k$ and $\cup_{i\in[k+1, k+K]}Q_{i}$ as a list $\bar{C}_k = ((a_{1,k,1}, a_{2,k,1}, ...), (a_{1,k,2}, a_{2,k,2}, ...), ..., (a_{1,k,K}, a_{2,k,K}, ...))$, where $a_{m,k,j}$ is the number of nodes in $Q_{k+j}$ of which each connects to exactly $m$ nodes in $Q_{k+j-1}$. Note that two different $\bar{C}_k$'s will lead to two different representations of $(w, A^{(1)})$ and $(w, A^{(2)})$ after layers $\lceil \frac{k}{K}\rceil + 2$ as \projA-1-$K$-hop uses at most 2 layers yield $\{h_u^{(2)}|\zeta_{spd}(u|w, A^{(1)}) = k\} \neq \{h_u^{(2)}|\zeta_{spd}(u|w, A^{(2)}) = k\}$ and uses at most $\lceil \frac{k}{K}\rceil$ layers to propagate such difference to the target node $w$. 

Actually, this definition of edge configuration is nothing but a concatenation of the edge configurations $\{C_i\}_{i\in[k, k+K-1]}$ where $C_i$ is the edge configuration between $Q_i$ and $Q_{i+1}$ as defined in the proof of Theorem~\ref{thm:power}. Then, we use the statement of step 3 on $C_i$ to characterize the probabilitic property of $\bar{C}_k$. For each $k \in (\frac{1}{2}\frac{\log n}{\log(r-1-\epsilon)}, \frac{4}{7}\frac{\log n}{\log(r-1-\epsilon)})$, each type of edge configuration $\bar{C}_k$ appears with only limited probability $\mathbb{P}(\bar{C}_k) = \Pi_{i=k}^{k+K-1} \mathbb{P}(\bar{C}_i) =  \Pi_{i=k}^{k+K-1} O(\frac{n^{1/2}}{p_i})$. Then, we consider $\epsilon \frac{\log n}{K\log(r-1-\epsilon)}$ many $k$'s in $(\frac{1}{2}\frac{\log n}{\log(r-1-\epsilon)}, \frac{4}{7}\frac{\log n}{\log(r-1-\epsilon)})$ such that these $k$'s hold the same integral interval $K$ and can be denoted as $k_0, k_0+K, ...$. The probability that $A^{(1)}$ and $A^{(2)}$ have all the same edge configurations for these $k$'s is bounded by 
\begin{align*}
\Pi_{k\in\{k_0, k_0+K,...\}} \mathbb{P}(\bar{C}_k) = \Pi_{i=k_0}^{k_0 + \epsilon \frac{\log n}{\log(r-1-\epsilon)}-1} \mathbb{P}(C_i)\sim o(n^{-\frac{3}{2}}).
\end{align*}
Therefore, we only need to consider edge configurations for $k\in  (\frac{1}{2}\frac{\log n}{\log(r-1-\epsilon)}, (\frac{1}{2} + \epsilon)\frac{\log n}{\log(r-1-\epsilon)})$ to distinguish $A^{(1)}$ and $A^{(2)}$. And within $\lceil(\frac{1}{2} + \epsilon)\frac{\log n}{K\log(r-1-\epsilon)}\rceil + 2$, \projA-1-$K$-hop yields different representaions for $(w, A^{(1)})$ and $(w, A^{(2)})$. Note that the constant 2 may be merged in $\epsilon$ for simplicity, which concludes the proof.
\end{proof}

\section{Details of the Experiments}  \label{sec:exp_supple}
\subsection{Datasets} \label{sec:data_supple}
%\begin{itemize}
%, any baselines that were evaluated o    \item The details of each datasets: statistics, where to download these datasets. Why do we choose these datasets (particularly those airport networks). For the triangle datasets, how many triangles in total?
%    \item Our approaches: the implementation specifics of those four models: we use which equations to design those models. 
%    \item Baselines: which codes of these baselines are used. How did we tune these models. 
%    \item Particularly, explain how our DE-2 is different from SEAL's DE-2 and why we save parameters. 
%\end{itemize}

The three air traffic networks for Task 1, Brazil-Airports, Europe-Airports, and USA-Airports were collected by \cite{ackland2005mapping} from the government websites throughout the year 2016 and were used to evaluate algorithms to learn structural representations of nodes~\cite{ribeiro2017struc2vec,donnat2018learning}. Networks are built such that nodes represent airports and there exists an edge between two nodes if there are commercial flights between them. Brazil-Airports is a network with 131 nodes, 1,038 edges and diameter 5; Europe-Airports is a network with 399 nodes, 5,995 edges and also diameter 5; USA-Airports is a network with 1,190 nodes, 13,599 edges and diameter 8. In each dataset, the airports are divided into 4 different levels according to the annual passengers flow distribution by 3 quantiles: 25\%, 50\%, 75\%. The goal is to infer the level of an airport using solely the connectivity pattern of them. 

Tasks 2 \& 3 were carried out on three other datasets used by SEAL \cite{zhang2018link} to facilitate comparison study: C.ele, NS and PB. C.ele \cite{4726313} is a
neural network of C. elegans with 297 nodes, 2,148 edges and 3241 triangles (closed node triads), and diameter of 5, in which nodes are neurons and edges are neural linkage between them. NS \cite{newman2006finding} is a network of collaboration relationship between scientists specialized in network science, comprising of 1461 nodes 2742 edges and 3764 triangles. PB \cite{ackland2005mapping} is a network of reference relationships between political post web-pages, consisting of 1222 nodes, 16714 edges, and of diameter 8. Following \cite{zhang2018link, you2019position}, for Task 2 \& 3, we remove all links or triangles in testing sets from graph structure during the training phase to avoid label leakage.

\subsection{Baseline Details} \label{sec:baseline_supple}
%\textbf{Input features.} The six baselines that we compare with all perform neural-network-based graph convolution except struc2vec \panli{GIN should use sum instead of mean so cannot be called graph convolution}, which is a kernel method using handcrafted structural features and trained in an unsupervised manner \panli{We do not need to claim unsupervised, as struc2vec is not ``trained''. It is essentially a matrix factorazition + supervised training.}. Since we are focusing on learning structural representation with inductiveness being a desirable property, all the five neural-network-based methods use node degrees as input features if node attributes are not available. Struc2vec, on the other hand, is allowed to handcraft whatever features as formulated by the original paper \cite{ribeiro2017struc2vec} \panli{Struc2vec has not features.}. As in \cite{zhang2018link, you2019position} for Task 2 \& 3 all links or triangles in testing sets are removed from graph structure during the training phase to avoid label leakage. \panli{to be continuous}
We have five baselines based on GNNs and one baseline, struc2vec~\cite{ribeiro2017struc2vec}, based on kernels using handcrafted structural features. We first introduce the implementation of struc2vec and then discuss other baselines. 

\textit{Struc2vec} is implemented in a 2-phase manner. In the first phase, embeddings for all the nodes are learned by running the n-gram framework over a constructed graph based on structural similarity kernels. We directly use the code provided by the original paper \cite{ribeiro2017struc2vec} \footnote{https://github.com/leoribeiro/struc2vec}. In the second phase, the embeddings of nodes that in the target node set are concatenated and further fed into an one-layer fully connected neural network to make further inference. 

Regarding other GNN-based baselines, \textit{GCN} is implemented according to Equation (9) of \cite{kipf2016semi} with self-loops added. \textit{SAGE} is implemented according to Algorithm 1 of Section 3.1 in \cite{hamilton2017inductive}. Mean pooling is used as the neighborhood aggregation function. \textit{GIN} is implemented by adapting the code provided by the original paper \cite{xu2018powerful} \footnote{https://github.com/weihua916/powerful-gnns}, where we use the sum-pooling aggregation and multi-linear perception to aggregate neighbors. In all three baselines described above, ReLU nonlinearities are applied to the output of each hidden layer, followed by a Dropout layer. \textit{PGNN} layer is implemented by adapting the code provided by the original paper \cite{you2019position} \footnote{https://github.com/JiaxuanYou/P-GNN}. \textit{SEAL} is implemented by adapting the code provided by the original paper \footnote{https://github.com/muhanzhang/SEAL}. As we focus on learning structural representation with inductive capability, all the five GNN-based methods use node degrees as input features if node attributes are not available.  

Final readout layers are tailored to suit different tasks. For Task 1 since the task is node classification, the final layer for all baselines is a one-layer neural network followed by a cross entropy loss. Tasks 2 \& 3 have slightly more complex readout layers since the target entity for prediction is a node set of size 2 or 3. Note that SEAL is specifically designed for Task 2 and has its own readout that uses SortPooling over all node representations over the ego-networks of node-pairs~\cite{zhang2018end}. we refer the readers to the original paper for details~\cite{zhang2018link}. For all the other baselines, to make a fair comparison, we use the following difference-pooling: Suppose the target node set is $S$ and the representation of node $v$ for $v\in S$ is denoted by $h_v$, then we readout the representation of $S$ as  
\begin{align}\label{eq:mcpool}
    z= \sum_{u,v \in S} |h_v-h_u|
\end{align}
where $|\cdot|$ denotes component-wise absolute value. Note that Tasks 2 \& 3 are to predict the existence of a link / triangle. So we use the inner product $\langle w, z\rangle$ where $w \in \mathbb{R}^{d}$ is a trainable final projection vector, and feed this product into the binary cross entropy loss to train the models.

\subsection{\proj Variants Details.} \label{sec:degnn_supple}

\textbf{Minibatch training based on ego-network extraction.} To understand our detailed framework it is helpful to first discuss the minibatch training of GCN, although the original GCN is trained in a full-batch manner~\cite{kipf2016semi}. To train GCN in minibatchs, we first extract, for each target node $v$ in a given minibatch, an ego-network centering at $v$ within $L$-hop neighbors by doing a depth-$L$ BFS from $v$, denoted by $G_v$. Here, $L$ is the number of GCN layers intended to be used in the model. Note that the representation of node $v$ via using $L$-layer GCN over $G_v$ is the same as that via using $L$-layer GCN over the whole graph. If node attributes are not available for GCN or other WLGNNs, we may use the degree of each node as its node attributes. 

Our models are implemented by following the above mini-batch training framework. For a target node set $S$, we first extract the union of ego-networks centering at any nodes in $S$ within $L$-hop neighbors. We call the union of ego-networks as the ego-network around $S$, denoted by $G_S = \cup_{v\in S} G_v$. Note that even if $S$ has multiple nodes, $G_S$ can be extracted as a whole by running BFS. All the edges of $G_S$ between nodes that are both in $S$ will be further removed, which is denoted by $G'_S$. For $G'_S$, we associate each node $u$ in these ego-networks with the DE $\zeta(u|S)$ as extra node attributes. Specifically, we use a simple aggregation for Eq.~\eqref{eq:node2setPE}: 
\begin{align} 
\zeta(u|S) = \frac{1}{|S|} \sum_{v\in S} \zeta(u|v)
\end{align}
Next, we detail the different versions of $\zeta(u|v)$ used by different variants of \proj.

\textbf{\proj-SPD.} This variant sets $\zeta(u|v)$, $v\in S$ and $u\in G_S$ as a one-hot vector of the truncated shortest-path-distance between $u$ and $v$. That is,
\begin{align} \label{eq:degnn-spd}
\zeta_{spd}(u|v) =\text{one hot}(\min(\text{SPD}(u,v), d_{max})),
\end{align}
where $d_{max}$ is the maximum distance to be encoded. As a result, the $\zeta_{spd}(u|v)$ is a vector of length $d_{\max}+1$. %When the target node set is not a single node Eq. \ref{eq:degnn-spd} is generalized by element-wise $Sum$ pooling:
%\begin{align}
%\zeta_{spd}(u|\mathcal{V}) = Mean(\{\zeta_{spd}(u|v) | \forall v \in \mathcal{V}_{target}\})
%\end{align}
The pairwise SPDs can either be pre-computed in preprocessing stage, or be computed by traversing the extracted ego-networks on the fly. The $d_{max}(\leq L)$ helps prevent overfitting the noise in an overly large neighborhood. 

\textbf{Compare \proj-SPD with SEAL.} \proj-SPD, when used for link prediction, is similar to SEAL \cite{zhang2018link} in a sense that we both encode the distance between any node and the two nodes in the target node-pairs. However, we are fundamentally different from SEAL as SEAL uses graph-level readout of all nodes in the ego-networks. SEAL also has no discussion on the expressive power of distance encoding. Their intention of node labeling as they reported is just to let the model know which node-pair in the extracted ego-networks is the target node-pair. Moreover, the specific DE $\zeta(u|S)$ for a node-pair $S=\{v_1, v_2\}$ used in SEAL is a one-hot encoding of the value $1 + \min(\text{SPD}(u, v_1) , \text{SPD}(u, v_2)) + (d/2)[(d/2) + (d\%2)-1]$, where $d = \text{SPD}(u, v_1)+\text{SPD}(u, v_2)$. The dimension of this DE is $O(d_{\max}^2)$ which is higher than our DE (Eq.~\eqref{eq:degnn-spd}) used in \proj-SPD, which may result in model overfitting for large $d_{\max}$

%Further complexity analysis also shows that while maintaining the same expressive power our distance encoding is of length $\mathcal{O}(d_{max})$ while that of SEAL is $\mathcal{O}(d_{max}^2)$

\textbf{\proj-LP.} This variant sets $\zeta(u|v)$, $v\in S$ and $u\in G_S$ as landing probabilities of random walks (of different lengths) from node $v$ to node $u$:
\begin{align} \label{eq:degnn-lp}
\zeta_{lp}(u|v) &=((W^{(0)})_{vu}, (W^{(1)})_{vu}, (W^{(2)})_{vu}, ... ,(W^{(d_{\text{rw}})})_{vu})
%\zeta_{lp}(u|\mathcal{V}) &= Mean(\{\zeta_{lp}(u|v) | \forall v \in \mathcal{V}_{target}\})
\end{align}
where $W^{(k)}=(AD^{-1})^k$ is the $k$-step random walk matrix, $d_{\text{rw}}$ is the max step number of random walks. Notice that in principle, $\zeta_{lp}(u|v)$ encodes the distance information at a finer granularity than $\zeta_{spd}(u|S)$. This is because SPD$(u,v)$ can be inferred from $\zeta_{lp}(u|v)$ as the index of the first non-zero random walk feature. However, in practice we observe that such encoding does not always bring a significant performance gain. %the fine distance encoding may lead to overfitting issues. %especially given the inherent randomness in real-world networks.

For both \proj-SPD and \proj-LP, we use the same 1-hop neighborhood aggregation as GCN. DE can also be used to control the message passing as shown in Eq.~\eqref{eq:agg-pe}. We further discuss two variants by incorporating \proj-SPD and Eq.~\eqref{eq:agg-pe}. 

\textbf{\projA-SPD.} \projA-SPD is built on top of \proj-SPD, using the same extra node features $\zeta_{spd}(u|S)$, but allows for multi-hop aggregation by specifying Eq.~\eqref{eq:agg-pe} as
\begin{align*}
    \text{AGG}(\{f_2(h_u^{(l)}, \mathbf{A}_{vu})\}_{u\in \mathcal{N}_v})   \rightarrow \text{AGG}(\{(f_2(h_u^{(l)}, \mathbf{A}_{vu}),  \zeta_{spd}(u|v))\}_{\zeta_{spd}(u|v)\leq K}).
\end{align*}
In experiments, we choose $K=2,3$, which means that each node aggregates representations of other nodes that are not only its direct neighbors but also its (exclusive) 2-hop and even 3-hop neighbors. %The aggregation scheme given by Eq. \eqref{eq:msgpath} is therefore modified as the following (we take 2-hop aggregation):
%\begin{align*}
%G^{(L)}_v\backslash G^{(L-1)}_v, G^{(L-1)}_v &\xrightarrow[]{Agg} G^{(L-2)}_v \\
%G^{(L-1)}_v\backslash G^{(L-2)}_v, G^{(L-2)}_v &\xrightarrow[]{Agg} G^{(L-3)}_v \\
%&...\\
%G^{(2)}_v\backslash G^{(1)}_v, G^{(1)}_v &\xrightarrow[]{Agg} G^{(0)}_v(= \{v\}) \\
%\end{align*}
As we do not have edge attributes in our data, we omit $\mathbf{A}_{vu}$. Our implementation of the aggregation for the layer $l$ follows 
\begin{align*}
    h_v^{(l+1)} = \sum_{k=1}^K \text{Relu}\left( \frac{1}{|S_{v,k}|+1} \left(h_v^{(l)} + \sum_{u\in S_{v,k}} h_u^{(l)} \Theta^{(lk)}\right)\right) ,\; S_{v,k} =\{u | \zeta_{spd}(u|v) = k\}.
\end{align*}
where $\Theta^{(lk)}$ is a trainable weight matrix and for each $k$, we aggregate $k$-hop neighbors via a GCN layer with a self-loop. Note that when implementing \projA-SPD, we need to extract the ego-network of nodes within $LK$-hops, if \projA-SPD has $L$ layers. 
%The exact form of $Agg$ operator here is that, to obtain the embedding for each node $u \in G^{(L-2)}_v$, its corresponding 1-hop neighbors in $G^{(L-1)}_v$ and 2-hop neighbors in $G^{(L)}_v$ go through a standard GCN aggregation independently (each set of neighboring features together with the central node feature go through a different linear projection and then get pooled). The results are then summed up as the aggregation output of that layer.

\textbf{\projA-PR.} \projA-SPD is also built on top of \proj-SPD, but the propagation is by specifying Eq.~\eqref{eq:agg-pe} as
\begin{align*}
    \text{AGG}(\{f_2(h_u^{(l)}, \mathbf{A}_{vu})\}_{u\in \mathcal{N}_v})   \rightarrow \text{AGG}(\{(f_2(h_u^{(l)}, \mathbf{A}_{vu}),  \zeta_{gpr}(u|v))\}_{v\in V}).
\end{align*}
As the aggregation is over the whole node set, this model does not extract the ego-networks but uses the entire graphs. For the layer $l$, we further specify the above aggregation by using
\begin{align*}
    h_v^{(l+1)} = \text{Relu}\left(\sum_{u\in V} \zeta_{ppr}(u|v) h_u^{(l)} \Theta^{(l)} \right)
\end{align*}
where $\Theta^{(l)}$ is a trainable weight matrix and $\zeta_{ppr}(u|v)$ is a specific form of $\zeta_{gpr}(u|v)$ based on Personalized Pagerank scores~\cite{jeh2003scaling}, i.e, 
\begin{align*}
\zeta_{ppr}(u|v) = [\sum_{k=0}^{\infty} (0.9W)^k]_{uv} = [(I - 0.9W)^{-1}]_{uv}.
\end{align*}
Note that the above $0.9$ is a hyper-parameter. As we are just willing to show the use case, 0.9 is set as a heuristic and is not obtained via parameter tuning. Other values may yield better performance. Other types of PageRank scores may be used, \textit{e.g.}, heat-kernel PageRank scores~\cite{chung2007heat}, time-dependent PageRank scores~\cite{gleich2014dynamical}.

We compare \projA-PR with all other methods, which yields the following Table~\ref{tab:performance_full}. \projA-PR performs worse than \projA-SPD while it still works much better than WLGNNs in link and triangle predictions. Comparing these observations with the statements on GDC~\cite{klicpera2019diffusion}, we argue that missing DEs as node attributes is the key that limits the performance of link prediction via GDC.
\begin{table}[H]
\centering
\resizebox{\textwidth}{!}{%
\vspace{-0.1cm}
\begin{tabular}{l|lll|lll|ll}
\hline
 &
  \multicolumn{3}{c|}{\cellcolor[HTML]{F5F9FF}Nodes (Task 1): Average Accuracy} &
  \multicolumn{3}{c|}{\cellcolor[HTML]{FFEFEE}Node-pairs (Task 2): AUC} &
  \multicolumn{2}{c}{\cellcolor[HTML]{EAFDEA}Node-triads (Task 3): AUC} \\
\multirow{-2}{*}{\backslashbox{Method}{Data}} &
  \multicolumn{1}{c}{Bra.-Airports} &
  \multicolumn{1}{c}{Eur.-Airports} &
  \multicolumn{1}{c|}{USA-Airports} &
  \multicolumn{1}{c}{C.elegans} &
  \multicolumn{1}{c}{NS} &
  \multicolumn{1}{c|}{PB} &
  \multicolumn{1}{c}{C.elegans} &
  \multicolumn{1}{c}{NS} \\ \hline
GCN~\cite{kipf2016semi} &
  64.55$\pm$4.18 &
  54.83$\pm$2.69 &
  56.58$\pm$1.11 &
  74.03$\pm$0.99 &
  74.21$\pm$1.72 &
  89.78$\pm$0.99 &
  80.94$\pm$0.51 &
  81.72$\pm$1.50 \\
SAGE~\cite{hamilton2017inductive} &
  70.65$\pm$5.33 &
  56.29$\pm$3.21 &
  50.85$\pm$2.83 &
  73.91$\pm$0.32 &
  79.96$\pm$1.40 &
  90.23$\pm$0.74 &
  84.72$\pm$0.40 &
  84.06$\pm$1.14 \\
GIN~\cite{xu2018powerful} &
  {71.89$\pm$3.60}$^\dagger$ &
  {57.05$\pm$4.08} &
  58.87$\pm$2.12&
  75.58$\pm$0.59&
  87.75$\pm$0.56&
  91.11$\pm$0.52&
  {86.42$\pm$1.12}$^\dagger$ &
  {94.59$\pm$0.66}$^\dagger$ \\
Struc2vec~\cite{ribeiro2017struc2vec} &
  70.88$\pm$4.26 &
  57.94$\pm$4.01$^\dagger$ &
  {61.92$\pm$2.61}$^\dagger$ &
  72.11$\pm$0.31 &
  82.76$\pm$0.59 &
  90.47$\pm$0.60 &
  77.72$\pm$0.58 &
  81.93$\pm$0.61 \\
PGNN~\cite{you2019position} &
  \multicolumn{1}{c}{N/A} &
  \multicolumn{1}{c}{N/A} &
  \multicolumn{1}{c}{N/A} &
  78.20$\pm$0.33 &
  94.88$\pm$0.77 &
  89.72$\pm$0.32 &
  86.36$\pm$0.74 &
  79.36$\pm$1.49 \\
SEAL~\cite{zhang2018link} &
  \multicolumn{1}{c}{N/A} &
  \multicolumn{1}{c}{N/A} &
  \multicolumn{1}{c|}{N/A} &
  {88.26$\pm$0.56}$^\dagger$ &
  {98.55$\pm$0.32}$^\dagger$ &
  {94.18$\pm$0.57}$^\dagger$ &
  \multicolumn{1}{c}{N/A} &
  \multicolumn{1}{c}{N/A} \\ \hline
\proj-SPD &
  \textbf{73.28$\pm$2.47} &
  {56.98$\pm$2.79} &
  \textbf{63.10$\pm$0.68}$^*$ &
  \textbf{89.37$\pm$0.17}$^*$ &
  \textbf{99.09$\pm$0.79} &
  \textbf{94.95$\pm$0.37}$^*$ &
  \textbf{92.17$\pm$0.72}$^*$ &
  \textbf{99.65$\pm$0.40}$^*$ \\
\proj-LP &
  \textbf{75.10$\pm$3.80}$^*$ &
  58.41$\pm$3.20$^*$ &
  \textbf{64.16$\pm$1.70}$^*$ &
  86.27$\pm$0.33 &
  98.01$\pm$0.55 &
  91.45$\pm$0.41 &
  86.24$\pm$0.18 &
  \textbf{99.31$\pm$0.12}$^*$ \\
\projA-SPD &
  \textbf{75.37$\pm$3.25}$^*$ &
  57.99$\pm$2.39$^*$ &
  \textbf{63.28$\pm$1.59} &
  \textbf{90.05$\pm$0.26}$^*$ &
  \textbf{99.43$\pm$0.63}$^*$ &
  \textbf{94.49$\pm$0.24}$^*$ &
  \textbf{93.35$\pm$0.65}$^*$ &
  \textbf{99.84$\pm$0.14}$^*$ \\
\projA-PR &
  \textbf{73.26$\pm$4.08} &
  51.41$\pm$2.39 &
  50.34$\pm$1.50 &
  {83.07$\pm$0.77} &
  \textbf{99.46$\pm$0.37}$^*$ &
  {92.68$\pm$0.57} &
  {83.15$\pm$1.11} &
  \textbf{99.86$\pm$0.03}$^*$ \\
  \hline 
\end{tabular}%
}
\caption{\small{Model performance (including \projA-PR) in Average Accuracy and Area Under the ROC Curve (AUC) (mean in percentage $\pm$ 95\% confidence level). $\dagger$ highlights the best baselines. $^*$, \textbf{bold font}, \textbf{bold font}$^*$ respectively highlights the case where our proposed model's performance: exceeds the best baseline on average, exceeds by 70\% confidence, exceeds by 95\% confidence.}}
\vspace{-0.7cm}
\label{tab:performance_full}
\end{table}

\subsection{Model performance without validation datasets} \label{sec:exp-res-sup}
We have confirmed with the original authors of Struc2vec \cite{ribeiro2017struc2vec} and SEAL~\cite{zhang2018link} that the performance of these two baselines reported in their papers do not use validation set. The performance therein is the best testing results ever achieved when models are being trained till convergence. We think that it is necessary to include validation datasets to achieve fair comparison and therefore reported the results with validation datasets in the main text. We put the results without validation datasets here. Under both experimental settings, we draw similar conclusions for our models in comparison with the baselines.    %is widely known as a "biased setting" that raises overfitting concerns. However, to ensure fair comparison with their reported results we additionally evaluate all our proposed models under this "biased" setting, whose result is reported in Appendix

% Please add the following required packages to your document preamble:
% \usepackage{graphicx}
% \usepackage[table,xcdraw]{xcolor}
% If you use beamer only pass "xcolor=table" option, \textit{i.e.} \documentclass[xcolor=table]{beamer}
\begin{table}[H]
\centering
\resizebox{\textwidth}{!}{%
\begin{tabular}{l|lll|lll|ll}
\hline
 &
  \multicolumn{3}{c|}{\cellcolor[HTML]{F5F9FF}Nodes (Task 1): Average Accuracy} &
  \multicolumn{3}{c|}{\cellcolor[HTML]{FFEFEE}Node-pairs (Task 2): AUC} &
  \multicolumn{2}{c}{\cellcolor[HTML]{EAFDEA}Node-triads (Task 3): AUC} \\
  \multirow{-2}{*}{\backslashbox{Method}{Data}} & Bra.-airports & Eur.-airports & USA-airports  & \multicolumn{1}{c}{C.elegans}     & \multicolumn{1}{c}{NS}            & \multicolumn{1}{c}{PB}            & \multicolumn{1}{c}{C.elegans} & \multicolumn{1}{c}{NS}            \\ \hline
GCN\cite{kipf2016semi}        & 82.01$\pm$3.09 & 55.56$\pm$2.90 & 60.08$\pm$4.20 & 78.10$\pm$1.24 & 81.92$\pm$0.90 & 90.04$\pm$1.23 & 83.15$\pm$0.50 & 88.56$\pm$3.54 \\
GraphSAGE\cite{hamilton2017inductive}  & 81.48$\pm$6.26 & 63.41$\pm$5.87 & 54.62$\pm$4.29 & 77.20$\pm$1.26 & 86.16$\pm$2.95 & 90.73$\pm$1.38 & 86.84$\pm$1.13 & 87.98$\pm$4.58 \\
GIN\cite{xu2018powerful}        & 85.19$\pm$5.85$^\dagger$ & 65.47$\pm$2.99$^\dagger$ & 63.45$\pm$5.29 & 79.79$\pm$1.42 & 89.84$\pm$3.62 & 91.47$\pm$0.58 & 89.37$\pm$0.44$^\dagger$ & 95.48$\pm$0.64$^\dagger$ \\
stuc2vec\cite{ribeiro2017struc2vec}   & 84.08$\pm$9.30 & 65.30$\pm$5.60 & 68.09$\pm$2.50$^\dagger$ & 74.17$\pm$0.37 & 87.43$\pm$2.48 & 90.96$\pm$0.26 & 83.66$\pm$3.26 & 87.53$\pm$3.12 \\
PGNN\cite{you2019position}       & N/A           & N/A           & N/A           & 80.76$\pm$0.98 & 94.99$\pm$1.44 & 90.21$\pm$0.78 & 87.18$\pm$0.52 & 83.39$\pm$1.74 \\
SEAL\cite{zhang2018link}       & N/A           & N/A           & N/A           & 90.30$\pm$1.35$^\dagger$ & 98.85$\pm$0.47$^\dagger$ &94.72$\pm$0.46$^\dagger$ & N/A           & N/A           \\ \hline
\proj-SPD  & \textbf{87.78$\pm$3.95} & 65.62$\pm$3.27$^*$ & \textbf{70.95$\pm$0.80}$^*$ & 90.67$\pm$0.91$^*$ & \textbf{99.50$\pm$0.87} & \textbf{95.21$\pm$0.53} & \textbf{92.54$\pm$0.63}$^*$ & \textbf{99.91$\pm$0.13}$^*$ \\
\proj-LP   & \textbf{88.36$\pm$4.62} & 66.25$\pm$2.07$^*$ & 69.33$\pm$1.33$^*$ & 88.48$\pm$0.46 & 98.24$\pm$0.67 & 92.43$\pm$1.29 & 87.65$\pm$0.71 & \textbf{99.74$\pm$0.21}$^*$ \\
\projA-SPD & \textbf{88.43$\pm$3.24}$^*$ & \textbf{67.14$\pm$1.57}$^*$ & \textbf{71.07$\pm$0.63}$^*$ & \textbf{91.09$\pm$0.96}$^*$ & \textbf{99.52$\pm$0.88} & \textbf{95.15$\pm$0.24}$^*$ & \textbf{93.28$\pm$0.51}$^*$ & \textbf{99.94$\pm$0.06}$^*$ \\
\projA-PR  & \textbf{87.30$\pm$2.52} & 64.38$\pm$1.78 & 62.46$\pm$1.18 & 84.23$\pm$0.55 & \textbf{99.58$\pm$0.46}$^*$ & 92.94$\pm$0.76 & 84.61$\pm$0.58 & \textbf{99.95$\pm$0.06}$^*$ \\
 \hline
\end{tabular}%
}
\caption{\small{Model performance without validation set (including \projA-PR) in Average Accuracy and Area Under the ROC Curve (AUC) (mean in percentage $\pm$ 95\% confidence level). $\dagger$ highlights the best baselines. $^*$, \textbf{bold font}, \textbf{bold font}$^*$ respectively highlights the case where our proposed model's performance: exceeds the best baseline on average, exceeds by 70\% confidence, exceeds by 95\% confidence.}}
\label{tab:performance_no_val}
\end{table}

We report additional results of our model on Task 2 and 3 versus other baselines, measured by average accuracy without validation set in Table \ref{tab:performance_no_val_acc}. Similar observations as reported in the main text can be drawn from both Table \ref{tab:performance_no_val} and \ref{tab:performance_no_val_acc}: the strongest baselines are given by SEAL~\cite{zhang2018link} and GIN \cite{xu2018powerful}, while our \projA variants further significantly outperform those baselines on all tasks.

\begin{table}[H]
\centering
\resizebox{0.8\textwidth}{!}{%
\vspace{-0.1cm}
\begin{tabular}{l|lll|ll}
\hline
 &
%   \multicolumn{3}{c|}{\cellcolor[HTML]{F5F9FF}Nodes (Task 1)} &
  \multicolumn{3}{c|}{\cellcolor[HTML]{FFEFEE}Node-pairs (Task 2) Average Accuracy} &
  \multicolumn{2}{c}{\cellcolor[HTML]{EAFDEA}Node-triads (Task 3) Average Accuracy} \\
\multirow{-2}{*}{\backslashbox{Method}{Data}} &
%   \multicolumn{1}{c}{Bra.-Airports} &
%   \multicolumn{1}{c}{Eur.-Airports} &
%   \multicolumn{1}{c|}{USA-Airports} &
  \multicolumn{1}{c}{C.elegans} &
  \multicolumn{1}{c}{NS} &
  \multicolumn{1}{c|}{PB} &
  \multicolumn{1}{c}{C.elegans} &
  \multicolumn{1}{c}{NS} \\ \hline
GCN~\cite{kipf2016semi} &
%   82.01$\pm$ 1.85 &
%   55.56$\pm$ 1.74 &
%   60.08$\pm$ 2.52 &
  65.01$\pm$ 0.45 &
  74.69$\pm$ 1.39 &
  78.35$\pm$ 1.69 &
  69.47$\pm$ 2.69 &
  82.16$\pm$ 1.27 \\
SAGE~\cite{hamilton2017inductive} &
%   81.48$\pm$ 3.76 &
%   63.41$\pm$ 3.52 &
%   54.62$\pm$ 2.57 &
  67.96$\pm$ 0.90 &
  78.29$\pm$ 2.40 &
  83.53$\pm$ 1.41 &
  76.67$\pm$ 0.72 &
  88.45$\pm$ 0.65 \\
GIN~\cite{xu2018powerful} &
%   {85.19$\pm$3.51}$^\dagger$ &
%   {65.47$\pm$1.79}$^\dagger$ &
%   63.45$\pm$3.17 &
  69.45$\pm$1.24 &
  82.17$\pm$1.07 &
  83.01$\pm$0.83 &
  {77.60$\pm$0.85}$^\dagger$ &
  {92.67$\pm$1.44}$^\dagger$ \\
Struc2vec~\cite{ribeiro2017struc2vec} &
%   84.08$\pm$5.50 &
%   65.30$\pm$3.36 &
%   {68.09$\pm$1.50}$^\dagger$ &
  68.41$\pm$2.57 &
  80.64$\pm$0.92 &
  74.43$\pm$3.48 &
  71.06$\pm$3.86 &
  83.42$\pm$1.67 \\
PGNN~\cite{you2019position} &
%   \multicolumn{1}{c}{N/A} &
%   \multicolumn{1}{c}{N/A} &
%   \multicolumn{1}{c|}{N/A} &
  71.40$\pm$1.68 &
  91.04$\pm$0.85 &
  86.44$\pm$1.14 &
  76.34$\pm$0.23 &
  80.67$\pm$0.49 \\
SEAL~\cite{zhang2018link} &
%   \multicolumn{1}{c}{N/A} &
%   \multicolumn{1}{c}{N/A} &
%   \multicolumn{1}{c|}{N/A} &
  {83.90$\pm$0.97}$^\dagger$ &
  {97.89$\pm$0.40}$^\dagger$ &
  {88.92$\pm$0.95}$^\dagger$ &
  \multicolumn{1}{c}{N/A} &
  \multicolumn{1}{c}{N/A} \\ \hline
\proj-SPD &
%   \textbf{87.78$\pm$2.86} &
%   {65.62$\pm$3.28}$^*$ &
%   \textbf{70.95$\pm$0.81}$^*$ &
  {83.78$\pm$0.88} &
  \textbf{99.54$\pm$0.28}$^*$ &
  \textbf{89.82$\pm$0.82}$^*$ &
  \textbf{90.82$\pm$0.62}$^*$ &
  \textbf{100.0$\pm$0.00}$^*$ \\
\proj-LP &
%   \textbf{88.36$\pm$4.51} &
%   \textbf{66.25$\pm$2.08} &
%   \textbf{69.33$\pm$1.33} &
  70.96$\pm$1.27 &
  96.96$\pm$0.57 &
  84.37$\pm$1.03 &
  \textbf{82.90$\pm$0.69}$^*$ &
  \textbf{100.0$\pm$0.00}$^*$ \\
\projA-SPD &
%   \textbf{88.43$\pm$3.23} &
%   \textbf{67.14$\pm$1.57} &
%   \textbf{71.07$\pm$0.63}$^*$ &
  \textbf{84.81$\pm$0.69}$^*$ &
  \textbf{99.77$\pm$0.24}$^*$ &
  \textbf{89.83$\pm$0.98}$^*$ &
  \textbf{91.24$\pm$0.22}$^*$ &
  \textbf{100.0$\pm$0.00}$^*$ \\ 
\projA-PR &
%   \textbf{87.30$\pm$2.51} &
%   64.38$\pm$1.78 &
%   62.46$\pm$1.19 &
  76.95$\pm$0.85 &
  \textbf{99.88$\pm$0.19}$^*$ &
  85.82$\pm$0.87 &
  \textbf{83.18$\pm$0.71}$^*$ &
  \textbf{100.0$\pm$0.00}$^*$ \\ \hline
\end{tabular}%
}
\caption{\small{Model performance without validation set (including \projA-PR) in Average Accuracy (mean in percentage $\pm$ 95\% confidence level). $\dagger$ highlights the best baselines. $^*$, \textbf{bold font}, \textbf{bold font}$^*$ respectively highlights the case where our proposed model's performance: exceeds the best baseline on average, exceeds by 70\% confidence, exceeds by 95\% confidence.}}
\label{tab:performance_no_val_acc}
\end{table}

\subsection{Hyperparameters Tuning.} \label{sec:hyper-supple}
Table \ref{tab:hypp} lists the most important hyperparameters' at a glance, which applies to both the baselines and our proposed models. Grid search is used to find the best hyperparameters combination. The models are sufficiently trained till the cross entropy loss converges and we report the best model by running each for 20 times over different random seed. For more details please refer to the code attached.
\begin{table}[H]
\centering
\begin{tabular}{lll}
\hline
\textbf{Hyperparameters} & \textbf{Value / Range} & \textbf{Notes}                                                \\ \hline
batch size               & 64, 128                &                                                      \\
learning rate            & 1e-4                   &                                                      \\
optimizer            & SGD        & stochastic gradient descent\\
conv. layers             & 1, 2, 3                & struc2vec does not follow this setting              \\
conv. hidden dim.        & 20, 50, 80, 100        & struc2vec does not follow this setting              \\
dropout                  & 0, 0.2                 &                                                      \\
$d_{rw}$                & 3, 4                   & \#steps of random walk, valid only for \proj-RW     \\
$d_{spd}$                 & 3, 4                   & maximum shortest path distance for \proj-SPD variants \\
prop\_depth & 1, 2, 3 & \begin{tabular}[c]{@{}l@{}} the number of hops of message in one layer, only\\ valid for \projA-SPD, 1 for all the others\end{tabular} \\ \hline
\end{tabular}
\caption{List of hyperparameters and their value / range.}
\label{tab:hypp}
\end{table}

% \textbf{Additional datasets.} 
% On top of the regular, real dataset introduced above, we additionally experiment on another three slightly more special datasets: Foodweb, Square-grid, and Trangular-grid. Foodweb’s full name is a Florida Bay Trophic Exchange Network [], which network comprises 128 vertices corresponding to different species or organisms that live in the Bay, and 2106 directed edges indicating carbon exchange between two species (\textit{i.e.} one species predates the other)[].Each node fall into one of the following three categories: Notice that the directions of edges incorporates additional latent information to be utilized which does not exist with undirected graph. Appendix XXX includes more details of how we handle this.

% Square-grid and Triangular-grid are two synthetic dataset …

\section{A Brief Introduction of Higher-Order Weisfeiler-Lehman Tests} \label{sec:WL-supp}
We mentioned higher-order WL tests in our context, especially on the 2-WL test. There are different definitions of the $k$-WL test ($k\geq 2$), while in this work we followed the definition in~\cite{cai1992optimal}. Note that the $k$-WL test here also corresponds to the $k$-WL' test in~\cite{grohe2017descriptive} and the $k$-FWL test in~\cite{maron2019provably}, and are equivalent to the $k+1$-WL tests in~\cite{maron2019provably,morris2019weisfeiler}.

\textbf{The $k$-WL test} ($k\geq 2$) follows the following coloring procedure: 
\begin{enumerate}
    \item For each $k$-tuple of node set $V_i=(v_{i_1}, v_{i_2},...,v_{i_k})\in V^k$, $i\in [n^k]$, we initialize $V_i$ with a color denoted by $C_i^{(0)}$. These colors satisfies that for two $k$-tuples, $V_i$ and $V_j$, $C_i^{(0)}$ and $C_j^{(0)}$ are the same if and only if for $a,b\in[k]$ (1) $v_{i_a} = v_{i_b} \Leftrightarrow v_{j_a} = v_{j_b}$ and (2) $(v_{i_a}, v_{i_b})\in E \Leftrightarrow (v_{i'_a}, v_{i'_b})\in E$.
    \item For each $k$-tuple $V_i$ and $u\in V$, define $N(V_i;u)$ as a sequence of $k$-tuples such that $N(V_i;u) = ((u, v_{i_2},...,v_{i_k}), (v_{i_1}, u,...,v_{i_k}), (v_{i_1}, v_{i_2},...,u))$. Then, the color of $V_i$ can be updated via the following mapping:
    \begin{align*}
        C_i^{(l+1)} \leftarrow g(C_i^{(l)}, \{(C_{j}^{(l)}|V_j\in N(V_i;u))\}_{u\in V}),
    \end{align*}
    where $g(\cdot)$ is injective coloring.
    \item For each step $l$, $\{C_i^{(l)}\}_{i\in[n^k]}$ is a coloring configuration of the graph $G$, which is essentially a multi-set. If two graphs have different coloring configurations, these two graphs are determined to be non-isomorphic, while the inverse is not true. 
\end{enumerate}
Note that the step 2 essentially requires to aggregate colors of $nk$-tuples and thus even when $k=2$, the 2-WL test may not leverage the sparsity of graph structure to keep good scalability. Ring-GNN~\cite{chen2019equivalence} and PPGN~\cite{maron2019provably} essentially try to achieve the expressive power of the 2-WL test. They are not scalable to process large graphs and they were only evaluated for entire-graph-level tasks such as graph classification and graph regression~\cite{chen2019equivalence,maron2019provably}.

\end{document}